\documentclass[review]{elsarticle}

\usepackage{lineno,hyperref}
\modulolinenumbers[5]

\journal{Artificial Intelligence}

\usepackage[left=2cm,right=2cm,top=2.0cm,bottom=2.0cm]{geometry}
\usepackage{fancyvrb}

\usepackage{verbatim}

\usepackage[fit]{truncate}
\usepackage{float}
\usepackage{amsthm}
\usepackage{appendix}
\usepackage{amsmath}
\usepackage{amssymb}
\usepackage{multicol}
\usepackage{multirow}
\usepackage[ruled]{algorithm2e}
\usepackage{amsfonts}
\usepackage{algpseudocode}
\usepackage{array}
\usepackage[numbers]{natbib}
\usepackage{enumitem}
\usepackage{changepage}
\usepackage{lipsum}
\usepackage{framed}
\usepackage{xcolor}
\colorlet{shadecolor}{gray!40}
\usepackage{setspace}
\usepackage{xspace}
\usepackage{titlecaps}
\usepackage{bibentry}
\nobibliography*
\usepackage[most]{tcolorbox}

\usepackage{pdfpages}
\usepackage{txfonts}
\usepackage{pxfonts}
\usepackage{listings}
\usepackage{pgfplots}
\usepackage{subfig}
\usepackage{tabularx}
\usepackage{bbm}

\Addlcwords{as is for an of and to at if the from with a each in on into -- *}

\newenvironment{acknowledgements}
{
  \cleardoublepage

  \thispagestyle{plain}

  \begin{center}
    {\large \bf Acknowledgements}
  \end{center}

}{\cleardoublepage}

\newenvironment{dedication}
{
  \cleardoublepage

  \thispagestyle{plain}

  \begin{center}
    {\large \bf Dedication}
  \end{center}

  \begin{center}

}{\cleardoublepage}
\def\enddedication{
  \par
\end{center}

\cleardoublepage

}

\usepackage{lineno}

\usepackage[nopostdot, style=super, toc]{glossaries}

\usepackage{tikz}
\usetikzlibrary{shapes,chains,arrows.meta,calc,decorations.markings,math,arrows.meta,matrix, positioning, patterns}

\newtheorem{theorem}{Theorem}

\newtheorem{definitionx}{Definition}
\newenvironment{definition}
  {\definitionx}
  {\enddefinitionx}

\theoremstyle{definition}
\newtheorem{examplex}{Example}
\newenvironment{example}
  {\pushQED{\qed}\examplex}
  {\popQED\endexamplex}

\newcommand{\define}[1]{{\bf #1}}

\usepackage{thrmappendix}
\newtype{lemma}{theorem}{Lemma}

\makeatletter
\def\thmhead@plain#1#2#3{%
  \thmname{#1}\thmnumber{\@ifnotempty{#1}{ }\@upn{#2}}%
\thmnote{ {\the\thm@notefont#3}}}
\let\thmhead\thmhead@plain
\makeatother

\newcommand{\asp}[1]{\mbox{$\mathtt{#1}$}}
\DeclareMathOperator{\codeif}{\mathtt{:-} }

\DeclareMathOperator{\naf}{\;\mathtt{not}\;}

\definecolor{shadecolor}{rgb}{0.95,0.95,0.95}

\makeatletter
\renewenvironment{abstract}{%
    \begin{center}%
      {\bfseries \abstractname\vspace{-.5em}\vspace{\z@}}%
    \end{center}%
  }
{}
\makeatother

\theoremstyle{definition}

\newcommand{\sequence}[1]{\textbf{#1, ...}}

\newcommand{\fork}{%
  \mathbin{%
    \supset
    \mathrel{\mkern-9mu}%
    \mathrel{-}%
  }%
}

\def\set#1{\{#1\}}

\newcommand{\bigsetbegin}{ \left\{ \begin{array}{l}}
\newcommand{\bigsetend}{ \end{array}\right\}}

\newcommand{\sys}{{\textsc{Apperception Engine}}}

\newcommand{\logic}{Datalog\textsuperscript{$\fork$}}

\newcommand{\defbegin}{
	\begin{shaded}
	 \begin{definition}
}
\newcommand{\defend}{
	 \end{definition}
	\end{shaded}
}

\newtheorem{innercustomgeneric}{\customgenericname}
\providecommand{\customgenericname}{}
\newcommand{\newcustomtheorem}[2]{%
  \newenvironment{#1}[1]
  {%
   \renewcommand\customgenericname{#2}%
   \renewcommand\theinnercustomgeneric{##1}%
   \innercustomgeneric
  }
  {\endinnercustomgeneric}
}

\newcustomtheorem{customTheorem}{Theorem}

\DeclareMathOperator*{\argmax}{arg\,max}

\let\OLDthebibliography\thebibliography
\renewcommand\thebibliography[1]{
  \OLDthebibliography{#1}
  \setlength{\parskip}{0pt}
  \setlength{\itemsep}{0pt plus 0.3ex}
}

\begin{document}

\begin{frontmatter}

\title{Making sense of sensory input}

\author[address-2-deepmind,address-1-imperial]{Richard Evans\corref{mycorrespondingauthor}}
\cortext[mycorrespondingauthor]{Corresponding author}
\ead{richardevans@google.com}

\author[address-3-valencia,address-4-cambridge]{Jos\'e Hern\'andez-Orallo}
\author[address-2-deepmind]{Johannes Welbl}
\author[address-2-deepmind]{\mbox{Pushmeet Kohli}}
\author[address-1-imperial]{Marek Sergot}

\address[address-2-deepmind]{DeepMind, London}
\address[address-1-imperial]{Imperial College London}
\address[address-3-valencia]{Universitat Polit\`ecnica de Val\`encia}
\address[address-4-cambridge]{CFI, University of Cambridge}

\end{frontmatter}



\begin{abstract}

This paper attempts to answer a central question in unsupervised learning: what does it mean to ``make sense'' of a sensory sequence? 
In our formalization, making sense involves constructing a symbolic causal theory that both explains the sensory sequence and also satisfies a set of unity conditions.
The unity conditions insist that the constituents of the causal theory -- objects, properties, and laws -- must be integrated into a coherent whole. 
On our account, making sense of sensory input is a type of program synthesis, but it is \emph{unsupervised} program synthesis. 

Our second contribution is a computer implementation, the \sys{}, that was designed to satisfy the above requirements.
Our system is able to produce interpretable human-readable causal theories from very small amounts of data, because of the strong inductive bias provided by the unity conditions.
A causal theory produced by our system is able to predict future sensor readings, as well as retrodict earlier readings, and impute (fill in the blanks of) missing sensory readings, in any combination. 
In fact, it is able to do all three tasks simultaneously.

We tested the engine in a diverse variety of domains, including cellular automata, rhythms and simple nursery tunes, multi-modal binding problems, occlusion tasks, and sequence induction intelligence tests. 
In each domain, we test our engine's ability to predict future sensor values, retrodict earlier sensor values, and impute missing sensory data.
The \sys{} performs well in all these domains, significantly out-performing neural net baselines.
We note in particular that in the sequence induction intelligence tests, our system achieved human-level performance.
This is notable because our system is not a bespoke system designed specifically to solve intelligence tests, but a \emph{general-purpose} system that was designed to make sense of \emph{any} sensory sequence.

\end{abstract}


\section{Introduction}
\label{sec:intro}

Imagine a machine, equipped with various sensors, that receives a stream of sensory information.
It must, somehow, \emph{make sense} of this stream of sensory data.
But what does it mean, exactly, to ``make sense'' of sensory data?
We have an intuitive understanding of what is involved in making sense of the sensory stream -- but can we specify precisely what is involved? Can this intuitive notion be formalized?

One approach is to treat the sensory sequence as the input to a supervised learning problem\footnote{See, for example, \cite{mathieu2015deep,finn2017deep,weber2017imagination,liu2017video,buesing2018learning}. See also the Predictive Processing paradigm: \cite{friston2012history,clark2013whatever,swanson2016predictive}.}:
given a sequence $x_{1:t}$ of sensory data from time steps 1 to $t$, maximize the probability of the next datum $x_{t+1}$. 
This family of approaches seeks to maximize $p(x_{t+1} \mid x_{1:t})$.
More generally, as well as training the system to predict future sensor readings, we may also train it to \emph{retrodict} past sensor readings (maximizing $p(x_1 \mid x_{2:t})$), and to \emph{impute} missing intermediate values (maximizing $p(x_i \mid x_{1:{i-1}}, x_{{i+1}:t})$). 

We believe there is more to ``making sense'' than prediction, retrodiction, and imputation. 
Predicting the future state of one's photoreceptors may be \emph{part} of what is involved in making sense -- but it is not on its own sufficient. 
The ability to predict, retrodict, and impute is a sign, a surface manifestation, that one has made sense of the input. 
We want to define the underlying mental model that is constructed when one makes sense of the sensory input, and to show how constructing this mental model \emph{ipso facto} enables one to predict, retrodict, and impute.

In this paper, we assume that making sense of sensory input involves constructing a symbolic theory that explains the sensory input \cite{tenenbaum2006theory,pasula2007learning,ray2009nonmonotonic,inoue2014learning}.
A number of authors, including Lake \cite{lake2017building} and Marcus \cite{marcus2019rebooting}, have argued that constructing an explanatory theory is a key component of common sense.
Following Spelke and others \cite{spelke2007core,diuk2008object}, we assume the theory must posit objects that persist over time, with properties that change over time according to general laws.
Further, we assume, following John McCarthy and others \cite{mccarthy2006challenges,inoue2016meta,teijeiro2018adoption}, that making sense of the surface sensory perturbations requires positing \emph{latent objects}: some sensory sequences can only be made intelligible by hypothesizing an underlying reality, distinct from the surface features of our sensors, that makes the surface phenomena intelligible. The underlying reality consists of latent objects that causally interact with our sensors to product the sensory perturbations we are given as input.
Once we have constructed such a theory, we can apply it to predict future sensor readings, to retrodict past readings, or to impute missing values.

Now constructing a symbolic theory that explains the sensory sequence is necessary for making sense of the sequence. But it is not, we claim, sufficient.
There is one further additional ingredient that we add to our characterisation of ``making sense''.
This is the requirement that our theory exhibits a particular form of \emph{unity}:
the constituents of our theory -- objects, properties, and atoms -- must be integrated into a coherent whole. 
Specifically, our unity condition requires that the objects are interrelated via chains of binary relations, the properties are connected via exclusion relations, and the atoms are unified by jointly satisfying the theory's constraints.
This extra unity condition is necessary, we argue, for the theory to achieve good accuracy at prediction, retrodiction, and imputation.
\footnote{We show, in the ablation experiments of Section \ref{sec:experiments}, that without these unity conditions, our computer implementation is much less accurate at prediction, retrodiction, and imputation.}

This paper makes two main contributions.
The first is a formalization of what it means to ``make sense'' of the stream of sensory data.
According to our definition, making sense of a sensory sequence involves positing a symbolic causal theory -- a set of objects, a set of concepts, a set of initial conditions, a set of rules, and a set of constraints -- that together satisfy two conditions. 
First, the theory must explain the sensory readings it is given.
Second, the theory must satisfy a particular type of unity.
Our definition of unity involves four conditions.
(i) \emph{Spatial unity}: all objects must be unified in space via a chain of binary relations.
(ii) \emph{Conceptual unity}: all concepts must be unified via constraints.
(iii) \emph{Static unity}: all propositions that are true at the same time must jointly satisfy the set of constraints.
(iv) \emph{Temporal unity}: all the states must be unified into a sequence by causal rules.

Our second contribution is a description of a particular computer system, the \sys{}\footnote{``Apperception'' comes from the French `apercevoir'. The term is introduced by Leibniz in the \emph{New Essays Concerning Human Understanding} \cite{leibniz1996leibniz}.
Apperception, as we use it in this paper, is the process of assimilating sensory information into a coherent unified whole. See Dewey: ``\emph{Apperception} is the relating activity which combines the various sensuous elements presented to the mind at one time into a \emph{whole}, and which \emph{unites} these wholes, recurring at successive times, into a continuous mental life, thereby making psychical life intelligent'' \cite{dewey1888leibniz}.}, that was designed to satisfy the conditions described above\footnote{Our code and datasets are publicly available at https://github.com/RichardEvans/apperception.}.
We introduce a causal language, \logic{}, that was designed for reasoning about infinite temporal sequences.
Given a sensory sequence, our system synthesizes a \logic{} program that, when executed, generates a trace that both explains the sensory sequence and also satisfies the four conditions of unity.
This can be seen as a form of \emph{unsupervised program synthesis} \cite{ellis2015unsupervised}.
In traditional supervised program synthesis, we are given input/output pairs, and search for a program that, when executed on the inputs, produces the desired outputs.
Here, in unsupervised program synthesis, we are given a sensory sequence, and search for a causal theory that, when executed, generates a trajectory that both respects the sensory sequence and also satisfies the conditions of unity.

The \sys{} has a number of appealing features.
(1) Because the causal theories it generates are symbolic, they are human-readable and hence verifiable. 
We can understand precisely how the system is making sense of its sensory data\footnote{Human readability is a much touted feature of Inductive Logic Programming (ILP) systems, but when the learned programs become
large and include a number of invented auxiliary predicates, the resulting programs become less readable
(see \cite{muggleton2018ultra}). But even a large and complex machine-generated logic program will be easier to understand
than a large tensor of floating point numbers.}.
(2) Because of the strong inductive bias (both in terms of the design of the causal language, \logic{}, but also in terms of the unity conditions that must be satisfied), the system is data-efficient, able to make sense of the shortest and scantiest of sensory sequences\footnote{Our sensory sequences are less than 300 bits. See Table \ref{table:experiments-overview}.}.
(3) Our system generates a causal model that is able to accurately predict future sensory input.
But \emph{that is not all it can do}; it is also able to retrodict previous values and impute missing sensory values in the middle of the sensory stream.
In fact, our system is able to predict, retrodict, and impute simultaneously\footnote{See Example \ref{ex:explains} for a case where the \sys{} jointly predicts, retrodicts, and imputes.}.
(4) The \sys{} has been tested in a diverse variety of domains, with encouraging results. 
The five domains we use are elementary cellular automata, rhythms and nursery tunes, ``Seek Whence'' and C-test sequence induction intelligence tests \cite{hofstadter2008fluid}, multi-modal binding tasks, and occlusion problems. 
These tasks were chosen because they require cognition rather than mere classificatory perception, and because they are simple for humans but not for modern machine learning systems, e.g. neural networks\footnote{Figure \ref{fig:baselines-chart} shows how neural baselines struggle to solve these tasks.}. 
The \sys{} performs well in all these domains, significantly out-performing neural net baselines.
These results are significant because neural systems typically struggle to solve the binding problem (where information from different modalities must somehow be combined into different aspects of one unified object) and fail to solve occlusion tasks (in which objects are sometimes visible and sometimes obscured from view).

We note in particular that in the sequence induction intelligence tests, our system achieved human-level performance.
This is notable because the \sys{} was not designed to solve these induction tasks; it is not a bespoke hand-engineered solution to this particular domain. 
Rather, it is a \emph{general-purpose}\footnote{Although the algorithm is general-purpose (the exact same code is applied to many different domains), domain-specific knowledge can be injected as needed. For example, in the sequence induction tasks (Section \ref{sec:seek-whence}), the successor relation on letters is provided to the system. But this information is arguably part of the problem formulation, rather than part of the solution.} system that attempts to make sense of \emph{any} sensory sequence.
This is, we believe, a highly suggestive result \cite{hernandez2016computer}.

In ablation tests, we tested what happened when each of the four unity conditions was turned off.
Since the system's performance deteriorates noticeably when each unity condition is ablated, this indicates that the unity conditions are indeed doing vital work in our engine's attempts to make sense of the incoming barrage of sensory data.


\subsection{Related work}
\label{sec:intro-related}

A human being who has built a mental model of the world can use that model for counterfactual reasoning, anticipation, and planning \cite{craik1967nature,harris2000work,gerstenberg2017intuitive}.
Similarly, computer agents endowed with mental models are able to achieve impressive performance in a variety of domains.
For instance, Lukasz Kaiser et al.~\cite{kaiser2019} show that a model-based RL agent trained on 100K interactions compares with a state-of-the-art model-free agent trained on tens or hundreds of millions of interactions. 
David Silver et al.~\cite{silver2018general} have shown that a model-based Monte Carlo tree search planner with policy distillation can achieve superhuman level performance in a number of board games. The tree search relies, crucially, on an accurate model of the game dynamics.

When we have an accurate model of the environment, we can leverage that model to anticipate and plan.
But in many domains, we do not have an accurate model.
If we want to apply model-based methods in these domains, we must \emph{learn} a model from the stream of observations.
In the rest of this section, we shall describe various different approaches to representing and learning models, 
and show where our particular approach fits into the landscape of model learning systems.

Before we start to build a model to explain a sensory sequence, one fundamental question is: what form should the model take?
We shall distinguish three dimensions of variation of models (adapted from \cite{hamrick2019analogues}): first, whether they simply model the observed phenomena, or whether they also model latent structure; second, whether the model is explicit and symbolic or implicit; and third, what type of prior knowledge is built into the model structure. 

We shall use the hidden Markov model (HMM)\footnote{Many systems predict state dynamics for partially observable Markov decision processes (POMDPs), rather than HMMs. In a POMDP, the state transition function depends on the previous state $z_t$ and the action $a_t$ performed by an agent. See Jessica Hamrick's paper for an excellent overview \cite{hamrick2019analogues} of model-based methods in deep learning that is framed in terms of POMDPs. In this paper, we consider HMMs. Adding actions to our model is not particularly difficult, but is left for further work.} \cite{baum1966statistical,ghahramani2001introduction} as a general framework for describing sequential processes.
Here, the observation at time $t$ is $x_t$, and the latent state is $z_t$.
In a HMM, the observation $x_t$ at time $t$ depends only on the latent (unobserved) state $z_t$. 
The state $z_t$ in turn depends only on the previous latent state $z_{t-1}$.

%
%

The first dimension of variation amongst models is whether they actually use latent state information $z_t$ to explain the observation $x_t$. 
Some approaches \cite{feinberg2018model,nagabandi2018neural,battaglia2016interaction,chang2016compositional,mrowca2018flexible,sanchez2018graph} assume we are \emph{given} the underlying state information $z_{1:t}$.
In these approaches, there is no distinction between the observed phenomena and the latent state: $x_i = z_i$.
With this simplifying assumption, the only thing a model needs to learn is the transition function.
Other approaches \cite{lerer2016learning,finn2017deep,bhattacharyya2018long} focus only on the observed phenomena $x_{1:t}$ and ignore latent information $z_{1:t}$ altogether. 
These approaches predict observation $x_{t+1}$ given observation $x_t$ without positing any hidden latent structure.
Some approaches take latent information seriously \cite{oh2015action,chiappa2017recurrent,ha2018recurrent,buesing2018learning,janner2018reasoning}.
These jointly learn a perception function (that produces a latent $z_t$ from an observed $x_t$), a transition function (producing a next latent state $z_{t+1}$ from latent state $z_t$) and a rendering function (producing a predicted observation $x_{t+1}$ from the latent state $z_{t+1}$).
Our approach also builds a latent representation of the state.
As well as positing latent properties (unobserved properties that explain observed phenomena), 
we also posit latent \emph{objects} (unobserved objects whose relations to observed objects explain observed phenomena).

The second dimension of variation concerns whether the learned model is explicit, symbolic and human-readable, or implicit and inscrutable. In some approaches \cite{oh2015action,chiappa2017recurrent,ha2018recurrent,buesing2018learning}, the latent states are represented by vectors and the dynamics of the model by weight tensors. 
In these cases, it is hard to understand what the system has learned. 
In other approaches \cite{zhang2018composable,xu2019,asai2018classical,asai2019unsupervised}, the latent state is represented symbolically, but the state transition function is represented by the weight tensor of a neural network and is inscrutable. We may have some understanding of what state the machine thinks it is in, but we do not understand why it thinks there is a transition from this state to that. In some approaches \cite{ray2009nonmonotonic,inoue2014learning,katzouris2015incremental,michelioudakis2016mathtt,katzouris2016online,michelioudakis2018semi}, both the latent state and the state transition function are represented symbolically.
Here, the latent state is a set of ground atoms\footnote{A ground atom is a logical atom that contains no variables.} and the state transition function is represented by a set of universally quantified rules.
Our approach falls into this third category. 
Here, the model is fully interpretable: we can interpret the state the machine thinks it is in, and we can understand the reason why it believes it will transition to the next state.

A third dimension of variation between models is the amount and type of prior knowledge that they include.
Some model learning systems have very little prior knowledge.
In some of the neural systems (e.g. \cite{finn2017deep}), the only prior knowledge is the spatial invariance assumption implicit in the convolutional network's structure.
Other models incorporate prior knowledge about the way objects and states should be represented.
For example, some models assume objects can be composed in hierarchical structures \cite{xu2019}.
Other systems additionally incorporate prior knowledge about the type of rules that are used to define the state transition function. 
For example, some \cite{michelioudakis2016mathtt,katzouris2016online,michelioudakis2018semi} use prior knowledge of the event calculus \cite{kowalski1989logic}.
Our approach falls into this third category.
We impose a language bias in the form of rules used to define the state transition function and also impose additional requirements on candidate sets of rules: they must satisfy the four unity conditions introduced above (and elaborated in Section \ref{sec:unity-conditions} below).

To summarize, in order to position our approach within the landscape of other approaches, we have distinguished three dimensions of variation. 
Our approach differs from neural approaches in that the posited theory is explicit and human readable. Not only is the representation of state explicit (represented as a set of ground atoms) but the transition dynamics of the system are also explicit (represented as universally quantified rules in a domain specific language designed for describing causal structures).
Our approach differs from other inductive program synthesis methods in that it posits significant latent structure in addition to the induced rules to explain the observed phenomena: in our approach, explaining a sensory sequence does not just mean constructing a set of rules that explain the transitions; it also involves positing a type signature containing a set of latent properties and a set of latent \emph{objects}. 
Our approach also differs from other inductive program synthesis methods in the type of prior knowledge that is used: as well as providing a strong language bias by using a particular representation language (a typed extension of datalog with causal rules and constraints), we also inject a substantial inductive bias: the unity conditions, the key constraints on our system, represent domain-\emph{independent} prior knowledge. 
Our approach also differs from other inductive program synthesis methods in being entirely unsupervised. In contrast,  OSLA and OLED \cite{michelioudakis2016mathtt,katzouris2016online} are supervised, and SPLICE \cite{michelioudakis2018semi} is semi-supervised. See Section \ref{sec:discrete-related} for detailed discussion.

\subsection{Paper outline}

Section \ref{sec:background} introduces basic notation. 
Section \ref{sec:apperception-framework} presents the main definition of what it means for a theory to count as a unified interpretation of a sensory sequence. 
Section \ref{sec:system} describes a computer system that is able to generate unified interpretations of sensory sequences. 
Section \ref{sec:experiments} describes our experiments in five different types of task: elementary cellular automata, rhythms and nursery tunes, ``Seek Whence'' sequence induction tasks, multi-modal binding tasks, and occlusion problems.
In Section \ref{sec:noise}, we show how our system is extended to robustly handle noise.
Related work is discussed in Section \ref{sec:discrete-related}.

\section{Background}
\label{sec:background}

In this paper, we use basic concepts and standard notation from logic programming \cite{kowalski1974predicate,apt1990logic,lloyd2012foundations}.
A function-free atom is an expression of the form $p(t_1,..., t_n)$, where $p$ is a predicate of arity $n \geq 0$ and
each $t_i$ is either a variable or a constant.
We shall use $a, b, c, ...$ for constants, $X, Y, Z, ...$ for variables, and $p, q, r, ...$ for predicate symbols.

A \define{substitution} $\sigma$ is a mapping from variables to terms. 
For example $\sigma = \{X / a, Y / b\}$ replaces variable $X$ with constant $a$ and replaces variable $Y$ with constant $b$.
We write $\alpha \sigma$ for the application of substitution $\sigma$ to atom $\alpha$, so e.g.~$p(X, Y) \sigma = p(a,b)$.

%

A \define{Datalog clause} is a definite clause of the form $\alpha_1 \wedge ... \wedge \alpha_n \rightarrow \alpha_{0}$
where each $\alpha_i$ is an atom and $n \geq 0$. 
It is traditional to write clauses from right to left: $\alpha_0 \leftarrow \alpha_1, ..., \alpha_n$. 
In this paper, we will define a Datalog interpreter implemented in another logic programming language, ASP (answer-set programming). In order to keep the two languages distinct, we write Datalog rules from left to right and ASP clauses from right to left.
A \define{Datalog program} is a set of Datalog clauses.

A key result of logic programming is that every Datalog program has a unique subset-minimal least Herbrand model that can be directly computed by repeatedly generating the consequences of the ground instances of the clauses \cite{van1976semantics}.

We turn now from Datalog to normal logic programs under the answer set semantics \cite{gelfond1988stable}.
A \define{literal} is an atom $\alpha$ or a negated atom $\mathit{not} \; \alpha$. 
A \define{normal logic program} is a set of clauses of the form:
\begin{eqnarray*}
\alpha_0 \leftarrow \alpha_1, ..., \alpha_n
\end{eqnarray*}
where $\alpha_0$ is an atom, $\alpha_1, ..., \alpha_n$ is a conjunction of literals, and $n \geq 0$.
Normal logic clauses extend Datalog clauses by allowing functions in terms and by allowing negation by failure in the body of the rule.

%

\define{Answer Set Programming} (ASP) is a logic programming language based on normal logic programs under the answer set semantics.
Given a normal logic program, an ASP solver finds the set of answer sets for that program. 
Modern ASP solvers can also be used to solve optimization problems by the introduction of weak constraints \cite{calimeri2012asp}.
A \define{weak constraint} is a rule that defines the cost of a certain tuple of atoms. 
Given a program with weak constraints, an ASP solver can find a preferred answer set with the lowest cost.

\section{A computational framework for making sense of sensory sequences}
\label{sec:apperception-framework}

What does it mean to make sense of a sensory sequence?
In this section, we formalize what this means, before describing our computer implementation.
We assume that the sensor readings have already been discretized into ground atoms of first-order logic, so a sensory reading featuring sensor $a$ can be represented by a ground atom $p(a)$ for some unary predicate $p$, or by an atom $r(a,b)$ for some binary relation $r$ and unique value $b$.\footnote{We restrict our attention to unary and binary predicates. This restriction can be made without loss of generality, since every $k$-ary relationship can be expressed as $k+1$ binary relationships \cite{kowalski1979logic}.}

\defbegin{}
\label{def:discrete-sequence}
An \define{unambiguous symbolic sensory sequence} is a sequence of sets of ground atoms.
Given a sequence $S = (S_1, S_2, ...)$, every \define{state} $S_t$ in $S$ is a set of ground atoms, representing a \emph{partial} description of the world at a discrete time step $t$.
An atom $p(a) \in S_t$ represents that sensor $a$ has property $p$ at time $t$.
An atom $r(a, b) \in S_t$ represents that sensor $a$ is related via relation $r$ to value $b$ at time $t$.
If $\mathcal{G}$ is the set of all ground atoms, then $S \in \left(2^\mathcal{G}\right)^*$.

\label{def:sensory-sequence}
\defend{}

\begin{example}
\label{ex:sensory-sequence}
Consider, the following sequence $S_{1:10}$.
Here there are two sensors $a$ and $b$, and each sensor can be either $\mathit{on}$ or $\mathit{off}$.
\begin{eqnarray*}
\begin{tabular}{lllll}
$S_1 = \set{}$  & $S_2 = \set{ \mathit{off}(a), \mathit{on}(b)}$ & $S_3 = \set{ \mathit{on}(a),  \mathit{off}(b)}$ & 
$S_4 = \set{ \mathit{on}(a), \mathit{on}(b)}$ & $S_5 = \set{ \mathit{on}(b)}$ \\ 
$S_6 = \set{ \mathit{on}(a),  \mathit{off}(b)}$ & $S_7 = \set{ \mathit{on}(a), \mathit{on}(b)}$ & $S_8 = \set{ \mathit{off}(a), \mathit{on}(b)}$ & $S_9 = \set{ \mathit{on}(a)}$ & $S_{10} = \set{ }$
\end{tabular}
\end{eqnarray*}
There is no expectation that a sensory sequence contains readings for all sensors at all time steps.
Some of the readings may be missing.
In state $S_5$, we are missing a reading for $a$, while in state $S_9$, we are missing a reading for $b$.
In states $S_1$ and $S_{10}$, we are missing sensor readings for both $a$ and $b$.
\label{ex:sensory-sequence}
\end{example}

The central idea is to make sense of a sensory sequence by \emph{constructing a unified theory that explains that sequence}. The key notions, here, are ``theory'', ``explains'', and ``unified''. We consider each in turn.

\subsection{The theory}
\label{sec:theory}

Theories are defined in a new language, \logic{}, designed for modelling dynamics.
In this language, one can describe how facts change over time by writing a causal rule stating that if the antecedent holds at the current time-step, then the consequent holds at the \emph{next} time-step.
Additionally, our language includes a frame axiom allowing facts to persist over time: each atom remains true at the next time-step unless it is overridden by a new fact which is incompatible with it.
Two facts are incompatible if there is a \emph{constraint} that precludes them from both being true. Thus, \logic{} extends Datalog with causal rules and constraints.

\defbegin{}
\label{def:theory}
A \define{theory} is a four-tuple $(\phi, I, R, C)$ of \logic{} elements where:
\begin{itemize}
\item
$\phi$ is a type signature specifying the types of constants, variables, and arguments of predicates
\item
$I$ is a set of initial conditions
\item
$R$ is a set of rules describing the dynamics
\item
$C$ is a set of constraints
\end{itemize}
\defend{}
We shall consider each element in turn, starting with the type signature.
\defbegin{}
\label{def:type-signature}
Given a set $\mathcal{T}$ of types, a set $\mathcal{O}$ of constants representing individual objects, and a set $\mathcal{P}$ of predicates representing properties and relations, let $\mathcal{G}$ be the set of all ground atoms formed from $\mathcal{T}$, $\mathcal{O}$, and $\mathcal{P}$.
Given a set $\mathcal{V}$ of variables, let $\mathcal{U}$ be the set of all unground atoms formed from $\mathcal{T}$, $\mathcal{V}$, and $\mathcal{P}$.

A \define{type signature} is a tuple $(T, O, P, V)$ where $T \subseteq \mathcal{T}$ is a finite set of types, $O \subseteq \mathcal{O}$ is a finite set of constants representing objects, $P \subseteq \mathcal{P}$ is a finite set of predicates representing properties and relations, and $V \subseteq \mathcal{V}$ is a finite set of variables.
We write $\kappa_O : O \rightarrow T$ for the type of an object, $\kappa_P : P \rightarrow T^*$ for the types of the predicate's arguments, and $\kappa_V : V \rightarrow T$ for the type of a variable.
\defend{}
Now some type signatures are suitable for some sensory sequences, while others are unsuitable, because they do not contain the right constants and predicates. The following definition formalizes this:
\defbegin{}
\label{def:suitable-type-signature}
Let $G_S = \bigcup_{t \geq 1} S_t$ be the set of all ground atoms that appear in sensory sequence $S = (S_1, ...)$.
Let $G_\phi$ be the set of all ground atoms that are well-typed according to type signature $\phi$.
If $\phi = (T, O, P, V)$ then $G_\phi = \{ p(a_1, ..., a_n) \mid p \in P, \kappa_P(p) = (t_1, ..., t_n), a_i \in O, \kappa_O(a_i) = t_i \; \mbox{for all} \; i = 1 .. n \}$.
A type signature $\phi$ is \define{suitable} for a sensory sequence $S$ if all the atoms in $S$ are well-typed according to signature $\phi$, i.e. $G_S \subseteq G_\phi$.
\defend{}

Next, we define the set of unground atoms for a particular type signature.
\defbegin{}
\label{def:unground}
Let $U_\phi$ be the set of all \define{unground} atoms that are well-typed according to signature $\phi$.
If $\phi = (T, O, P, V)$ then $U_\phi = \{ p(v_1, ..., v_n) \mid p \in P, \kappa_P(p) = (t_1, ..., t_n), v_i \in V, \kappa_V(v_i) = t_i \; \mbox{for all} \; i = 1 .. n \}$.
Note that, according to this definition, an atom is unground if \emph{all} its terms are variables.
Note that ``unground'' means more than simply not ground. For example, $p(a, X)$ is neither ground nor unground.
\defend{}

\begin{example}
\label{ex:type-signature}
One suitable type signature for the sequence of Example \ref{ex:sensory-sequence} is $(T, O, P, V)$, consisting of types 
$T = \set{s}$, objects $O = \set{a{:}s, b{:}s}$, predicates $P = \set{\mathit{on}(s), \mathit{off}(s)}$, and variables
$V = \set{X{:}s, Y{:}s}$.
Here, and throughout, we write $a{:}s$ to mean that object $a$ is of type $s$, $\mathit{on}(s)$ to mean that unary predicate $\mathit{on}$ takes one argument of type $s$, and $X{:}s$ to mean that variable $X$ is of type $s$.
The unground atoms are $U_\phi = \set{\mathit{on}(X), \mathit{off}(X),\mathit{on}(Y), \mathit{off}(Y)}$.
There are, of course, an infinite number of other suitable signatures.
\end{example}

\defbegin{}
\label{def:initial-conditions}
The \define{initial conditions} $I$ of a theory $(\phi, I, R, C)$ is a set of ground atoms from $G_\phi$ representing a partial description of the facts true at the initial time step.
\defend{}
The initial conditions are needed to specify the initial values of the latent unobserved information. 
Some systems (e.g. LFIT \cite{inoue2014learning}) define a predictive model without using a set $I$ of initial conditions. These systems are able to avoid positing initial conditions because they do not use latent unobserved information. But any system that does invoke latent information beneath the surface of the sensory stimulations must also define the initial values of the latent information.

The rules define the dynamics of the theory:
\defbegin{}
\label{def:rules}
There are two types of rule in $\logic{}$.
A \define{static rule} is a definite clause of the form $\alpha_1 \wedge ... \wedge \alpha_n \rightarrow \alpha_0$, where $n \geq 0$ and each $\alpha_i$ is an \emph{unground} atom from $U_\phi$ consisting of a predicate and a list of variables.
Informally, a static rule is interpreted as: if conditions $\alpha_1, ... \alpha_n$ hold at the current time step, then $\alpha_0$ also holds at that time step.
A \define{causal rule} is a clause of the form $\alpha_1 \wedge ... \wedge \alpha_n \fork \alpha_0$, where $n \geq 0$ and each $\alpha_i$ is an unground atom from $U_\phi$.
A causal rule expresses how facts change over time. 
Rule $\alpha_1 \wedge ... \wedge \alpha_n \fork \alpha_0$ states that if conditions $\alpha_1, ... \alpha_n$ hold at the current time step, then $\alpha_0$ holds at the \emph{next} time step.
\defend{}
All variables in rules are implicitly universally quantified.
So, for example, $\mathit{on}(X) \fork \mathit{off}(X)$ states that for all objects $X$, if $X$ is currently $\mathit{on}$, then $X$ will become  $\mathit{off}$ at the next-time step.

The constraints rule out certain combinations of atoms\footnote{Note that exclusive disjunction between atoms $p_1(X), ..., p_n(X)$ is different from \emph{xor} between the $n$ atoms. 
The \emph{xor} of $n$ atoms is true if an \emph{odd} number of the atoms hold, while the exclusive disjunction is true if \emph{exactly one} of the atoms holds.
We write $p_1(X) \oplus ... \oplus p_n(X)$ to mean exclusive disjunction between $n$ atoms, not the application of $n-1$ \emph{xor} operations.}:
\defbegin{}
\label{def:constraints}
There are three types of constraint in \logic{}.
A \define{unary constraint} is an expression of the form
$\forall X, p_1(X) \oplus ... \oplus p_n(X)$, where $n > 1$, meaning that for all $X$, exactly one of $p_1(X), ..., p_n(X)$ holds.
A \define{binary constraint} is an expression of the form
$\forall X, \forall Y, r_1(X, Y) \oplus ... \oplus r_n(X, Y)$ where $n > 1$, meaning that for all objects $X$ and $Y$, exactly one of the binary relations hold.
A \define{uniqueness constraint} is an expression of the form
$\forall X, \exists ! Y{:}t_2, r(X, Y)$,
which means that for all objects $X$ of type $t_1$ there exists a \emph{unique} object $Y$ such that $r(X, Y)$.
\defend{}

Note that the rules and constraints are constructed entirely from \emph{unground} atoms. 
Disallowing constants prevents special-case rules that apply to particular objects, and forces the theory to be general.\footnote{This restriction also occurs in some ILP systems \cite{inoue2014learning,evans2018learning}.}


\subsection{Explaining the sensory sequence}
\label{sec:explaining}

A theory explains a sensory sequence if the theory generates a trace\footnote{In \cite{inoue2014learning}, the trace is called the \emph{orbit}.} that covers that sequence.
In this section, we explain the trace and the covering relation.

\defbegin{}
\label{def:trace}
Every theory $\theta = (\phi, I, R, C)$ generates an infinite sequence $\tau(\theta)$ of sets of ground atoms, called the \define{trace} of that theory.
Here, $\tau(\theta) = (A_1, A_2, ...)$, where each $A_t$ is the smallest set of atoms satisfying the following conditions:
\begin{itemize}
\item 
$I \subseteq A_1$ 
\item
If there is a static rule $\beta_1 \wedge ... \wedge \beta_m \rightarrow \alpha$ in $R$ and a ground substitution $\sigma$ such that $A_t$ satisfies $\beta_i \sigma$ for each antecedent $\beta_i$, then $\alpha \sigma \in A_t$
\item
If there is a causal rule $\beta_1 \wedge ... \wedge \beta_m \fork \alpha$ in $R$ and a ground substitution $\sigma$ such that $A_{t-1}$ satisfies $\beta_i \sigma$ for each antecedent $\beta_i$, then $\alpha \sigma \in A_{t}$
\item
\emph{Frame axiom}: if $\alpha$ is in $A_{t-1}$ and there is no atom in $A_t$ that is incompossible with $\alpha$ w.r.t constraints $C$, then $\alpha \in A_t$. 
Two ground atoms are \define{incompossible} if there is some constraint $c$ in $C$ and some substitution $\sigma$ such that the ground constraint $c \sigma$ precludes both atoms being true.
\end{itemize}
\defend{}
The frame axiom is a simple way of providing inertia: a proposition continues to remain true until something new comes along which is incompatible with it.
Including the frame axiom makes our theories much more concise: instead of needing rules to specify all the atoms which remain the same, we only need rules that specify the atoms that change. 

Note that the state transition function is deterministic: $A_{t}$ is uniquely determined by $A_{t-1}$.

\begin{theorem}
\label{cycles}
The trace of every theory repeats after some finite number of steps.
For any theory $\theta$, there exists a $k$ such that $\tau(\theta) = (A_1, ..., A_{k-1}, A_k, A_{k+1}, ...)$ and for all $i \geq 0$, $A_i = A_{k+i}$.
\end{theorem}
\begin{proof}
Since the set $G_\phi$ of ground atoms is finite, there must be a $k$ such that $A_1 = A_k$.
The proof proceeds by induction on $i$. 
If $i=0$, the proof is trivial.
When $i > 0$, note that the trace function $\tau$ satisfies the Markov condition that the next state $A_{t+1}$ depends only on the current state $A_t$, and not on any earlier states. Hence if $A_i = A_{i+k}$, then $A_{i+1} = A_{i+k+1}$.
\end{proof}

One important consequence of Theorem \ref{cycles} is:
\begin{theorem}
\label{thm:decidable}
Given a theory $\theta$ and a ground atom $\alpha$, it is decidable whether $\alpha$ appears somewhere in the infinite trace $\tau(\theta)$.
\end{theorem}
\begin{proof}
Let $\tau(\theta)$ be the infinite sequence $(A_1, A_2, ...)$.
From Theorem \ref{cycles}, the trace must repeat after $k$ time steps.
Thus, to check whether ground atom $\alpha$ appears somewhere in $\tau(\theta)$, it suffices to test if $\alpha$ appears in $A_1, ..., A_k$.
\end{proof}

Next we define what it means for a theory to ``explain'' a sensory sequence.
\defbegin{}
Given finite sequence $S = (S_1, ..., S_T)$ and (not necessarily finite) $S'$, $S \sqsubseteq S'$ if $S' = (S'_1, S'_2, ...)$ and 
$S_i \subseteq S'_i$ for all $1 \leq i \leq T$.
If $S \sqsubseteq S'$, we say that $S$ is \define{covered by} $S'$, or that $S'$ covers $S$.
A theory $\theta$ \define{explains} a sensory sequence $S$ if the trace of $\theta$ covers $S$, i.e. $S \sqsubseteq \tau(\theta)$.
\label{def:explains}
\defend{}

In providing a theory $\theta$ that explains a sensory sequence $S$, we make $S$ intelligible by placing it within a bigger picture:
while $S$ is a scanty and incomplete description of a fragment of the time-series, $\tau(\theta)$ is a complete and determinate description of the whole time-series.

\begin{example}
We shall provide a theory to explain the sensory sequence $S$ of Example \ref{ex:sensory-sequence}.

Consider the type signature $\phi = (T, O, P, V)$, consisting of types $T = \set{s}$, objects $O = \set{a{:}s, b{:}s}$, predicates $P = \set{ \mathit{on}(s), \mathit{off}(s), p_1(s), p_2(s), p_3(s), r(s, s)}$, and variables $V=\set{X {:}s, Y {:}s}$.
Here, $\phi$ extends the type signature of Example \ref{ex:type-signature} by adding three unary predicates $p_1$, $p_2$, $p_3$, and one binary relation $r$.\footnote{Extended type signatures are generated by the machine, not by hand. Our computer implementation searches through the space of increasingly complex type signatures extending the original signature. This search process is described in Section \ref{sec:iterating}.}

Consider the theory $\theta = (\phi, I, R, C)$, where:
\begin{eqnarray*}
\begin{tabular}{lll}
$I = \bigsetbegin{}
p_1(b) \\
p_2(a) \\
r(a, b) \\
r(b, a) \\
\bigsetend{}$ &
$R = \bigsetbegin{}
p_1(X) \fork p_2(X) \\
p_2(X) \fork p_3(X) \\
p_3(X) \fork p_1(X) \\
p_1(X) \rightarrow \mathit{on}(X) \\
p_2(X) \rightarrow \mathit{on}(X) \\
p_3(X) \rightarrow \mathit{off}(X)
\bigsetend{}$ &
$C = \bigsetbegin{}
\forall X {:}s, \; \mathit{on}(X) \oplus \mathit{off}(X) \\
\forall X {:}s, \; p_1(X) \oplus p_2(X) \oplus p_3(X) \\
\forall X {:}s, \; \exists ! Y {:}s \; r(X, Y)
\bigsetend{}$
\end{tabular}
\end{eqnarray*}

The infinite trace $\tau(\theta) = (A_1, A_2, ...)$ for theory $\theta$ begins with:
\begin{eqnarray*}
\begin{tabular}{ll}
$A_1 = \{ \mathit{on}(a), \mathit{on}(b), p_2(a), p_1(b), r(a, b), r(b, a) \}$ &
$A_2 = \{ \mathit{off}(a), \mathit{on}(b), p_3(a), p_2(b), r(a, b), r(b, a) \}$ \\
$A_3 = \{ \mathit{on}(a), \mathit{off}(b), p_1(a), p_3(b), r(a, b), r(b, a) \}$ &
$A_4 = \{ \mathit{on}(a), \mathit{on}(b), p_2(a), p_1(b), r(a, b), r(b, a) \}$ \\ 
\dotso 
\end{tabular}
\end{eqnarray*}
Note that the trace repeats at step 4. In fact, it is always true that the trace repeats after some finite set of time steps.

Theory $\theta$ explains the sensory sequence $S$ of Example \ref{ex:sensory-sequence}, since the trace $\tau(\theta)$ covers $S$.
Note that $\tau(\theta)$ ``fills in the blanks'' in the original sequence $S$, both predicting final time step 10, retrodicting   initial time step 1, and imputing missing values for time steps 5 and 9.
\label{ex:explains}
\end{example}

\subsection{Unifying the sensory sequence}
\label{sec:unity-conditions}

Next, we proceed from explaining a sensory sequence to ``making sense'' of that sequence.
In order for $\theta$ to make sense of $S$, it is \emph{necessary} that $\tau(\theta)$ covers $S$.
But this condition is not, on its own, \emph{sufficient}.
The extra condition that is needed for $\theta$ to count as ``making sense'' of $S$ is for $\theta$ to be \emph{unified}.
We require that the constituents of the theory are integrated into a coherent whole.
A trace $\tau(\theta)$ of theory $\theta$ is a (i) sequence of (ii) sets of ground atoms composed of (iii) predicates and (iv) objects.
For the theory $\theta$ to be unified is for unity to be achieved at each of these four levels:

\defbegin{}
\label{def:unity}
A theory $\theta$ is \define{unified} if each of the four conditions hold:
\begin{enumerate}
\item
Objects are united in space (see Section \ref{sec:spatial-unity})
\item
Predicates are united via constraints (see Section \ref{sec:conceptual-unity})
\item
Ground atoms are united into states by jointly respecting constraints and static rules (see Section \ref{sec:static-unity})
\item
States are united into a sequence by causal rules (see Section \ref{sec:temporal-unity})
\end{enumerate}
\defend{}

\subsubsection{Spatial unity}
\label{sec:spatial-unity}

\defbegin{}
\label{def:spatial-unity}
A theory $\theta$ satisfies \define{spatial unity} if for each state $A_t$ in $\tau(\theta) = (A_1, A_2, ...)$, for each pair $(x, y)$ of distinct objects, $x$ and $y$ are connected via a chain of binary atoms $\{r_1(x, z_1), r_2(z_1, z_2), ... r_n(z_{n-1}, z_n), r_{n+1}(z_n, y)\} \subseteq A_t$.
\defend{}
If this condition is satisfied, it means that given any object, we can get to any other object by hopping along relations. 
Everything is connected, even if only indirectly.

Note that this notion of spatial unity is rather abstract: the requirement is only that every pair of objects are indirectly connected via some chain of binary relations. Although some of these binary relations might be spatial relations (e.g. ``left-of''), they need not all be.
The requirement is only that every pair of objects are connected via some chain of binary relations; it does not insist that each binary relation has a specifically ``spatial'' interpretation.

\subsubsection{Conceptual unity}
\label{sec:conceptual-unity}

A theory satisfies conceptual unity if every predicate is involved in some constraint, either exclusive disjunction ($\oplus)$ or unique existence ($\exists !$).
The intuition here is that constraints combine predicates into clusters of mutual incompatibility.
\defbegin{}
\label{def:conceptual-unity}
A theory $\theta = (\phi, I, R, C)$ satisfies \define{conceptual unity} if for each unary predicate $p$ in $\phi$, there is some xor constraint in $C$ of the form $\forall X{:}t, p(X) \oplus q(X) \oplus ...$ containing $p$;
and, for each binary predicate $r$ in $\phi$, there is some xor constraint in $C$ of the form $\forall X{:}t_1, \forall Y{:}t_2, r(X, Y) \oplus s(X, Y) \oplus ...$ or some $\exists!$ constraint in $C$ of the form $\forall X{:}t, \exists ! Y{:}t_2, r(X, Y)$.
\defend{}
To see the importance of this, observe that if there are no constraints, then there are no exhaustiveness or exclusiveness relations between atoms. 
An xor constraint e.g.~$\forall X{:}t, \mathit{on}(X) \oplus \mathit{off}(X)$ both rules out the possibility that an object is simultaneously $\mathit{on}$ and $\mathit{off}$ (exclusiveness) and also rules out the possibility that an object of type $t$ is neither $\mathit{on}$ nor $\mathit{off}$ (exhaustiveness).
It is exhaustiveness which generates states that are \emph{determinate}, in which it is guaranteed every object of type $t$ is e.g.~either $\mathit{on}$ or $\mathit{off}$.
It is exclusiveness which generates \emph{incompossibility} between atoms, e.g.~that $\mathit{on}(a)$ and $\mathit{off}(a)$ are incompossible. Incompossibility, in turn, is needed to constrain the scope of the frame axiom (see Definition \ref{def:trace} above). 
Without incompossibility, \emph{all} atoms from the previous time-step would be transferred to the next time-step, and the set of true atoms in the sequence $(S_1, S_2, ...)$ would grow monotonically over time: $S_i \subseteq S_j$ if $i \leq j$, which is clearly unacceptable.
The purpose of the constraint of conceptual unity is to collect predicates into groups, to provide determinacy in each state, and to ground the incompossibility relation that constrains the way information persists between states.\footnote{A natural question to ask at this point is: why use exclusive disjunction to represent constraints? Why not instead represent constraints using strong negation or negation as failure\cite{baral1994logic}?
An exclusive disjunction can always be converted into a set of extended clauses representing the predicates' exclusiveness, and one normal clause representing their exhaustiveness.
For example, $\forall X: t, p(X) \oplus q(X)$ can be rendered as:
\begin{eqnarray*}
& & \asp{\neg p(X) \codeif q(X)} \\
& & \asp{\neg q(X) \codeif p(X)} \\
& & \asp{\codeif \naf p(X), \naf q(X), t(X) }
\end{eqnarray*}
In general, if we have an exclusive disjunction featuring $n$ predicates, we can turn this into $n * (n-1)$ clauses (using strong negation) to capture the exclusiveness of the $n$ predicates, and one clause (using negation as failure) to capture the exhaustiveness.

The exclusive disjunction constraint is a compact way of representing a lot of information about the connection between predicates.
Although the exclusive disjunction constraint can always be translated into a set of clauses (using both negation as failure and strong negation), the representation using exclusive disjunction is much more compact.

One reason, then, for expressing the constraint as an exclusive disjunction is that it is a significantly more compact representation than the representation using negation as failure. 
But another, more substantial reason is that it means we can avoid the complexities involved in the semantics if we added negation as failure to our target language \logic{}.
There are various semantics for normal logic programs that include negation as failure (e.g. Clark completion \cite{clark1978negation}, stable model semantics \cite{gelfond1988stable}, well-founded models \cite{van1991well}), but each of them introduces significant additional complexities when compared with the least model of a definite logic program:
the Clark completion is not always consistent (does not always have a model), the stable model semantics assigns the meaning of a normal logic program to a \emph{set} of models rather than a single model, and the well-founded model uses a 3-valued logic where atoms can be true, false, or undefined.
Thus, the main reason for expressing constraints using exclusive disjunction (rather than using negation as failure) is to restrict the rules to definite rules and avoid the complexities of the various semantics of normal logic programs.
(Although we do plan to extend our rules to include stratified negation, as this does not complicate the semantics in the same way that unrestricted negation does).
The inner loop of our program synthesis system is the calculation of the trace $\tau(\theta)$ by executing a \logic{} program, so it is essential that the execution is as efficient as possible. 
Hence our strong preference for definite logic programs over normal logic programs.

Why do we not allow more complex constraints (e.g. allowing any first-order sentence to be a constraint)? If we allowed any arbitrary set of first-order formulas as constraints, then computing the incompossibility relation would become much harder, given that computing entailment in first-order logic is only semi-decidable.
The reason, then, why we focus on xor constraints is that they are the simplest construct that generates the incompossibility relation needed to constrain the frame axiom.
}

\subsubsection{Static unity}
\label{sec:static-unity}

In our effort to interpret the sensory sequence, we construct various ground atoms.
These need to be grouped together, somehow, into states (sets of atoms).
But what determines how these atoms are grouped together into states?

Treating a set $A$ of ground atoms as a state is (i) to insist that $A$ satisfies all the constraints in $C$ and (ii) to insist that $A$ is closed under the static rules in $R$:
\defbegin{}
\label{def:static-unity}
A theory $\theta = (\phi, I, R, C)$ satisfies \define{static unity} if every state $(A_1, A_2, ...)$ in $\tau(\theta)$ satisfies all the constraints in $C$ and is closed under the static rules in $R$.
\defend{}
Static unity is an uncontroversial requirement and is used in other ILP systems \cite{goodman2011learning,ullman2012theory}.
Note that, from the definition of the trace in Definition \ref{def:trace}, all the states in $\tau(\theta)$ are automatically closed under the static rules in $R$.

\subsubsection{Temporal unity}
\label{sec:temporal-unity}

Given a set of states, we need to unite these elements in a \emph{sequence}. 
According to the fourth and final condition of unity, the only thing that can unite states in a sequence is a set of \emph{causal rules}.
These causal rules are \emph{universal} in two senses: they apply to all object tuples, and they apply at all times.
A causal rule $\alpha_1 \wedge ... \wedge \alpha_n \fork \alpha_0$ fixes the temporal relation between the atoms $\alpha_1, ..., \alpha_n$ (which are true at $t$) and the atom $\alpha_0$ (which is true at $t+1$):
\defbegin{}
\label{def:temporal-unity}
A sequence $(A_1, A_2, ...)$ of states satisfies \define{temporal unity} with respect to a set $R_{\fork}$ of causal rules if, for each  $\alpha_1 \wedge ... \wedge \alpha_n \fork \alpha_0$ in $R_{\fork}$, for each ground substitution $\sigma$, for each time-step $t$, if $\{\alpha_1 \sigma, ..., \alpha_n \sigma\} \subseteq A_t$ then $\alpha_0 \sigma \in A_{t+1}$.
\defend{}
Temporal unity is an uncontroversial requirement and is also used in other ILP systems such as LFIT \cite{inoue2014learning}.
Note that, from the definition of the trace in Definition \ref{def:trace}, the trace $\tau(\theta)$ \emph{automatically} satisfies temporal unity.

\subsubsection{The four conditions of unity}
\label{sec:four-unity-conditions}

To recap, the trace of a theory is a sequence of sets of atoms. The four types of element are objects, predicates, sets of atoms, and sequences of sets of atoms.
Each of the four types of element has its own form of unity:
\begin{enumerate}
\item
\emph{Spatial unity}: objects are united in space by being connected via chains of relations
\item
\emph{Conceptual unity}: predicates are united by constraints
\item
\emph{Static unity}: atoms are united in a state by jointly satisfying constraints and static rules
\item
\emph{Temporal unity}: states are united in a sequence by causal rules
\end{enumerate}


Since temporal unity is automatically satisfied from the definition of a trace in Definition \ref{def:trace}, 
we are left with only three unity conditions that need to be explicitly checked: spatial unity, conceptual unity, and static unity.
A trace partially satisfies static unity since the static rules are automatically enforced by Definition \ref{def:trace}; but the constraints are not necessarily satisfied.

Note that both checking spatial unity and checking static unity require checking every time-step, and the trace is infinitely long.
However, as long as the trace repeats at some point, Theorem \ref{cycles} ensures that we need only check the finite portion of the trace until we find the first repetition (the first $k$ such that $A_1 = A_k$ where $\tau(\theta) = (A_1, ...)$). 

\subsection{Making sense}

Now we are ready to define the central notion of ``making sense'' of a sequence.
\defbegin{}
\label{def:makes-sense}
A theory $\theta$ \define{makes sense} of a sensory sequence $S$ if $\theta$ explains $S$, i.e. $S \sqsubseteq \tau(\theta)$, and $\theta$ satisfies the four conditions of unity of Definition \ref{def:unity}.
If $\theta$ makes sense of $S$, we also say that $\theta$ is a \define{unified interpretation} of $S$.
\defend{}

\begin{example}
The theory $\theta$ of Example \ref{ex:explains} satisfies the four unity conditions since:
\begin{enumerate}
\item
For each state $A_i$ in $\tau(\theta)$, $a$ is connected to $b$ via the singleton chain $\{r(a, b)\}$, and $b$ is connected to $a$ via $\{r(b, a)\}$.
\item
The predicates of $\theta$ are $\mathit{on}, \mathit{off}, p_1, p_2, p_3, r$. 
Here, $\mathit{on}$ and $\mathit{off}$ are involved in the constraint $\forall X {:}s, \; \mathit{on}(X) \oplus \mathit{off}(X)$, 
while $p_1, p_2, p_3$ are involved in the constraint $\forall X {:}s, \; p_1(X) \oplus p_2(X) \oplus p_3(X)$, and $r$ is involved in the constraint $\forall X {:}s, \; \exists ! Y {:}s \; r(X, Y)$.
\item
Let $\tau(\theta) = (A_1, A_2, A_3, A_4, ...)$. It is straightforward to check that $A_1$, $A_2$, and $A_3$ satisfy each constraint in $C$. Observe that $A_4$ repeats $A_1$, thus Theorem \ref{cycles} ensures that we do not need to check any more time steps. 
\item
Temporal unity is automatically satisfied by the definition of the trace $\tau(\theta)$ in Definition \ref{def:trace}.
\end{enumerate}
Hence, $\theta$ makes sense of sensory sequence $S$ of Example \ref{ex:sensory-sequence}, since $S \sqsubseteq \tau(\theta)$ (Example \ref{ex:explains}) and $\theta$ also satisfies the four conditions of unity.
\end{example}

In our search for interpretations that make sense of sensory sequences, we are particularly interested in \emph{parsimonious} interpretations. 
To this end, we define the cost of a theory\footnote{Note that this simple measure of cost does not depend on the constraints in $C$ or the type signature $\phi$.
There are various alternative more complex definitions of $\mathit{cost}$. 
We could, for example, use the Kolmogorov complexity \cite{kolmogorov1963tables} of $\theta$: the size of the smallest program that can generate $\theta$. 
Or we could use Levin complexity \cite{levin1973universal} and also take into account the log of the computation time needed to generate $\tau(\theta)$, up to the point where the trace first repeats.}:

\defbegin{}
\label{def:theory-cost}
Given a theory $\theta = (\phi, I, R, C)$, the \define{cost} of $\theta$ is
\begin{eqnarray*}
|I| + \sum \bigsetbegin{} n+1 \mid \alpha_1 \wedge ... \wedge \alpha_n \circ \alpha_0 \in R, \circ \in \set{\rightarrow, \fork} \bigsetend{}
\end{eqnarray*}
Here, $\mathit{cost}(\theta)$ is just the total number of ground atoms in $I$ plus the total number of unground atoms in the rules of $R$.
\defend{}

The key notion of this section is the discrete apperception task. 
\defbegin{}
\label{def:apperception-task}
The input to an apperception task is a triple $(S, \phi, C)$ consisting of a sensory sequence $S$, a suitable type signature $\phi$, and a set $C$ of (well-typed) constraints such that (i) each predicate in $S$ appears in some constraint in $C$ and (ii) $S$ can be extended to satisfy $C$: there exists a sequence $S'$ covering $S$ such that each state in $S'$ satisfies each constraint in $C$.

Given such an input triple $(S, \phi, C)$, the \define{discrete apperception task} is to find the lowest cost theory $\theta = (\phi', I, R, C')$ such that $\phi'$ extends $\phi$, $C' \supseteq C$, and $\theta$ makes sense of $S$.
\label{def:simple-apperception-task}
\defend{}

Note that the input to an apperception task is more than just a sensory sequence $S$. It also contains a type signature $\phi$ and a set $C$ of constraints.
A natural question at this point is: why not simply let the input to an apperception task be just the sequence $S$, and ask the system to produce some theory $\theta$ satisfying the unity conditions such that $S \sqsubseteq \tau(\theta)$? 
The reason that the input needs to contain types $\phi$ and constraints $C$ to supplement $S$ is that otherwise the task is severely under-constrained, as the following example shows.

\begin{example}
Suppose our sequence is $S = (\set{\mathit{on}(a)}, \set{\mathit{off}(a)}, \set{\mathit{on}(a)}, \set{\mathit{off}(a)}, \set{\mathit{on}(a)}, \set{\mathit{off}(a)})$.
If we are not given any constraints (such as $\forall X : t, \mathit{on}(X) \oplus \mathit{off}(X)$), 
if we are free to construct any $\phi$ and any set $C$ of constraints, then the following interpretation $\theta = (\phi, I, R, C)$ will suffice, where $\phi = (T, O, P, V)$, consisting of types
$T = \set{t}$, objects $O = \set{a{:}t}$, predicates $P = \set{ \mathit{on}(t), \mathit{off}(t), p(t), q(t)}$, and variables $V = \set{X {:}t}$, and suppose that $I, R, C$ are defined as:
\begin{eqnarray*}
\begin{tabular}{lll}
$I = \bigsetbegin{}
\mathit{on}(a) \\
\mathit{off}(a)
\bigsetend{}$ &
$R = \bigsetbegin{}
\bigsetend{}$ &
$C = \bigsetbegin{}
\forall X {:}t, \; \mathit{on}(X) \oplus p(X) \\
\forall X {:}t, \; \mathit{off}(X) \oplus q(X)
\bigsetend{}$
\end{tabular}
\end{eqnarray*}
Here we have introduced two latent predicates $p$ and $q$ which are incompatible with $\mathit{on}$ and $\mathit{off}$ respectively.
But in this interpretation, $\mathit{on}$ and $\mathit{off}$ are not incompatible with each other, so the degenerate interpretation (where both $\mathit{on}$ and $\mathit{off}$ are true at all times) is acceptable.
This shows the need for including constraints on the input predicates as part of the \emph{task formulation}.

More generally, for \emph{any} sensory sequence $(S_1, ..., S_T)$ featuring predicates $p_1, ..., p_n$, but no constraints between $p_1, ..., p_n$, we can always construct a degenerate interpretation by adding new predicates $q_1, ..., q_n$ with an xor constraint $\forall X: t, p_i(X) \oplus q_i(X)$ between each predicate $p_i$ and the corresponding new predicate $q_i$. In the degenerate interpretation, the initial conditions $I$ are $S_1 \cup ... \cup S_T$, and the rules $R$ are empty. This shows that, without constraints on the predicates appearing in the initial sequence, the problem is underspecified.
\end{example}

The apperception task can be generalized to the case where we are given as input, not a single sensory sequence $S$, but a set of $m$ such sequences.
\defbegin{}
\label{def:generalized-apperception-task}
Given a set $\set{S^1, ..., S^m}$ of sensory sequences, a type signature $\phi$ and constraints $C$ such that each $(S^i, \phi, C)$ is a valid input to an apperception task as defined in Definition \ref{def:apperception-task}, the \define{generalized apperception task} is to find a lowest-cost theory $(\phi', \set{}, R, C')$ and sets $\set{I^1, ..., I^m}$ of initial conditions such that $\phi'$ extends $\phi$, $C' \supseteq C$, and for each $i = 1 .. m$, $(\phi', I^i, R, C')$ makes sense of $S^i$.
\defend{}

\subsection{Examples}

In this section, we provide a worked example of an apperception task, along with different unified interpretations.
We wish to highlight that there are always many alternative ways of interpreting a sensory sequence, each with different latent information (although some may have higher cost than others).

We continue to use our running example, the sensory sequence from Example \ref{ex:sensory-sequence}.
Here there are two sensors $a$ and $b$, and each sensor can be $\mathit{on}$ or $\mathit{off}$.
\begin{eqnarray*}
\begin{tabular}{lllll}
$S_1 = \set{}$  & $S_2 = \set{ \mathit{off}(a), \mathit{on}(b)}$ & $S_3 = \set{ \mathit{on}(a),  \mathit{off}(b)}$ & 
$S_4 = \set{ \mathit{on}(a), \mathit{on}(b)}$ & $S_5 = \set{ \mathit{on}(b)}$ \\ 
$S_6 = \set{ \mathit{on}(a),  \mathit{off}(b)}$ & $S_7 = \set{ \mathit{on}(a), \mathit{on}(b)}$ & $S_8 = \set{ \mathit{off}(a), \mathit{on}(b)}$ & $S_9 = \set{ \mathit{on}(a)}$ & $S_{10} = \set{ }$
\end{tabular}
\end{eqnarray*}
Let $\phi = (T, O, P, V)$ where $T = \{\mathit{sensor}\}$, $O = \{a, b\}$, $P = \{\mathit{on}(\mathit{sensor}), \mathit{off}(\mathit{sensor})\}$, $V = \{X {:}\mathit{sensor}\}$.
Let $C = \{\forall X {:}\mathit{sensor}, \; \mathit{on}(X) \oplus \mathit{off}(X)\}$.

Examples \ref{example:eca1}, \ref{example:eca2}, and \ref{example:eca3} below show three different unified interpretations of Example \ref{ex:sensory-sequence}.

\begin{example}
\label{example:eca1}
One possible way of interpreting Example \ref{ex:sensory-sequence} is as follows.
The sensors $a$ and $b$ are simple state machines that cycle between states $p_1$, $p_2$, and $p_3$.
Each sensor switches between $\mathit{on}$ and $\mathit{off}$ depending on which state it is in. 
When it is in states $p_1$ or $p_2$, the sensor is on; when it is in state $p_3$, the sensor is off.
In this interpretation, the two state machines $a$ and $b$ do not interact with each other in any way.
Both sensors are following the same state transitions.
The reason the sensors are out of sync is because they start in different states.

The type signature for this first unified interpretation is $\phi' = (T, O, P, V)$, where $T = \set{\mathit{sensor}}$, 
$O = \set{a{:}\mathit{sensor}, b{:}\mathit{sensor}}$, 
$P = \set{ \mathit{on}(\mathit{sensor}), \mathit{off}(\mathit{sensor}), r(\mathit{sensor}, \mathit{sensor}), p_1(\mathit{sensor}), p_2(\mathit{sensor}), p_3(\mathit{sensor})}$, and 
$V = \set{X {:}\mathit{sensor}, Y {:}\mathit{sensor}}$.
The three unary predicates $p_1$, $p_2$, and $p_3$ are used to represent the three states of the state machine.

Our first unified interpretation is the tuple $(\phi', I, R, C')$, where:
\begin{eqnarray*}
\begin{tabular}{lll}
$I = \bigsetbegin{}
p_2(a) \\
p_1(b) \\
r(a, b) \\
r(b, a) \\
\bigsetend{}$ &
$R = \bigsetbegin{}
p_1(X) \fork p_2(X) \\
p_2(X) \fork p_3(X) \\
p_3(X) \fork p_1(X) \\
p_1(X) \rightarrow \mathit{on}(X) \\
p_2(X) \rightarrow \mathit{on}(X) \\
p_3(X) \rightarrow \mathit{off}(X)
\bigsetend{}$ &
$C' = \bigsetbegin{}
\forall X {:}\mathit{sensor}, \; \mathit{on}(X) \oplus \mathit{off}(X) \\
\forall X {:}\mathit{sensor}, \; p_1(X) \oplus p_2(X) \oplus p_3(X) \\
\forall X {:}\mathit{sensor}, \; \exists ! Y {:}\mathit{sensor} \; r(X, Y)
\bigsetend{}$
\end{tabular}
\end{eqnarray*}
The update rules $R$ contain three causal rules (using $\fork$) describing how each sensor cycles from state $p_1$ to $p_2$ to $p_3$, and then back again to $p_1$.
For example, the causal rule $p_1(X) \fork p_2(X)$ states that if sensor $X$ satisfies $p_1$ at time $t$, then $X$ satisfies $p_2$ at time $t+1$. We know that $X$ is a sensor from the variable typing information in $\phi'$.
$R$ also contains three static rules (using $\rightarrow$) describing how the $\mathit{on}$ or $\mathit{off}$ attribute of a sensor depends on its state.
For example, the static rule $p_1(Y) \rightarrow \mathit{on}(X) $ states that if $X$ satisfies $p_1$ at time $t$, then $X$ also satisfies $\mathit{on}$ at time $t$.

The constraints $C'$ state that (i) every sensor is (exclusively) either $\mathit{on}$ or $\mathit{off}$, that every sensor is (exclusively) either $p_1$, $p_2$, or $p_3$, and that every sensor has exactly one sensor that is related by $r$ to it.
The binary $r$ predicate, or something like it, is needed to satisfy the constraint of spatial unity.

In this first interpretation, three new predicates are invented ($p_1$, $p_2$, and $p_3$) to represent the three states of the state machine.
In the next interpretation, we will introduce new invented objects instead of invented predicates.

Given the initial conditions $I$ and the update rules $R$, we can use our interpretation to compute which atoms hold at which time step.
In this case, $\tau(\theta) = (A_1, A_2, ...)$ where $S_i \sqsubseteq A_i$.
Note that this trace repeats: $A_i = A_{i+3}$.
We can use the trace to predict the future values of our two sensors at time step 10, since
$A_{10} = \set{\mathit{on}(a), \mathit{on}(b), r(a, b), r(b, a), p_2(a), p_1(b)}$.

As well as being able to predict future values, we can retrodict past values (filling in $A_1$), or interpolate intermediate unknown values (filling in $A_5$ or $A_9$).\footnote{This ability to ``impute'' intermediate unknown values is straightforward given an interpretation. Recent results show that current neural methods for sequence learning are more comfortable predicting future values than imputing intermediate values.}
But although an interpretation provides the resources to ``fill in'' missing data, it has no particular bias to predicting future time-steps. 
The conditions which it is trying to satisfy (the unity conditions of Section \ref{sec:unity-conditions}) do not explicitly insist that an interpretation must be able to predict future time-steps.
Rather, the ability to predict the future (as well as the ability to retrodict the past, or interpolate intermediate values) is a \emph{derived} capacity that emerges from the more fundamental capacity to ``make sense'' of the sensory sequence.

\end{example}

\begin{example}
\label{example:eca2}

There are always infinitely many different ways of interpreting a sensory sequence.
Next, we show a rather different interpretation of $S_{1:10}$ from that of Example \ref{example:eca1}.
In our second unified interpretation, we no longer see sensors $a$ and $b$ as self-contained state-machines.
Now, we see the states of the sensors as depending on their left and right neighbours. 
In this new interpretation, we no longer need the three invented unary predicates ($p_1$, $p_2$, and $p_3$), but instead introduce a new \emph{object}.

Object invention is much less explored than predicate invention in inductive logic programming. 
But Dietterich et al.~\cite{dietterich2008structured} anticipated the need for it, and Inoue \cite{inoue2016meta} uses meta-level abduction to posit unperceived objects.

Our new type signature $\phi' = (T, O, P, V)$, where
$T= \set{\mathit{sensor}}$, 
$O = \set{a{:}\mathit{sensor}, b{:}\mathit{sensor}, c{:}\mathit{sensor}}$, \\
$P = \set{ \mathit{on}(\mathit{sensor}), \mathit{off}(\mathit{sensor}), r(\mathit{sensor}, \mathit{sensor})}$, and 
$V = \set{X{:}\mathit{sensor}, Y{:}\mathit{sensor}}$.

In this new interpretation, imagine there is a one-dimensional cellular automaton with three cells, $a$, $b$, and (unobserved) $c$. 
The three cells wrap around: the right neighbour of $a$ is $b$, the right neighbour of $b$ is $c$, and the right neighbour of $c$ is $a$.
In this interpretation, the spatial relations are fixed. (We shall see another interpretation later where this is not the case).
The cells alternate between on and off according to the following simple rule: if $X$'s left neighbour is on (respectively off) at $t$, then $X$ is on (respectively off) at $t+1$.

Note that objects $a$ and $b$ are the two sensors we are given, but $c$ is a new unobserved latent object that we posit in order to make sense of the data. 
Many interpretations follow this pattern: new latent unobserved objects are posited to make sense of the changes to the sensors we are given.

Note further that part of finding an interpretation is constructing the spatial relation between objects; this is not something we are given, but something we must \emph{construct}. In this case, we posit that the imagined cell $c$ is inserted to the right of $b$ and to the left of $a$.

We represent this interpretation by the tuple $(\phi', I, R, C')$, where:
\begin{eqnarray*}
\begin{tabular}{lll}
$I = \bigsetbegin{}
\mathit{on}(a) \\
\mathit{on}(b) \\
\mathit{off}(c) \\
r(a, b) \\
r(b, c) \\
r(c, a)
\bigsetend{}$ & 
$R = \bigsetbegin{}
r(X, Y) \wedge \mathit{off}(X) \fork \mathit{off}(Y) \\
r(X, Y) \wedge \mathit{on}(X) \fork \mathit{on}(Y)
\bigsetend{} $ &
$C' = \bigsetbegin{}
\forall X{:}\mathit{sensor}, \; \mathit{on}(X) \oplus \mathit{off}(X) \\
\forall X{:}\mathit{sensor}, \; \exists ! Y{:}\mathit{sensor}, \; r(X, Y)
\bigsetend{}$
\end{tabular}
\end{eqnarray*}
Here, $\phi'$ extends $\phi$, $C'$ extends $C$, and the interpretation satisfies the unity conditions.
\end{example}

\begin{example}
\label{example:eca3}
We shall give one more way of interpreting the same sensory sequence, to show the variety of possible interpretations.

In our third interpretation, we will posit three latent cells, $c_1$, $c_2$, and $c_3$ that are distinct from the sensors $a$ and $b$.
Cells have static attributes: each cell can be either black or white, and this is a permanent unchanging feature of the cell.
Whether a sensor is on or off depends on whether the cell it is currently contained in is black or white.
The reason why the sensors change from on to off is because they \emph{move} from one cell to another.

Our new type signature $(T, O, P, V)$ distinguishes between cells and sensors as separate types:
$T = \set{\mathit{cell}, \mathit{sensor}}$, 
$O = \set{a: \mathit{sensor}, b: \mathit{sensor}, c_1: \mathit{cell}, c_2: \mathit{cell}, c_3: \mathit{cell}}$, 
$P = \set{ \mathit{on}(\mathit{sensor}), \mathit{off}(\mathit{sensor}), \mathit{part}(\mathit{sensor}, \mathit{cell}), r(\mathit{cell}, \mathit{cell}), \mathit{black}(\mathit{cell}), \mathit{white}(\mathit{cell})}$, and
$V = \set{X: \mathit{sensor}, Y: \mathit{cell}, Y_2: \mathit{cell}}$.
Our interpretation is the tuple $(\phi, I, R, C)$, where:
\begin{eqnarray*}
\begin{tabular}{lll}
$I = \bigsetbegin{}
\mathit{part}(a, c_1) \\
\mathit{part}(b, c_2) \\
r(c_1, c_2) \\
r(c_2, c_3) \\
r(c_3, c_1) \\
\mathit{black}(c_1) \\
\mathit{black}(c_2) \\
\mathit{white}(c_3)
\bigsetend{}$ &
$R = \bigsetbegin{}
\mathit{part}(X,Y) \wedge \mathit{black}(Y) \rightarrow \mathit{on}(X) \\
\mathit{part}(X,Y) \wedge \mathit{white}(Y) \rightarrow \mathit{off}(X) \\
r(Y,Y_2) \wedge \mathit{part}(X,Y_2) \fork \mathit{part}(X,Y)
\bigsetend{}$ &
$C = \bigsetbegin{}
\forall X{:}\mathit{sensor}, \; \mathit{on}(X) \oplus \mathit{off}(X) \\
\forall Y{:}\mathit{cell}, \; \mathit{black}(Y) \oplus \mathit{white}(Y) \\
\forall X{:}\mathit{sensor}, \; \exists ! Y: \mathit{cell}, \; \mathit{part}(X, Y) \\
\forall Y{:}\mathit{cell}, \; \exists ! Y_2: \mathit{cell}, \; r(Y, Y_2)
\bigsetend{}$
\end{tabular}
\end{eqnarray*}
The update rules $R$ state that the $\mathit{on}$ or $\mathit{off}$ attribute of a sensor depends on whether its current cell is black or white.
They also state that the sensors move from right-to-left through the cells.

In this interpretation, there is no state information in the sensors. 
All the variability is explained by the sensors moving from one static object to another.

Here, the sensors move about, so spatial unity is satisfied by different sets of atoms at different time-steps.
For example, at time-step 1, sensors $a$ and $b$ are indirectly connected via the ground atoms $\mathit{part}(a, c_1), r(c_1, c_2), \mathit{part}(b, c_2)$.
But at time-step 2, $a$ and $b$ are indirectly connected via a different set of ground atoms $\mathit{part}(a, c_3), r(c_3, c_1), \mathit{part}(b, c_1)$.
Spatial unity requires all pairs of objects to always be connected via some chain of ground atoms at each time-step, but it does not insist that it is the \emph{same} set of ground atoms at each time-step.
\end{example}

Examples \ref{example:eca1}, \ref{example:eca2}, and \ref{example:eca3} provide different ways of interpreting the same sensory input.
In Example \ref{example:eca1}, the sensors are interpreted as self-contained state machines. 
Here, there are no causal interactions between the sensors: each is an isolated machine.
In Examples \ref{example:eca2} and \ref{example:eca3}, by contrast, there are causal interactions between the sensors.
In Example \ref{example:eca2}, the $\mathit{on}$  and $\mathit{off}$ attributes move from left to right along the sensors.
In Example \ref{example:eca3}, it is the sensors that move, not the attributes, moving from right to left.
The difference between these two interpretations is in terms of what is moving and what is static.

Note that the interpretations of Examples \ref{example:eca1}, \ref{example:eca2}, and \ref{example:eca3} have costs 16, 12, and 17 respectively.
So the theory of Example \ref{example:eca2}, which invents an unseen object, is preferred to the other theories that posit more complex dynamics.

\subsection{Properties of interpretations}
\label{sec:theory}

In this section, we provide some results about the correctness and completeness of unified interpretations. 

\begin{theorem}
\label{soundness}
For each sensory sequence $S = (S_1, ..., S_t)$ and each unified interpretation $\theta$ of $S$, for each object $x$ that features in $S$ (i.e. $x$ appears in some ground atom $p(x)$ or $q(x,y)$ in some state $S_i$ in $S$), for each state $A_i$ in $\tau(\theta) = (A_1, A_2, ...)$, $x$ features in $A_i$.
In other words, if $x$ features in any state in $S$, then $x$ features in every state in $\tau(\theta)$.
\end{theorem}
\begin{proof}
Let $\theta = (\phi, I, R, C)$ and $\phi = (T, O, P, V)$.
Since object $x$ features in sequence $S$, there exists some atom $\alpha$ involving $x$ in some state $S_j$ in $(S_1, ..., S_t)$. 
Since $\theta$ is an interpretation, $S \sqsubseteq \tau(\theta)$, and hence $\alpha \in (\tau(\theta))_j$. Consider the two possible forms of $\alpha$:
\begin{enumerate}
\item
$\alpha = p(x)$. 
Since $\theta$ satisfies \emph{conceptual unity}, there must be a constraint involving $p$ of the form $\forall X : t, p(X) \oplus q_1(X) ... \oplus q_n(X)$ in $C$. Since $\phi$ is suitable for $S$, $x \in O$ and $\kappa_O(x) = t$.
Let $\tau(\theta) = (A_1, A_2, ...)$ and consider any $A_i$ in $\tau(\theta)$.
Since $\theta$ satisfies \emph{static unity}, $A_i$ satisfies each constraint in $C$ and in particular $A_i \models \forall X : t, p(X) \oplus q_1(X) ... \oplus q_n(X)$.
Since $\kappa_O(x) = t$, $A_i \models p(x) \oplus q_1(x) ... \oplus q_n(x)$.
Hence $\{p(x), q_1(x), ..., q_n(x)\} \cap A_i \neq \emptyset$ i.e. $x$ features in $A_i$.
\item
$\alpha = q(x, y)$ for some $y$.
Since $\theta$ satisfies \emph{conceptual unity}, there must be a constraint involving $q$. This constraint can either be (i) a binary constraint of the form $\forall X : t_1, \forall Y : t_2, q(X, Y) \oplus p_1(X, Y) \oplus ... \oplus p_n(X, Y)$ or (ii) a uniqueness constraint of the form $\forall X : t_1, \exists ! Y : t_2, q(X, Y)$.

Considering first case (i), since $\phi$ is suitable for $S$, $x, y \in O$, $\kappa_O(x) = t_1$, and $\kappa_O(y) = t_2$.
Again, let $\tau(\theta) = (A_1, A_2, ...)$ and consider any $A_i$ in $\tau(\theta)$.
Since $\theta$ satisfies \emph{static unity}, $A_i$ satisfies each constraint in $C$ and in particular $A_i \models \forall X : t_1, \forall Y : t_2, q(X, Y) \oplus p_1(X, Y) \oplus ... \oplus p_n(X, Y)$.
Since $\kappa_O(x) = t_1, \kappa_O(y) = t_2$, $A_i \models q(x, y) \oplus p_1(x, y) \oplus ... \oplus p_n(x, y)$.
Hence $\{q(x, y), p_1(x, y), ..., p_n(x, y)\} \cap A_i \neq \emptyset$ i.e. $x$ features in $A_i$.

For case (ii), again let $\tau(\theta) = (A_1, A_2, ...)$ and consider any $A_i$ in $\tau(\theta)$.
Since $\theta$ satisfies \emph{static unity}, $A_i$ satisfies each constraint in $C$ and in particular $A_i \models \forall X : t_1, \exists ! Y : t_2, q(X, Y)$.
Since $\kappa_O(x) = t_1$, $A_i \models  \exists ! Y : t_2, q(x, Y)$.
Therefore there must be some $y$ such that $\kappa_O(y) = t_2$ and $q(x,y) \in A_i$. \qedhere
\end{enumerate}
\end{proof}

Theorem \ref{soundness} provides some guarantee that admissible interpretations that satisfy the unity conditions will always be acceptable in the minimal sense that they always provide some value for each sensor. 
This theorem is important because it justifies the claim that a unified interpretation will always be able to support prediction (of future values), retrodiction (of previous values), and imputation (of missing values).

Note that this theorem does not imply that the predicate of the atom in which $x$ appears is one of the predicates appearing in the sensory sequence $S$.
It is entirely possible that it is some distinct predicate that appears in $\phi$ but has never been observed in $S$.
The following example illustrates this possibility.
\begin{example}
Suppose the sensory sequence is just $S = (\set{p(a)})$.
Suppose the type signature $(T, O, P, V)$ introduces another unary predicate $q$, i.e. $T = \set{t}$, $O = \set{a}$, $P = \set{p(t), q(t)}$, and $V = \set{X : t}$.
Suppose our interpretation is $(\phi, I, R, C)$ where $I = \set{p(a)}$, $R = \set{p(X) \fork q(X)}$, and $C = \set{\forall X : t, p(X) \oplus q(X)}$.
Here, $\tau(\theta) = (\set{p(a)}, \set{q(a)}, \set{q(a)}, \set{q(a)}, ...)$. Note that $q$ is a new predicate that does not appear in the sensory input; $q$ is a ``peer'' of $p$ (in that they are connected by an xor constraint), but $q$ was never observed.
\end{example}

The next theorem is a form of ``completeness'', showing that every sensory sequence has some admissible interpretation that satisfies the unity conditions. 

\begin{theorem}
\label{completeness}
For every apperception task $(S, \phi, C)$ there exists some interpretation $\theta = (\phi', I, R, C')$ that makes sense of $S$, where $\phi'$ extends $\phi$ and $C' \supseteq C$.
\end{theorem}
\begin{proof}
First, we define $\phi'$ given $\phi = (T, O, P, V)$.
For each sensor $x_i$ that features in $S$, $i = 1 .. n$, and each state $S_j$ in $S$, $j = 1 .. m$, create a new unary predicate $p^i_j$.
The intention is that $p^i_j(X)$ is true if $X$ is the $i$'th object $x_i$ at the $j$'th time-step.
If $\kappa_O(x_i) = t$ then let $\kappa_P(p^i_j) = (t)$.
For each type $t \in T$, create a new variable $X_t$ where $\kappa_V(X_t) = t$.
Let $\phi'  = (T, O, P', V')$ where $P' = P \cup \set{p^i_j \mid i = 1 .. n, j = 1 .. m}$, and $V' = V \cup \set{X_t \mid  t \in T}$.

Second, we define $\theta = (\phi', I, R, C')$.
Let the initial conditions $I$ be:
\begin{eqnarray*}
\bigsetbegin{}p^i_1(x_i) \mid i = 1 .. n\bigsetend{}
\end{eqnarray*}
Let the rules $R$ contain the following causal rules for $i = 1 .. n$ and $j = 1 .. m-1$ (where $x_i$ is of type $t$):
\begin{eqnarray*}
p^i_j(X_t) \fork p^i_{j+1}(X_t)
\end{eqnarray*}
together with the following static rules for each unary atom $q(x_i) \in S_j$:
\begin{eqnarray*}
p^i_j(X_t) \rightarrow q(X_t)
\end{eqnarray*}
and the following static rules for each binary atom $r(x_i, x_k) \in S_j$ (where $x_i$ is of type $t$ and $x_k$ is of type $t'$):
\begin{eqnarray*}
p^i_j(X_t) \wedge p^k_j(Y_{t'}) \rightarrow r(X_t, Y_{t'})
\end{eqnarray*}
We augment $C$ to $C'$ by adding the following additional constraints.
Let $P_t$ be the unary predicates for all objects of type $t$:
\begin{eqnarray*}
P_t = \bigsetbegin{} p^i_j \mid \kappa_O(x_i) = t, j = 1 .. m \bigsetend{}
\end{eqnarray*}
Let $P_t = \{p'_1, ..., p'_k\}$.
Then for each type $t$ add a unary constraint:
\begin{eqnarray*}
\forall X_t : t, p'_1(X_t) \oplus ... \oplus p'_k(X_t)
\end{eqnarray*}
It is straightforward to check that $\theta$ as defined satisfies the constraint of conceptual unity, that the constraints $C'$ are satisfied by each state in $\tau(\theta)$, and that the sensory sequence is covered by $\tau(\theta)$. 
To satisfy spatial unity, add a new ``world'' object $w$ of a new type $t_w$ and for each type $t$ add a relation $\mathit{part}_t(t, t_w)$ and a constraint $\forall X : t, \exists ! Y : t_w, \mathit{part}_t(X, Y)$. For each object $x$ of type $t$, add an initial condition atom $\mathit{part}_t(x, w)$ to $I$.
Thus, all the conditions of unity are satisfied, and $\theta$ is a unified interpretation of $S$. \qedhere
\end{proof}

\begin{example}
\label{example:completeness}
Consider the following apperception problem $(S, \phi, C)$.
Suppose there is one sensor $a$ with values $\mathit{on}$ and $\mathit{off}$.
Suppose the sensory sequence is $S_{1:7}$ where:
\begin{eqnarray*}
\begin{tabular}{llll}
$S_1 = \set{ \mathit{on}(a)}$ &
$S_2 = \set{ \mathit{off}(a)}$ &
$S_3 = \set{ \mathit{on}(a)}$ &
$S_4 = \set{ \mathit{off}(a)}$ \\
$S_5 = \set{ \mathit{on}(a)}$ &
$S_6 = \set{ \mathit{off}(a)}$ &
$S_7 = \set{ \mathit{on}(a)}$ &
\end{tabular}
\end{eqnarray*}
Let $\phi = (T, O, P, V)$ where $T= \set{t}$, $O = \set{a : t}$, $P = \set{\mathit{on}(t), \mathit{off}(t)}$, and $V = \set{}$.
Clearly, $\phi$ is suitable for $S$.
The constraints $C$ are just $\set{ \forall X : t, \mathit{on}(X) \oplus \mathit{off}(X) }$.

Applying Theorem \ref{completeness}, we generate 7 unary predicates $p_1, ..., p_7$.
The type signature $\phi'$ for this interpretation is $(T', O', P', V')$ where
$T' = \set{t, t_w}$, 
$O' = \set{a, w}$, 
$P' = P \cup \set{p_1(t), p_2(t), ..., p_7(t), \mathit{part}(t, t_w)}$, and
$V' = \set{X : t, Y : t_w} $.
Our interpretation is $(\phi', I, R, C')$ where:
\begin{eqnarray*}
\begin{tabular}{lll}
$I = \bigsetbegin{}
p_1(a) \\
\mathit{part}(a, w)
\bigsetend{}$ &
$R = \bigsetbegin{}
p_1(X) \fork p_2(X) \\
p_2(X) \fork p_3(X) \\
... \\
p_6(Y) \fork p_7(Y) \\
p_1(X) \rightarrow \mathit{on}(X) \\
p_2(X) \rightarrow \mathit{off}(X) \\
p_3(X) \rightarrow \mathit{on}(X) \\
p_4(X) \rightarrow \mathit{off}(X) \\
p_5(X) \rightarrow \mathit{on}(X) \\
p_6(X) \rightarrow \mathit{off}(X) \\
p_7(X) \rightarrow \mathit{on}(X)
\bigsetend{}$ &
$C' = \bigsetbegin{}
\forall X : t, \; \mathit{on}(X) \oplus \mathit{off}(X) \\
\forall X : t, \; p_1(X) \oplus p_2(X) \oplus ... \oplus p_7(X) \\
\forall X : t, \; \exists ! Y : t_w \; \mathit{part}(X, Y)
\bigsetend{}$
\end{tabular}
\end{eqnarray*}
\end{example}

\noindent
Note that the interpretation provided by Theorem \ref{completeness} is degenerate and unilluminating: it treats each object entirely separately (failing to capture any regularities between objects' behaviour) and treats every time-step entirely separately (failing to capture any laws that hold over multiple time-steps). 
This unilluminating interpretation provides an \emph{upper bound} on the complexity of the theory needed to make sense of the sensory sequence. 

\section{Computer implementation}
\label{sec:system}

The \sys{} is our system for solving apperception tasks.\footnote{The source code is available at https://github.com/RichardEvans/apperception.}
Given as input an apperception task $(S, \phi, C)$, the engine searches for a type signature $\phi'$ and a theory $\theta = (\phi', I, R, C')$ where $\phi'$ extends $\phi$, $C' \supseteq C$ and $\theta$ makes sense of $S$.
In this section, we describe how it is implemented.

\defbegin{}
A \define{template} is a structure for circumscribing a large but finite set of theories. It is a type signature together with constants that bound the complexity of the rules in the theory.
Formally, a template $\chi$ is a tuple $(\phi, N_{\rightarrow}, N_{\fork}, N_B)$ where $\phi$ is a type signature, $N_{\rightarrow}$ is the max number of static rules allowed in $R$, $N_{\fork}$ is the max number of causal rules allowed in $R$,
and $N_B$ is the max number of atoms allowed in the body of a rule in $R$.
\label{def:template}
\defend{}

Each template $\chi$ specifies a large (but finite) set of theories that conform to $\chi$.
Let $\Theta_{\chi, C} \subset \Theta$ be the subset of theories $(\phi, I, R, C')$ in $\Theta$ that conform to $\chi$ and where $C' \supseteq C$.

\begin{algorithm}
\DontPrintSemicolon
\SetKwInOut{Input}{input}
\SetKwInOut{Output}{output}
\Input{$(S, \phi, C)$, an apperception task}
\Output{$\theta^*$, a unified interpretation of $S$}
\BlankLine
$(s^*, \theta^*) \leftarrow (\mathit{max(float)}, nil)$ \;
\BlankLine
\ForEach{template $\chi$ extending $\phi$ of increasing complexity } {
$\theta \leftarrow \operatorname*{argmin}_\theta  \set{ \mathit{cost}(\theta) \mid \theta \in \Theta_{\chi, C}, S \sqsubseteq \tau(\theta), \mathit{unity}(\theta)}$ \; 
\If{$\theta \neq \mathit{nil}$}{
$s \leftarrow \mathit{cost}(\theta)$ \;
\If{$s < s^*$}{
$(s^*, \theta^*) \leftarrow (s, \theta)$ \;
}
}
\If{exceeded processing time}{
\KwRet{$\theta^*$} \;
}
}
\BlankLine
\BlankLine
\caption{The \sys{} algorithm in outline}
\label{algo}
\end{algorithm}

Our method, presented in Algorithm \ref{algo}, is an anytime algorithm that enumerates templates of increasing complexity. 
For each template $\chi$, it finds the $\theta \in \Theta_{\chi, C}$ with lowest cost (see Definition \ref{def:theory-cost}) that satisfies the conditions of unity.
If it finds such a $\theta$, it stores it.
When it has run out of processing time, it returns the lowest cost $\theta$ it has found from all the templates it has considered.

Note that the relationship between the complexity of a template and the cost of a theory satisfying the template is not always simple. Sometimes a theory of lower cost may be found from a template of higher complexity. This is why we cannot terminate as soon as we have found the first theory $\theta$. We must keep going, in case we later find a lower cost theory from a more complex template.


The two non-trivial parts of this algorithm are the way we enumerate templates, and the way we find the lowest-cost theory $\theta$ for a given template $\chi$.
We consider each in turn.

\subsection{Iterating through templates}
\label{sec:iterating}

We need to enumerate templates in such a way that every template is (eventually) visited by the enumeration.
Since the objects, predicates, and variables are \emph{typed} (see Definition \ref{def:type-signature}), the acceptable ranges of $O$, $P$, and $V$ depend on $T$.
Because of this, our enumeration procedure is two-tiered:
first, enumerate sets $T$ of types;
second, given a particular $T$, enumerate $(O, P, V, N_\rightarrow, N_{\fork{}}, N_B)$ tuples for that particular $T$.
We cannot, of course, enumerate \emph{all} $(O, P, V, N_\rightarrow, N_{\fork{}}, N_B)$ tuples because there are infinitely many.
Instead, we specify a constant bound ($n$) on the number of tuples, and gradually increase that bound:

\begin{algorithm}[H]
\DontPrintSemicolon
\SetKwInOut{Input}{input}
\SetKwInOut{Output}{output}
\BlankLine
\ForEach{$(T, n)$} {
emit $n$ tuples of the form $(O, P, V, N_\rightarrow, N_{\fork{}}, N_B)$ \;
}
\label{algo:enumerate}
\end{algorithm}
\noindent
In order to enumerate $(T, n)$ pairs, we use a standard diagonalization procedure. See Table \ref{table:diagonalization}.

\begin{table}
\centering
\begin{tabular}{|l|lllll}
\hline
& 100 & 200 & 300 & 400 & ...\\
\hline
1 & 1 & 2 & 4 & 7 & ... \\
2 & 3 & 5 & 8 & ... & \\
3 & 6 & 9 & ... & & \\
4 & 10 & ... & & & \\
... & ... & & & & \\
\end{tabular}
\caption[Enumerating $(t, n)$ pairs]{Enumerating $(t, n)$ pairs. Row $t$ means that there are $t$ types in $T$, while column $n$ means there are $n$ tuples of the form $(O, P, V, N_\rightarrow, N_{\fork{}}, N_B)$ to enumerate. We increment $n$ by 100. The entries in the table represent the order in which the $(t, n)$ pairs are visited.}
\label{table:diagonalization}
\end{table}

Once we have a $(T, n)$ pair, we need to emit $n$ $(O, P, V, N_\rightarrow, N_{\fork{}}, N_B)$ tuples using the types in $T$.
One way of enumerating $k$-tuples, where $k > 2$, is to use the diagonalization technique recursively: first enumerate pairs, then apply the diagonalization technique to enumerate pairs consisting of individual elements paired with pairs, and so on.
But this recursive application will result in heavy biases towards certain $k$-tuples.
Instead, we use the Haskell function \verb|Universe.Helpers.choices| to enumerate $n$-tuples while minimizing bias.
The \verb|choices :: [[a]] -> [[a]]| function takes a finite number $n$ of (possibly infinite) lists, and produces a (possibly infinite) list of  $n$-tuples, generating a $n$-way Cartesian product that is guaranteed to eventually produce every such $n$-tuple.

We use \verb|choices| to generate 6-tuples $(O, P, V, N_\rightarrow, N_{\fork{}}, N_B)$ tuples by creating six infinite streams: 
(i) $\mathcal{S}_O$: an infinite list of finite lists of typed objects, 
(ii) $\mathcal{S}_P$: an infinite list of finite lists of typed predicates,
(iii) $\mathcal{S}_V$: an infinite list of finite lists of typed variables,
(iv) $\mathcal{S}_{\rightarrow} = \{0, 1, ...\}$: the number of static rules,
(v) $\mathcal{S}_{\fork} = \{0, 1, ...\}$: the number of causal rules, and
(vi) $\mathcal{S}_B = \{0, 1, ...\}$: the max number of body atoms.
Now when we pass this list of streams to the \verb|choices| function, it produces an enumeration of the 6-way Cartesian product $\mathcal{S}_O \times \mathcal{S}_P \times \mathcal{S}_V \times \mathcal{S}_{\rightarrow} \times \mathcal{S}_{\fork} \times \mathcal{S}_B$.

\begin{example}
\label{example:enumeration}
Recall the apperception problem from Example \ref{ex:sensory-sequence}.
There are two sensors $a$ and $b$, and each sensor can be $\mathit{on}$ or $\mathit{off}$.
\begin{eqnarray*}
\begin{tabular}{lllll}
$S_1 = \set{}$  & $S_2 = \set{ \mathit{off}(a), \mathit{on}(b)}$ & $S_3 = \set{ \mathit{on}(a),  \mathit{off}(b)}$ & 
$S_4 = \set{ \mathit{on}(a), \mathit{on}(b)}$ & $S_5 = \set{ \mathit{on}(b)}$ \\ 
$S_6 = \set{ \mathit{on}(a),  \mathit{off}(b)}$ & $S_7 = \set{ \mathit{on}(a), \mathit{on}(b)}$ & $S_8 = \set{ \mathit{off}(a), \mathit{on}(b)}$ & $S_9 = \set{ \mathit{on}(a)}$ & $S_{10} = \set{ }$
\end{tabular}
\end{eqnarray*}
We shall start with an initial template $\chi_0 = (\phi = (T, O, P, V), N_{\rightarrow}, N_{\fork}, N_B)$, where $T = \set{\mathit{sensor}, \mathit{grid}}$, 
$O = \set{a{:}\mathit{sensor}, b{:}\mathit{sensor}, g{:}\mathit{grid}}$, 
$P = \set{ \mathit{on}(\mathit{sensor}), \mathit{off}(\mathit{sensor}), \mathit{part}(\mathit{sensor}, \mathit{grid})}$, 
$V = \set{X{:}\mathit{sensor}, Y{:}\mathit{sensor}}$, 
$N_{\rightarrow} = 1$,
$N_{\fork} = 3$, and
$N_B = 2$.
We use the template enumeration procedure described above to generate increasingly complex templates $\chi_1, \chi_2, ...$, using $\chi_0$ as a base.
This produces the following augmented templates : 
\begin{eqnarray*}
\begin{tabular}{lll}
$\Delta \chi_1 = (\emptyset, \emptyset, \emptyset, \{V_1{:}\mathit{sensor}\}, 0, 0, 0)$ &
$\Delta \chi_2 = (\emptyset, \emptyset, \emptyset, \emptyset, 0, 0, 1)$ &
$\Delta \chi_3 = (\emptyset, \emptyset, \{p_1(\mathit{sensor})\}, \emptyset, 0, 0, 0)$ \\
$\Delta \chi_4 = (\emptyset, \emptyset, \emptyset, \{V_1{:}\mathit{sensor}\}, 0, 0, 1)$ &
$\Delta \chi_5 = (\emptyset, \emptyset, \emptyset, \emptyset, 0, 0, 2)$ &
$\Delta \chi_6 = (\emptyset, \{o_1{:}\mathit{sensor}\}, \emptyset, \emptyset, 0, 0, 0)$ \\
$\Delta \chi_7 = (\emptyset, \emptyset, \{p_1(\mathit{sensor})\}, \emptyset, 0, 0, 1)$ &
$\Delta \chi_8 = (\emptyset, \emptyset, \emptyset, \{V_1{:}\mathit{sensor}\}, 0, 0, 2)$ &
$\Delta \chi_9 = (\emptyset, \emptyset, \emptyset, \emptyset, 0, 0, 2)$ \\
$\Delta \chi_{10} = (\emptyset, \emptyset, \{p_1(\mathit{sensor})\},  \{V_1{:}\mathit{sensor}\}, 0, 0, 0)$ & $\dotso$ & $\dotso$
\end{tabular}
\end{eqnarray*}
In the list above, we display the change from the base template $\chi_0$, so $\Delta \chi_i$ means the changes in template $\chi_i$ from the base template $\chi_0$.
Each template $\chi = (\phi=(T, O, P, V), N_{\rightarrow}, N_{\fork}, N_B)$ is flattened as a 7-tuple $(T, O, P, V,  N_{\rightarrow}, N_{\fork}, N_B)$.

Many of these templates do not have the expressive resources to find a unified interpretation.
But some do.
The first solution the \sys{} finds has the following type signature (the new elements are in bold):
\begin{eqnarray*}
\begin{tabular}{llll}
$T = \left\{ \begin{array}{l}
\mathit{grid}\\
\mathit{sensor}\\
\end{array}\right\}$ &
$O = \left\{ \begin{array}{l}
a: \mathit{sensor}\\
b: \mathit{sensor}\\
g: \mathit{grid}\\
\end{array}\right\}$ &
$P = \left\{ \begin{array}{l}
\mathbf{p_1(\mathit{sensor})}\\
\mathbf{p_2(\mathit{sensor})}\\
\mathit{off}(\mathit{sensor})\\
\mathit{on}(\mathit{sensor})\\
\mathit{part}(\mathit{sensor}, \mathit{grid})\\
\end{array}\right\}$ & 
$V = \left\{ \begin{array}{l}
S: \mathit{sensor}\\
S2: \mathit{sensor}\\
\end{array}\right\}$
\end{tabular}
\end{eqnarray*}
together with the following theory $\theta = (\phi, I, R, C)$, where:
\begin{eqnarray*}
\begin{tabular}{llll}
$I = \bigsetbegin{}
p_1(a)\\
p_2(b) \\
\mathit{on}(a) \\
\mathit{part}(a,g) \\
\mathit{part}(b,g)
\bigsetend{}$ & 
$R =  \bigsetbegin{}
p_2(\mathit{S}) \rightarrow \mathit{on}(\mathit{S})\\
p_2(\mathit{S}) \fork p_1(\mathit{S})\\
p_1(\mathit{S}) \wedge \mathit{on}(\mathit{S}) \fork \mathit{off}(\mathit{S})\\
\mathit{off}(\mathit{S}) \wedge p_1(\mathit{S}) \fork p_2(\mathit{S})\\
\bigsetend{}$ & 
$C = \bigsetbegin{}
\forall X : sensor,\;p_1(X) \oplus p_2(X)\\
\forall X : sensor, \; \exists ! Y : grid, \; part(X, Y) \\
\bigsetend{}$
\end{tabular}
\end{eqnarray*}
This solution uses the invented predicates $p_1$ and $p_2$ to represent two states of a state-machine. This is recognisable as a compressed version of Example \ref{example:eca1} above.

Later, the \sys{} finds another solution using the type signature $\phi = (T, O, P, V)$ (again, the augmented parts of the type signature are in bold):
\begin{eqnarray*}
\begin{tabular}{llll}
$T = \left\{ \begin{array}{l}
\mathit{grid}\\
\mathit{sensor}\\
\end{array}\right\}$ &
$O = \left\{ \begin{array}{l}
a: \mathit{sensor}\\
b: \mathit{sensor}\\
g: \mathit{grid}\\
\mathbf{o_1}: \mathit{sensor}\\
\end{array}\right\}$ &
$P = \left\{ \begin{array}{l}
\mathbf{r_1}(\mathit{sensor}, \mathit{sensor})\\
\mathit{off}(\mathit{sensor})\\
\mathit{on}(\mathit{sensor})\\
\mathit{part}(\mathit{sensor}, \mathit{grid})\\
\end{array}\right\}$ &
$V = \left\{ \begin{array}{l}
S: \mathit{sensor}\\
S2: \mathit{sensor}\\
\end{array}\right\}$
\end{tabular}
\end{eqnarray*}
together with the following theory $\theta = (\phi, I, R, C)$, where:
\begin{eqnarray*}
\begin{tabular}{llll}
$I = \left\{ \begin{array}{l}
\begin{tabular}{ll}
$\mathit{off}(o_1)$ &
$\mathit{on}(a)$ \\
$r_1(a,o_1)$ &
$r_1(b,a)$ \\
$r_1(o_1,b)$ &
$\mathit{part}(a,\mathit{grid})$ \\
$\mathit{part}(b,\mathit{grid})$ &
$\mathit{part}(o_1,\mathit{grid})$\\
\end{tabular}
\end{array}\right\}$ &
$R =  \left\{ \begin{array}{l}
\mathit{off}(\mathit{S}) \wedge r_1(\mathit{S},\mathit{S}2) \rightarrow \mathit{on}(\mathit{S}2)\\
\mathit{off}(\mathit{S}2) \wedge r_1(\mathit{S},\mathit{S}2) \wedge \mathit{on}(\mathit{S}) \fork \mathit{off}(\mathit{S})\\
\end{array}\right\}$ &
$C = \left\{ \begin{array}{l}
\forall X{:}\mathit{sensor}, \; \exists ! Y{:}\mathit{grid}, \; \mathit{part}(X, Y) \\
\forall X{:}\mathit{sensor}, \; \exists ! Y{:}\mathit{sensor}, \; r_1(X, Y) \\
\end{array}\right\}$
\end{tabular}
\end{eqnarray*} 
Here, it has constructed an invented object $o_1{:}\mathit{sensor}$ and posited a one-dimensional spatial relationship $r_1$ between the three sensors. 
This solution is recognisable as a variant of Example \ref{example:eca2} above.

\end{example}

\subsection{Finding the best theory from a template}
\label{sec:searching}

The most complex part of Algorithm \ref{algo} is:
\begin{eqnarray*}
\theta \leftarrow \operatorname*{argmin}_\theta  \set{ \mathit{cost}(\theta) \mid \theta \in \Theta_{\chi, C}, S \sqsubseteq \tau(\theta), \mathit{unity}(\theta)}
\end{eqnarray*}
Here, we search for a theory $\theta$ with the lowest cost (see Definition \ref{def:theory-cost}) such that $\theta$ conforms to the template $\chi$ and includes the constraints in $C$, such that $\tau(\theta)$ covers $S$, and $\theta$ satisfies the conditions of unity.
In this sub-section, we explain in outline how this works.

\begin{figure}
\centering
\begin{minipage}{18 cm}
  \subfloat[Deduction\label{fig:inference-method-1}]{%
    \includegraphics[width=0.2\textwidth]{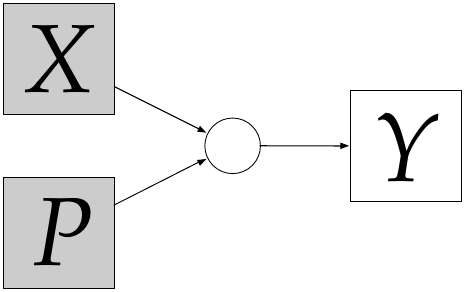}
  }
  \qquad  
  \subfloat[Abduction\label{fig:inference-method-2}]{%
    \includegraphics[width=0.2\textwidth]{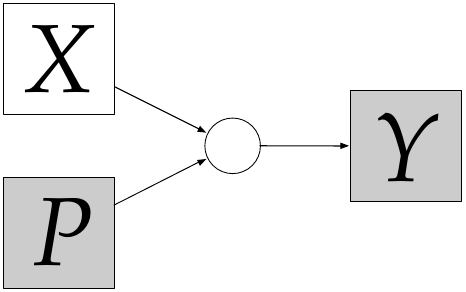}
  }  
  \qquad  
  \subfloat[Induction\label{fig:inference-method-3}]{%
    \includegraphics[width=0.2\textwidth]{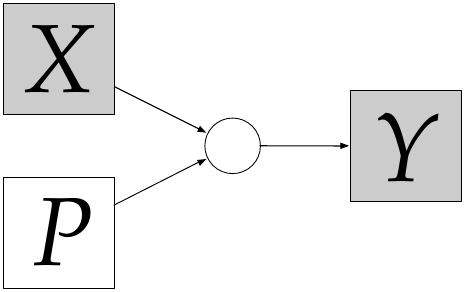}
  }  
 \qquad
  \subfloat[Abduction and induction\label{fig:inference-method-4}]{%
    \includegraphics[width=0.2\textwidth]{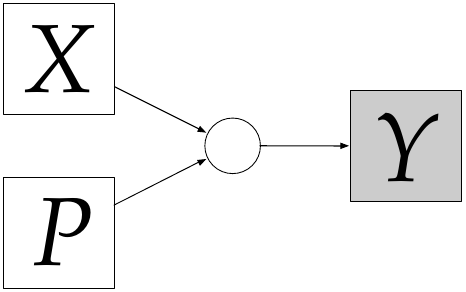}
  }
  
\caption[The varieties of inference]{The varieties of inference. Here, shaded elements are given, and unshaded elements must be generated. $X$ and $Y$ are sets of facts, while $P$ is a set of rules for transforming $X$ into $Y$.} 
\label{fig:inference-methods}
\end{minipage}
\end{figure}

Our approach combines abduction and induction to generate a unified interpretation $\theta$.\footnote{For related work that also uses abduction, see \cite{ray2009nonmonotonic,henson2012semantic,teijeiro2018adoption}.}
See Figure \ref{fig:inference-methods}.
Here, $X \subseteq \mathcal{G}$ is a set of facts (ground atoms), $P : \mathcal{G} \rightarrow \mathcal{G}$ is a procedure for generating the consequences of a set of facts, and $Y \subseteq \mathcal{G}$ is the result of applying $P$ to $X$.
If $X$ and $P$ are given, and we wish to generate $Y$, then we are performing \emph{deduction}.
If $P$ and $Y$ are given, and we wish to generate $X$, then we are performing \emph{abduction}.
If $X$ and $Y$ are given, and we wish to generate $P$, then we are performing \emph{induction}.
Finally, if only $Y$ is given, and we wish to generate both $X$ and $P$, then we are jointly performing abduction and induction.
This is what the \sys{} does.\footnote{At a high level, our system is similar to XHAIL \cite{ray2009nonmonotonic}. But there are a number of differences.
First, our program $P$ contains causal rules and constraints as well as standard Horn clauses. 
Second, our conclusion $Y$ is an infinite sequence $(S_1, S_2, ...)$ of sets, rather than a single set.
Third, we add additional filters on acceptable theories in the form of the unity conditions (see Definition \ref{def:unity}).}

Our method is described in Algorithm \ref{algo2}. 
In order to jointly abduce a set $I$ (of initial conditions) and induce sets $R$ and $C$ (of rules and constraints), 
we implement a \logic{} interpreter in ASP. See Section \ref{sec:computational-details} for the details.
This interpreter takes a set $I$ of atoms (represented as a set of ground ASP terms) and sets $R$ and $C$ of rules and constraints (represented again as a set of ground ASP terms), and computes the trace of the theory $\tau(\theta) = (S_1, S_2, ...)$ up to a finite time limit.

Concretely, we implement the interpreter as an ASP program $\pi_\tau$ that computes $\tau(\theta)$ for theory $\theta$.
We implement the conditions of unity as ASP constraints in a program $\pi_u$. 
We implement the cost minimization as an ASP program $\pi_m$ that counts the number of atoms in each rule plus the number of initialisation atoms in $I$, and uses an ASP weak constraint \cite{calimeri2012asp} to minimize this total. 
Then we generate ASP programs representing the sequence $S$, the initial conditions, the rules and constraints. We combine the ASP programs together and ask the ASP solver (\verb|clingo| \cite{gebser2014clingo}) to find a lowest cost solution.
(There may be multiple solutions that have equally lowest cost; the ASP solver chooses one of the optimal answer sets). 
We extract a readable interpretation $\theta$ from the ground atoms of the answer set.
In Section \ref{sec:computational-details}, we explain how Algorithm \ref{algo2} is implemented in ASP.
In Section \ref{sec:complexity}, we evaluate the computational complexity.
In Section \ref{sec:optimization}, we describe the various optimisations used to prune the search. 
In Section \ref{sec:ilasp}, we compare with ILASP, a state-of-the-art ILP system.

\begin{algorithm}
\DontPrintSemicolon
\SetKwInOut{Input}{input}
\SetKwInOut{Output}{output}
\Input{$S$, a sensory sequence}
\Input{$\chi = (\phi, N_{\rightarrow}, N_{\fork}, N_B)$, a template}
\Input{$C$, a set of constraints on the predicates of the sensory sequence}
\Output{$\theta$, the simplest unified interpretation of $S$ that conforms to $\chi$}
\BlankLine
$\pi_S \leftarrow \mathsf{gen\_input} (S)$ \;
$\pi_I \leftarrow \mathsf{gen\_inits}(\phi)$\;
$\pi_R \leftarrow \mathsf{gen\_rules}(\phi, N_{\rightarrow}, N_{\fork}, N_B)$ \;
$\pi_C \leftarrow \mathsf{gen\_constraints}(\phi, C)$ \;
$\Pi \leftarrow \pi_\tau \cup \pi_u \cup \pi_m \cup \pi_S \cup \pi_I \cup \pi_R \cup \pi_C$  \;
$A \leftarrow \mathsf{clingo}(\Pi)$ \;
\If{$\mathsf{satisfiable}(A)$}{
$\theta \leftarrow \mathsf{extract}(A)$ \;
\KwRet{$\theta$} \;
}
\KwRet{nil} \;
\BlankLine
\caption[Finding the lowest cost $\theta$ for sequence $S$ and template $\chi$]{Finding the lowest cost $\theta$ for sequence $S$ and template $\chi$. 
Here, $\pi_\tau$ computes the trace, $\pi_u$ checks that the unity conditions are satisfied, and $\pi_m$ minimizes the cost of $\theta$.
}
\label{algo2}
\end{algorithm}

\subsection{The ASP encoding}
\label{sec:computational-details}

Our \logic{} interpreter is written in ASP.
All elements of \logic{}, including variables, are represented by ASP constants.
A variable $X$ is represented by a constant \verb|var_x|, and a predicate $p$ is represented by a constant \verb|c_p|.
Elements of the target language are \emph{reified} in ASP, so an unground atom $p(X)$ of \logic{} is represented by a term \verb|s(c_p, var_x)|, and 
a rule is represented by a set of ground atoms for the body, and a single ground atom for the head.
For example, the static rule $p(X) \wedge q(X, Y) \rightarrow r(Y)$ is represented as:
\begin{Verbatim}[baselinestretch=0.9]
rule_body(r1, s(c_p, var_x)). 
rule_body(r1, s2(c_q, var_x, var_y)). 
rule_head_static(r1, s(c_r, var_y)). 
\end{Verbatim}
The causal rule $\mathit{on}(X) \fork \mathit{off}(X)$ is represented as:
\begin{Verbatim}[baselinestretch=0.9]
rule_body(r2, s(c_on, var_x)). 
rule_head_causes(r2, s(c_off, var_x)). 
\end{Verbatim}
Given a type signature $\phi$, we construct ASP terms that represent every well-typed unground atom in $U_\phi$, and wrap these terms in the \verb|is_var_atom| predicate. 
For example, to represent that $r(C, C_2)$ is a well-typed unground atom, we write \verb|is_var_atom(atom(c_r, var_c, var_c2))|.
Similarly, we construct ASP terms that represent every well-typed ground atom in $G_\phi$ using the \verb|is_ground_atom| predicate.

%

The initial conditions $I \subseteq G_\phi$ are represented by the \verb|init| predicate. 
For example, to represent that $I = \{p(a), r(b, c)\}$, we write:
\begin{Verbatim}[baselinestretch=0.9]
init(s(c_p, c_a)).
init(s2(c_r, c_b, c_c)).
\end{Verbatim}

Our meta-interpreter $\pi_\tau$ implements $\tau : \Theta \rightarrow (2^\mathcal{G})^*$ from Definition \ref{def:trace}. We use \verb|holds(a, t)| to represent that atom \verb|a| is true at time \verb|t| (i.e. $a \in S_t$ where $\tau(\theta) = (S_1, S_2, ..., S_t, ...)$).
\begin{Verbatim}[baselinestretch=0.9]
% initial conditions
holds(A, T) :-
    init(A),
    init_time(T).

% update for static rules
holds(GA, T) :-
    rule_head_static(R, VA),
    eval_body(R, Subs, T),
    ground_atom(VA, GA, Subs).

% update for causal rules
holds(GC, T+1) :-     
    rule_head_causes(R, VC),
    eval_body(R, Subs, T),
    ground_atom(VC, GC, Subs),
    is_time(T+1).

% frame axiom
holds(S, T+1) :-
    holds(S, T),
    is_time(T+1),
    not -holds(S, T+1).
\end{Verbatim}    
Since $\tau(\theta)$ is an infinite sequence $(A_1, A_2, ...)$, we cannot compute the whole of it.
Instead, we only compute the sequence up to the max time of the original sensory sequence $S$.

Note that the frame axiom uses negation as failure and strong negation \cite{baral1994logic} to check that some other incompatible atom has not already been added. 
Thus, the frame axiom is not restrictive and can be overridden as needed by the causal rules, to handle predicates that are not inertial.

The conditions of unity described in Section \ref{sec:unity-conditions} are represented directly as ASP constraints in $\pi_u$.
For example, spatial unity is encoded as:
\begin{Verbatim}[baselinestretch=0.9]
:-  spatial_unity_counterexample(X, Y, T).
\end{Verbatim}
Here, there is a counterexample to spatial unity at time \verb|T| if objects \verb|X| and \verb|Y| are not \verb|related|:
\begin{Verbatim}[baselinestretch=0.9]
spatial_unity_counterexample(X, Y, T) :-
    is_object(X),
    is_object(Y),
    is_time(T),
    not related(X, Y, T).
\end{Verbatim}
Here, \verb|related| is the reflexive symmetric transitive closure of the relation holding between \verb|X| and \verb|Y| if there is some relation \verb|R| connecting them.
Note that \verb|related| can quantify over relations because the \logic{} atoms and predicates have been reified into terms:
\begin{Verbatim}[baselinestretch=0.9]
related(X, Y, T) :-
    holds(s2(R, X, Y), T).

related(X, X, T) :-
    is_object(X),
    is_time(T).
    
related(X, Y, T) :- related(Y, X, T).

related(X, Y, T) :-
    related(X, Z, T),
    related(Z, Y, T).
\end{Verbatim}

When constructing a theory $\theta = (\phi, I, R, C)$, the solver needs to choose which ground atoms to use as initial conditions in $I$, which static and causal rules to include in $R$, and which xor or uniqueness conditions to use as conditions in $C$.

To allow the solver to choose what to include in $I$, we add the ASP choice rule to $\pi_I$:
\begin{Verbatim}[baselinestretch=0.9]
{ init(A) } :- is_ground_atom(A).
\end{Verbatim}

To allow the solver to choose which rules to include in $R$, we add the following clauses to $\pi_R$:
\begin{Verbatim}[baselinestretch=0.9]
0 { rule_body(R, VA) : is_var_atom(VA) } k_max_body :- is_rule(R).
1 { rule_head_static(R, VA) : is_var_atom(VA) } 1 :- is_static_rule(R).
1 { rule_head_causes(R, VA) : is_var_atom(VA) } 1 :- is_causes_rule(R).
\end{Verbatim}
Here, \verb|k_max_body| is the $N_B$ parameter of the template that specifies the max number of body atoms in any rule.
The number of rules satisfying \verb|is_static_rule| and \verb|is_causes_rule| is determined by the parameters $N_{\fork}$ and $N_\rightarrow$ in the template (see Definition \ref{def:template}). 

The ASP program $\pi_m$ minimizes the cost of the theory $\theta$ (see Definition \ref{def:theory-cost}) by using weak constraints \cite{calimeri2012asp}:
\begin{eqnarray*}
& & \asp{:\sim  rule\_body(R, A). \; [1 @ 1, R, A]} \\
& & \asp{:\sim  rule\_head\_static(R, A). \; [1 @ 1, R, A]} \\
& & \asp{:\sim  rule\_head\_causes(R, A). \; [1 @ 1, R, A]} \\
& & \asp{:\sim  init(A). \; [1 @ 1, A]} 
\end{eqnarray*}

\subsection{Complexity}
\label{sec:complexity}

This section describes the complexity of Algorithm \ref{algo2}.

We assume basic concepts and standard terminology from complexity theory.
Let $P$ be the class of problems that can be solved in polynomial time by a deterministic Turing machine, 
$NP$ be the class of problems solved in polynomial time by a non-deterministic Turing machine, and $EXPTIME$ be the class of problems solved in time $2^{n^d}$ by a deterministic Turing machine.
Let $\Sigma_{i+1}^P = NP^{\Sigma_i^P}$ be the class of problems that can be solved in polynomial time by a non-deterministic Turing machine with a $\Sigma_i^P$ oracle.
If $\Pi$ is a Datalog program, and $A$ and $B$ are sets of ground atoms, then:
\begin{itemize}
\item
the \define{data complexity} is the complexity of testing whether $\Pi \cup A \models B$, as a function of $A$ and $B$, when $\Pi$ is fixed
\item
the \define{program complexity} (also known as ``expression complexity'') is the complexity of testing whether $\Pi \cup A \models B$, as a function of $\Pi$ and $B$, when $A$ is fixed
\end{itemize}
Datalog has polynomial time data complexity but exponential time program complexity: deciding whether a ground atom is in the least Herbrand model of a Datalog program is EXPTIME-complete. The reason for this complexity is because the number of ground instances of a clause is an exponential function of the number of variables in the clause.
Finding a solution to an ASP program is in NP \cite{ben1994propositional,dantsin2001complexity}, while finding an optimal solution to an ASP program with weak constraints is in $\Sigma^P_2$ \cite{brewka2003answer,gebser2011complex}.

Since deciding whether a non-disjunctive ASP program has a solution is in NP \cite{ben1994propositional,dantsin2001complexity}, our ASP encoding of Algorithm \ref{algo2} shows that finding a unified interpretation $\theta$ for a sequence given a template is in NP.
Since verifying whether a solution to an ASP program with preferences is indeed optimal is in $\Sigma^P_2$ \cite{brewka2003answer,gebser2011complex}, our ASP encoding shows that finding the lowest cost $\theta$ is in $\Sigma^P_2$.

However, the standard complexity results assume the ASP program has \emph{already been grounded into a set of propositional clauses}.
To really understand the space and time complexity of Algorithm \ref{algo2}, we need to examine how the set of ground atoms in the ASP encoding grows as a function of the parameters in the template $\chi = (\phi=(T, O, P, V), N_{\rightarrow}, N_{\fork}, N_B)$.

Observe that, since we restrict ourselves to unary and binary predicates, the number of ground and unground atoms is a small polynomial function of the type signature parameters\footnote{The actual numbers will be less than these bounds because type-checking rules out certain combinations.}:
\begin{eqnarray*}
|G_\phi| & \leq & |P| \cdot |O|^2 \\
|U_\phi| & \leq & |P| \cdot |V|^2
\end{eqnarray*}
But note that the number of substitutions $\Sigma_\phi$ that is compatible with the signature $\phi$ is an exponential function of the number of variables $V$:
\begin{eqnarray*}
|\Sigma_\phi| & \leq & |O|^{|V|}
\end{eqnarray*}

Table \ref{table:num-ground-clauses} shows the number of ground clauses for the three most expensive clauses.
The total number of ground atoms for the three most expensive clauses is approximately $5 \cdot |\Sigma_\phi| \cdot (N_{\rightarrow} + N_{\fork}) \cdot |U_\phi| \cdot t$.

\begin{table}
\centering
\begin{tabular}{|l|l|}
\hline
{\bf Predicate} & {\bf \# ground clauses} \\
\hline
\verb|holds| & $|\Sigma_\phi| \cdot |U_\phi| \cdot (N_{\rightarrow} + N_{\fork}) \cdot t + |G_\phi| \cdot (t+1)$ \\
\hline
\verb|eval_atom(VA, Subs, T)| & $|\Sigma_\phi| \cdot |U_\phi| \cdot t$ \\
\hline
\verb|eval_body(R, Subs, T)| & $|\Sigma_\phi| \cdot (N_{\rightarrow} + N_{\fork}) \cdot t$ \\
\hline
\end{tabular}
\caption{The number of ground clauses in the ASP encoding of Algorithm \ref{algo2}.}
\label{table:num-ground-clauses}
\end{table}

%

\begin{table}
\begin{center}
\begin{tabular}{|p{6.0cm}|r|r|}
\hline
{\bf Clause} & {\bf \# ground clauses} & {\bf \# atoms}\\
\hline
\begin{lstlisting} 
holds(GA, T) :-
    rule_head_static(R, VA),
    eval_body(R, Subs, T),
    ground_atom(VA, GA, Subs).\end{lstlisting}&
$|\Sigma_\phi| \cdot |U_\phi| \cdot N_{\rightarrow} \cdot t$
&
$4$ \\
\hline
\begin{lstlisting}
holds(GA, T+1) :-
    rule_head_causes(R, VA),
    eval_body(R, Subs, T),
    ground_atom(VA, GA, Subs).\end{lstlisting}&
$|\Sigma_\phi| \cdot |U_\phi| \cdot N_{\fork} \cdot t$
&
$4$ \\
\hline
\begin{lstlisting}
eval_body(R, Subs, T) :-
    is_rule(R),
    is_subs(Subs),
    is_time(T),
    eval_atom(V, Subs, T) : rule_body(R, V).\end{lstlisting}&
$|\Sigma_\phi| \cdot (N_{\rightarrow} + N_{\fork}) \cdot t$
&
$|U_\phi| + 4$ \\
\hline
\end{tabular} 
\caption[The number of ground clauses in the ASP encoding]{The number of ground clauses in the ASP encoding of Algorithm \ref{algo2}.}
\label{table:num-ground-clauses}
\end{center}
\end{table}

\subsection{Optimization}
\label{sec:optimization}

Because of the combinatorial complexity of the apperception task, 
we had to introduce a number of optimizations to get reasonable performance on even the simplest of domains.

\subsubsection{Reducing grounding with type checking}
\label{sec:type-checking}

We use the type signature $\phi$ to dramatically restrict the set of ground atoms ($G_\phi$), the unground atoms ($U_\phi$), the substitutions ($\Sigma_\phi$), and rules $R_\phi$. 
Type-checking has been shown to drastically reduce the search space in program synthesis tasks \cite{typing-morel}.

\subsubsection{Symmetry breaking}
\label{sec:symmetry-breaking}

We use symmetry breaking to remove candidates that are equivalent.
We remove programs that are equivalent up to a variable renaming by using a strict ordering on variables.
We also remove programs that are equivalent up to a reordering of the rules by using a strict ordering on unground atoms.
%

\subsubsection{Adding redundant constraints}
\label{sec:redundant-constraints}

ASP programs can be significantly optimized by adding redundant constraints (constraints that are provably entailed by the other clauses in the program) \cite{gebser2012answer}.
We speeded up solving time (by about 30\%) by adding the following redundant constraints:
\begin{Verbatim}[baselinestretch=0.9]
:-  init(A),
    init(B),
    incompossible(A, B).

:-  rule_body(R, A),
    rule_body(R, B),
    incompossible_var_atoms(A, B).    

:-  rule_body(R, A),
    rule_head_static(R, A).

:-  rule_body(R, A),
    rule_head_causes(R, A).

:-  rule_body(R, A),
    rule_head_static(R, B),
    incompossible_var_atoms(A, B).    

:-  rule_body(R, A),
    rule_head_causes(R, B),
    incompossible_var_atoms(A, B).    
\end{Verbatim}     

%

\section{Experiments}
\label{sec:experiments}

\subsection{Five experimental domains}
\label{sec:domains}

To evaluate the generality of our system, we tested it in a variety of domains: elementary (one-dimensional) cellular automata,
drum rhythms and nursery tunes, sequence induction tasks, multi-modal binding tasks, and occlusion tasks.
These particular domains were chosen because they represent a diverse range of tasks that are simple for humans but are hard for state-of-the-art machine learning systems.
The tasks were chosen to highlight the difference between mere perception (the classification tasks that machine learning systems already excel at) and apperception (assimilating information into a coherent integrated theory, something traditional machine learning systems are not designed to do).

\subsubsection{Results}
\label{sec:results-summary}

We implemented the \sys{} in Haskell and ASP.
We used \verb|clingo| \cite{gebser2014clingo} to solve the ASP programs generated by our system.
We ran all experiments with a time-limit of 4 hours on a standard Unix desktop.

Our experiments (on the prediction task) are summarised in Table \ref{table:experiments-overview}.
Note that our accuracy metric for a single task is rather exacting: the model is accurate (Boolean) on a task iff \emph{every} hidden sensor value is predicted correctly.\footnote{The reason for using this strict notion of accuracy is that, as the domains are deterministic and noise-free, there is a simplest possible theory that explains the sensory sequence. In such cases where there is a correct answer, we wanted to assess whether the system found that correct answer exactly -- not whether it was fortunate enough to come close while misunderstanding the underlying dynamics.} It does not score any points for predicting most of the hidden values correctly.
As can be seen from Table \ref{table:experiments-overview}, our system is able to achieve good accuracy across all five domains.

\begin{table}
\centering
\begin{tabular}{|l|p{1.2cm}|p{1.2cm}|p{1.7cm}|p{2.2cm}|p{1.5cm}|}
\hline
{\bf Domain} & {\bf Tasks (\#)} & {\bf Memory (megs)} & {\bf Input size (bits)} &  {\bf Held out size (bits)} & {\bf Accuracy (\%)} \\
\hline
ECA & 256 & 473.2 & 154.0 & 10.7 & 97.3\% \\
\hline
Rhythm \& music & 30 & 2172.5 & 214.4 & 15.3 & 73.3\% \\
\hline
\emph{Seek Whence} & 30 & 3767.7 & 28.4 & 2.5 & 76.7\% \\
\hline
Multi-modal binding & 20 & 1003.2 & 266.0 & 19.1 & 85.0\% \\
\hline
Occlusion & 20 & 604.3 & 109.2 & 10.1 & 90.0 \% \\
\hline
\end{tabular}
\caption[Results for prediction tasks on the five experimental domains]{Results for prediction tasks on five domains. We show the mean information size of the sensory input, to stress the scantiness of our sensory sequences. 
We also show the mean information size of the held-out data.
Our metric of accuracy for prediction tasks is whether the system predicted \emph{every} sensor's value correctly. }
\label{table:experiments-overview}
\end{table}


\subsubsection{Elementary cellular automata}
\label{sec:eca}

An Elementary Cellular Automaton (ECA) \cite{wolfram1983statistical,cook2004universality} is a one-dimensional Cellular Automaton. 
The world is a circular array of cells. 
Each cell can be either $\mathit{on}$ or $\mathit{off}$.
The state of a cell depends only on its previous state and the previous state of its left and right neighbours.

\begin{figure}
\begin{center}
\begin{tikzpicture}[b/.style={draw, minimum size=3mm,   
       fill=black},w/.style={draw, minimum size=3mm},
       m/.style={matrix of nodes, column sep=1pt, row sep=1pt, draw, label=below:#1}, node distance=1pt]

\matrix (A) [m=0]{
|[b]|&|[b]|&|[b]|\\
&|[w]|\\
};
\matrix (B) [m=1, right=of A]{
|[b]|&|[b]|&|[w]|\\
&|[b]|\\
};
\matrix (C) [m=1, right=of B]{
|[b]|&|[w]|&|[b]|\\
&|[b]|\\
};
\matrix (D) [m=0, right=of C]{
|[b]|&|[w]|&|[w]|\\
&|[w]|\\
};
\matrix (E) [m=1, right=of D]{
|[w]|&|[b]|&|[b]|\\
&|[b]|\\
};
\matrix (F) [m=1, right=of E]{
|[w]|&|[b]|&|[w]|\\
&|[b]|\\
};
\matrix (G) [m=1, right=of F]{
|[w]|&|[w]|&|[b]|\\
&|[b]|\\
};
\matrix (H) [m=0, right=of G]{
|[w]|&|[w]|&|[w]|\\
&|[w]|\\
};
\end{tikzpicture}  
\caption[Updates for ECA rule 110]{Updates for ECA rule 110. The top row shows the context: the target cell together with its left and right neighbour. The bottom row shows the new value of the target cell given the context. A cell is black if it is on and white if it is off.}
\label{fig:update110}
\end{center}
\end{figure}
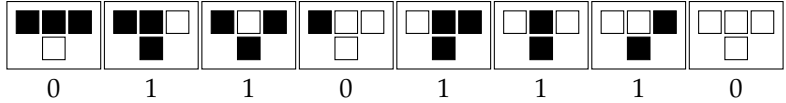
Figure \ref{fig:update110} shows one set of ECA update rules\footnote{This particular set of update rules is known as Rule $110$. Here, $110$ is the decimal representation of the binary $01101110$ update rule, as shown in Figure \ref{fig:update110}.  This rule has been shown to be Turing-complete \cite{cook2004universality}.}. 
Each update specifies the new value of a cell based on its previous left neighbour, its previous value, and its previous right neighbour. 
The top row shows the values of the left neighbour, previous value, and right neighbour.
The bottom row shows the new value of the cell.
There are 8 updates, one for each of the $2^3$ configurations.
In the diagram, the leftmost update states that if the left neighbour is $\mathit{on}$, and the cell is $\mathit{on}$, and its right neighbour is $\mathit{on}$, then at the next time-step, the cell will be turned $\mathit{off}$.
Given that each of the $2^3$ configurations can produce $\mathit{on}$ or $\mathit{off}$ at the next time-step, there are $2^{2^3} = 256$ total sets of update rules.

Given update rules for each of the 8 configurations, and an initial starting state, the trajectory of the ECA is determined.
Figure \ref{fig:trajectory110} shows the state sequence for Rule 110 above from one initial starting state of length 11.
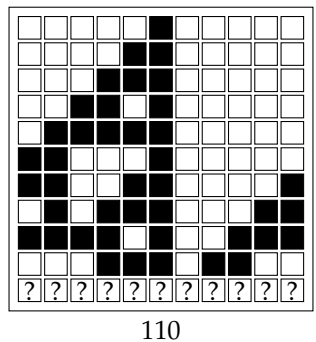
\begin{figure}
\begin{center}
\begin{tikzpicture}[b/.style={draw, minimum size=3mm,   
       fill=black},w/.style={draw, minimum size=3mm},
       m/.style={matrix of nodes, column sep=1pt, row sep=1pt, draw, label=below:#1}, node distance=1pt]

\matrix (A) [m=110]{
|[w]|&|[w]|&|[w]|&|[w]|&|[w]|&|[b]|&|[w]|&|[w]|&|[w]|&|[w]|&|[w]|\\
|[w]|&|[w]|&|[w]|&|[w]|&|[b]|&|[b]|&|[w]|&|[w]|&|[w]|&|[w]|&|[w]|\\
|[w]|&|[w]|&|[w]|&|[b]|&|[b]|&|[b]|&|[w]|&|[w]|&|[w]|&|[w]|&|[w]|\\
|[w]|&|[w]|&|[b]|&|[b]|&|[w]|&|[b]|&|[w]|&|[w]|&|[w]|&|[w]|&|[w]|\\
|[w]|&|[b]|&|[b]|&|[b]|&|[b]|&|[b]|&|[w]|&|[w]|&|[w]|&|[w]|&|[w]|\\
|[b]|&|[b]|&|[w]|&|[w]|&|[w]|&|[b]|&|[w]|&|[w]|&|[w]|&|[w]|&|[w]|\\
|[b]|&|[b]|&|[w]|&|[w]|&|[b]|&|[b]|&|[w]|&|[w]|&|[w]|&|[w]|&|[b]|\\
|[w]|&|[b]|&|[w]|&|[b]|&|[b]|&|[b]|&|[w]|&|[w]|&|[w]|&|[b]|&|[b]|\\
|[b]|&|[b]|&|[b]|&|[b]|&|[w]|&|[b]|&|[w]|&|[w]|&|[b]|&|[b]|&|[b]|\\
|[w]|&|[w]|&|[w]|&|[b]|&|[b]|&|[b]|&|[w]|&|[b]|&|[b]|&|[w]|&|[w]|\\
|[w]|&|[w]|&|[w]|&|[w]|&|[w]|&|[w]|&|[w]|&|[w]|&|[w]|&|[w]|&|[w]|\\
};
\node at (A-11-1) {?};
\node at (A-11-2) {?};
\node at (A-11-3) {?};
\node at (A-11-4) {?};
\node at (A-11-5) {?};
\node at (A-11-6) {?};
\node at (A-11-7) {?};
\node at (A-11-8) {?};
\node at (A-11-9) {?};
\node at (A-11-10) {?};
\node at (A-11-11) {?};
\end{tikzpicture}  
\end{center}
\caption[One trajectory for ECA rule 110]{One trajectory for Rule 110. Each row represents the state of the ECA at one time-step. In this prediction task, the bottom row (representing the final time-step) is held out. }
\label{fig:trajectory110}
\end{figure}

In our experiments, we attach sensors to each of the 11 cells, produce a sensory sequence, and then ask our system to find an interpretation that makes sense of the sequence.
For example, for the state sequence of Figure \ref{fig:trajectory110}, the sensory sequence is $(S_1, ..., S_{10})$ where:
\begin{eqnarray*}
S_1 & = & \set{\mathit{off}(c_1), \mathit{off}(c_2), \mathit{off}(c_3), \mathit{off}(c_4), \mathit{off}(c_5), \mathit{on}(c_6), \mathit{off}(c_7), \mathit{off}(c_8), \mathit{off}(c_(), \mathit{off}(c_{10}), \mathit{off}(c_{11})} \\
S_2 & = & \set{\mathit{off}(c_1), \mathit{off}(c_2), \mathit{off}(c_3), \mathit{off}(c_4), \mathit{on}(c_5), \mathit{on}(c_6), \mathit{off}(c_7), \mathit{off}(c_8), \mathit{off}(c_(), \mathit{off}(c_{10}), \mathit{off}(c_{11})} \\
S_3 & = & \set{\mathit{off}(c_1), \mathit{off}(c_2), \mathit{off}(c_3), \mathit{on}(c_4), \mathit{on}(c_5), \mathit{on}(c_6), \mathit{off}(c_7), \mathit{off}(c_8), \mathit{off}(c_(), \mathit{off}(c_{10}), \mathit{off}(c_{11})} \\
S_4 & = & \set{\mathit{off}(c_1), \mathit{off}(c_2), \mathit{on}(c_3), \mathit{on}(c_4), \mathit{off}(c_5), \mathit{on}(c_6), \mathit{off}(c_7), \mathit{off}(c_8), \mathit{off}(c_(), \mathit{off}(c_{10}), \mathit{off}(c_{11})} \\
S_5 & = & \set{\mathit{off}(c_1), \mathit{on}(c_2), \mathit{on}(c_3), \mathit{on}(c_4), \mathit{on}(c_5), \mathit{on}(c_6), \mathit{off}(c_7), \mathit{off}(c_8), \mathit{off}(c_(), \mathit{off}(c_{10}), \mathit{off}(c_{11})} \\
... &&
\end{eqnarray*}
Note that we do \emph{not} provide the spatial relation between cells. The system does not know that e.g.~cell $c_1$ is directly to the left of cell $c_2$. 

\paragraph{Results}
Given the 256 ECA rules, all with the same initial configuration, we treated the trajectories as a prediction task and applied our system to it.
Our system was able to predict 249/256 correctly.
In each of the 7/256 failure cases, the \sys{} found a unified interpretation, but this interpretation produced a prediction which was not the same as the oracle. 
For example, the dynamics found for Figure \ref{fig:trajectory110} above are:
\begin{eqnarray*}
R = \bigsetbegin{}
r(X, Y) \wedge \mathit{on}(X) \wedge \mathit{off}(Y) \fork \mathit{on}(Y) \\
r(X, Y) \wedge r(Y, Z) \wedge \mathit{on}(X) \wedge \mathit{on}(Z) \wedge \mathit{on}(Y) \fork \mathit{off}(Y)
\bigsetend{}
\end{eqnarray*}
These rules exactly capture the dynamics of the cells that \emph{change}. The other cells retain their value from the previous time-step, according to the frame axiom of Definition \ref{def:trace}.

The initial conditions found by the \sys{} describe the initial values of the cells, and also specify the latent $r$ relation between cells: 
\begin{eqnarray*}
I = \bigsetbegin{}
\begin{tabular}{lllllllllll}
$\mathit{off}(c_1)$ &
$\mathit{off}(c_2)$ &
$\mathit{off}(c_3)$ &
$\mathit{off}(c_4)$ &
$\mathit{off}(c_5)$ &
$\mathit{on}(c_6)$ &
$\mathit{off}(c_7)$ &
$\mathit{off}(c_8)$ &
$\mathit{off}(c_9)$ &
$\mathit{off}(c_{10})$ &
$\mathit{off}(c_{11})$ \\
$r(c_1, c_{11})$ &
$r(c_2, c_1)$ &
$r(c_3, c_2)$ &
$r(c_4, c_3)$ &
$r(c_5, c_4)$ &
$r(c_6, c_5)$ &
$r(c_7, c_6)$ &
$r(c_8, c_7)$ &
$r(c_9, c_8)$ &
$r(c_{10}, c_9)$ &
$r(c_{11}, c_{10})$
\end{tabular}
\bigsetend{}
\end{eqnarray*}
Here, the system uses $r(X, Y)$ to mean that cell $Y$ is immediately to the right of cell $X$.
Note that \emph{the system has constructed the spatial relation itself}. 
It was not given the spatial relation $r$ between cells.
All it was given was the sensor readings of the 11 cells.
It constructed the spatial relationship $r$ between the cells in order to make sense of the data.


\subsubsection{Drum rhythms and nursery tunes}
\label{sec:rhythms-and-tunes}

We also tested our system on simple melodies and rhythms.
Here, each sensor is an auditory receptor that is tuned to listen for a particular note or drum beat.
In the tune tasks, there is one sensor for $C$, one for $D$, one for $E$, all the way to $\mathit{HighC}$. (There are no flats or sharps).
In the rhythm tasks, there is one sensor listening out for bass drum, one for snare drum, and one for hi-hat.
Each sensor can distinguish four loudness levels, between 0 and 3. 
When a note is pressed, it starts at max loudness (3), and then decays down to 0.
Multiple notes can be pressed simultaneously. 

For example, the \emph{Twinkle Twinkle Little Star} tune generates the following sensor readings (assuming 8 time-steps for a bar):
\begin{eqnarray*}
S_1 & = & \set{ v(s_c, 3), v(s_d, 0), v(s_e, 0), v(s_f, 0),  v(s_g, 0), v(s_a, 0), v(s_b, 0), v(s_{c*}, 0) } \\
S_2 & = & \set{ v(s_c, 2), v(s_d, 0), v(s_e, 0), v(s_f, 0),  v(s_g, 0), v(s_a, 0), v(s_b, 0), v(s_{c*}, 0) } \\
S_3 & = & \set{ v(s_c, 3), v(s_d, 0), v(s_e, 0), v(s_f, 0),  v(s_g, 0), v(s_a, 0), v(s_b, 0), v(s_{c*}, 0) } \\
S_4 & = & \set{ v(s_c, 2), v(s_d, 0), v(s_e, 0), v(s_f, 0),  v(s_g, 0), v(s_a, 0), v(s_b, 0), v(s_{c*}, 0) } \\
S_5 & = & \set{ v(s_c, 1), v(s_d, 0), v(s_e, 0), v(s_f, 0),  v(s_g, 3), v(s_a, 0), v(s_b, 0), v(s_{c*}, 0) } \\
S_6 & = & \set{ v(s_c, 0), v(s_d, 0), v(s_e, 0), v(s_f, 0),  v(s_g, 2), v(s_a, 0), v(s_b, 0), v(s_{c*}, 0) } \\
...
\end{eqnarray*}

\paragraph{Results}
Recall that our accuracy metric is stringent and only counts a prediction as accurate if \emph{every} sensor's value is predicted correctly.
In the rhythm and music domain, this means the \sys{} must correctly predict the loudness value (between 0 and 3) for each of the sound sensors. There are 8 sensors for tunes and 3 sensors for rhythms. 
When we tested the \sys{} on the 20 drum rhythms and 10 nursery tunes, our system was able to predict 22/30 correctly.
Note that the interpretations found are large and complex programs by the standards of state-of-the-art ILP systems.
In \emph{Three Blind Mice}, for example, the interpretation contained 10 update rules and 34 initialisation atoms making a total of 44 clauses.

\subsubsection{\emph{Seek Whence} and C-test sequence induction intelligence tests}
\label{sec:seek-whence}

Hofstadter introduced the \define{Seek Whence}\footnote{The name is a pun on ``sequence''. See also the related Copycat domain \cite{mitchell1993analogy}.} domain in \cite{hofstadter2008fluid}. 
The task is, given a sequence $s_1, ..., s_t$ of symbols, to predict the next symbol $s_{t+1}$.
Typical \emph{Seek Whence} tasks include\footnote{Hofstadter used natural numbers, but we transpose the sequences to letters, to bring them in line with the Thurstone letter completion problems \cite{thurstone1941factorial} and the C-test \cite{hernandez1998formal}.}:
\begin{itemize}
\item
\sequence{b, b, b, c, c, b, b, b, c, c, b, b, b, c, c}
\item
\sequence{a, f, b, f, f, c, f, f, f, d, f, f}
\item
\sequence{b, a, b, b, b, b, b, c, b, b, d, b, b, e, b}
\end{itemize}
Hofstadter called the third sequence the ``theme song'' of the \emph{Seek Whence} project because of its difficulty.
There is a ``perceptual mirage'' in the sequence because of the sub-sequence of five $b$'s in a row that makes it hard to see the intended pattern: $(b,x,b)^*$ for ascending $x$.

\begin{figure}
\centering
\begin{tabular}{|l|l|}
\hline
\sequence{a,a,b,a,b,c,a,b,c,d,a} & 
\sequence{a,b,c,d,e}\\
\sequence{b,a,b,b,b,b,b,c,b,b,d,b,b,e} &
\sequence{a,b,b,c,c,c,d,d,d,d,e}\\
\sequence{a,f,e,f,a,f,e,f,a,f,e,f,a} &
\sequence{b,a,b,b,b,c,b,d,b,e}\\
\sequence{a,b,b,c,c,d,d,e,e} &
\sequence{a,b,c,c,d,d,e,e,e,f,f,f}\\
\sequence{f,a,f,b,f,c,f,d,f} &
\sequence{a,f,e,e,f,a,a,f,e,e,f,a,a}\\
\sequence{b,b,b,c,c,b,b,b,c,c,b,b,b,c,c} &
\sequence{b,a,a,b,b,b,a,a,a,a,b,b,b,b,b}\\
\sequence{b,c,a,c,a,c,b,d,b,d,b,c,a,c,a} &
\sequence{a,b,b,c,c,d,d,e,e,f,f}\\
\sequence{a,a,b,a,b,c,a,b,c,d,a,b,c,d,e} &
\sequence{b,a,c,a,b,d,a,b,c,e,a,b,c,d,f}\\
\sequence{a,b,a,c,b,a,d,c,b,a,e,d,c,b} &
\sequence{c,b,a,b,c,b,a,b,c,b,a,b,c,b}\\
\sequence{a,a,a,b,b,c,e,f,f} &
\sequence{a,a,b,a,a,b,c,b,a,a,b,c,d,c,b}\\
\sequence{a,a,b,c,a,b,b,c,a,b,c,c,a,a,a} &
\sequence{a,b,a,b,a,b,a,b,a}\\
\sequence{a,c,b,d,c,e,d} &
\sequence{a,c,f,b,e,a,d}\\
\sequence{a,a,f,f,e,e,d,d} &
\sequence{a,a,a,b,b,b,c,c}\\
\sequence{a,a,b,b,f,a,b,b,e,a,b,b,d} &
\sequence{f,a,d,a,b,a,f,a,d,a,b,a}\\
\sequence{a,b,a,f,a,a,e,f,a} &
\sequence{b,a,f,b,a,e,b,a,d}\\
\hline
\end{tabular}
\caption{Sequences from \emph{Seek Whence} and the C-test}
\label{diag:sw}
\end{figure}

\paragraph{Results}
Given the 30 \emph{Seek Whence} sequences, we treated the trajectories as a prediction task and applied our system to it.
Our system was able to predict 23/30 correctly.
For the 7 failure cases, 4 of them were due to the system not being able to find any unified interpretation within the memory and time limits, while in 3 of them, the system found a unified interpretation that produced the ``incorrect'' prediction.

The first key point we want to emphasise here is that our system was able to achieve human-level performance\footnote{See Meredith \cite{meredith1986seek} for empirical results 25 students  on the``Blackburn dozen'' \emph{Seek Whence} problems.}  on these tasks without hand-coded domain-specific knowledge.\footnote{The one piece of domain-specific knowledge we inject is the successor relation between the letters $a$, $b$, $c$, ...}
This is a \emph{general} system for making sense of sensory data that, when applied to the \emph{Seek Whence} domain, is able to solve these particular problems.
The second point we want to stress is that our system did not learn to solve these sequence induction tasks after seeing many previous examples.\footnote{Machine learning approaches to these tasks need thousands of examples before they can learn to predict. See for example \cite{barrett2018measuring}.} 
On the contrary: our system had never seen any such sequences before; it confronts each sequence \emph{de novo}, without prior experience.
This system is, to the best of our knowledge, the first such general system that is able to achieve such a result.

\subsubsection{Binding tasks}
\label{sec:probe-task-binding}

The binding problem \cite{holcombe2009binding} is the task of recognising that information from different sensory modalities should be collected together as different aspects of a single external object.
For example, you hear a buzzing and a siren in your auditory field and you see an insect and an ambulance in your visual field.
How do you associate the buzzing and the insect-appearance as aspects of one object, and the siren and the ambulance appearance as aspects of a separate object?

To investigate how our system handles such binding problems,  we tested it on the following multi-modal variant of the ECA described above.
Here, there are two types of sensor. The light sensors have just two states: black and white, while the touch sensors have four states: fully un-pressed (0), fully pressed (3), and two intermediate states (1, 2).
After a touch sensor is fully pressed (3), it slowly depresses, going from states 2 to 1 to 0 over 3 time-steps.
In this example, we chose Rule 110 (the Turing-complete ECA rule) with the same initial configuration as in Figure \ref{fig:trajectory110}, as described earlier.
In this multi-modal variant, there are 11 light sensors, one for each cell in the ECA, and two touch sensors on cells 3 and 11.
See Figure \ref{fig:binding}.

\paragraph{Results} We ran 20 multi-modal binding experiments, with different ECA rules, different initial conditions, and the touch sensors attached to different cells. The engine achieved 85\% accuracy.

\begin{figure}
\begin{center}
\begin{tikzpicture}[b/.style={draw, minimum size=3mm,   
       fill=black},w/.style={draw, minimum size=3mm},
       m/.style={matrix of nodes, column sep=1pt, row sep=1pt, draw, label=below:#1}, node distance=1pt]

\matrix (A) [m=110]{
$l_1$&$l_2$&${\color{red}l_3}$&$l_4$&$l_5$&$l_6$&$l_7$&$l_8$&$l_9$&$l_{10}$&${\color{blue}l_{11}}$&$t_1$&$t_2$\\
W&W&{\color{red}W}&W&W&B&W&W&W&W&{\color{blue}W}&{\color{red}0}&{\color{blue}0}\\
W&W&{\color{red}W}&W&B&B&W&W&W&W&{\color{blue}W}&{\color{red}0}&{\color{blue}0}\\
W&W&{\color{red}W}&B&B&B&W&W&W&W&{\color{blue}W}&{\color{red}0}&{\color{blue}0}\\
W&W&{\color{red}B}&B&W&B&W&W&W&W&{\color{blue}W}&{\color{red}3}&{\color{blue}0}\\
W&B&{\color{red}B}&B&B&B&W&W&W&W&{\color{blue}W}&{\color{red}3}&{\color{blue}0}\\
B&B&{\color{red}W}&W&W&B&W&W&W&W&{\color{blue}W}&{\color{red}2}&{\color{blue}0}\\
B&B&{\color{red}W}&W&B&B&W&W&W&W&{\color{blue}B}&{\color{red}1}&{\color{blue}3}\\
W&B&{\color{red}W}&B&B&B&W&W&W&B&{\color{blue}B}&{\color{red}0}&{\color{blue}3}\\
B&B&{\color{red}B}&B&W&B&W&W&B&B&{\color{blue}B}&{\color{red}3}&{\color{blue}3}\\
W&W&{\color{red}W}&B&B&B&W&B&B&W&{\color{blue}W}&{\color{red}2}&{\color{blue}2}\\
?&?&?&?&?&?&?&?&?&?&?&?&?\\
};
\end{tikzpicture}  
\end{center}
\caption[A multi-modal trace of ECA rule 110 with light sensors and touch sensors]{A multi-modal trace of ECA rule 110 with eleven light sensors (left) $l_1, ..., l_{11}$ and two touch sensors (right) $t_1, t_2$ attached to cells 3 and 11. Each row represents the states of the sensors for one time-step. For this prediction task, the final time-step is held out.}
\label{fig:binding}
\end{figure}

\subsubsection{Occlusion tasks}
\label{sec:probe-task-binding}

Neural nets that predict future sensory data conditioned on past sensory data struggle to solve occlusion tasks because it is hard to inject into them the prior knowledge that objects persist over time.
Our system, by contrast, was designed to posit latent objects that persist over time.

To test our system's ability to solve occlusion problems, we generated a set of tasks of the following form:
there is a 2D grid of cells in which objects move horizontally.
Some move from left to right, while others move from right to left, with wrap around when they get to the edge of a row.
The objects move at different speeds. 
Each object is placed in its own row, so there is no possibility of collision.
There is an ``eye'' placed at the bottom of each column, looking up. 
Each eye can only see the objects in the column it is placed in.
An object is occluded if there is another object below it in the same column.
See Figure \ref{fig:occlusion}.

\begin{figure}
\centering
\includegraphics[scale=0.50]{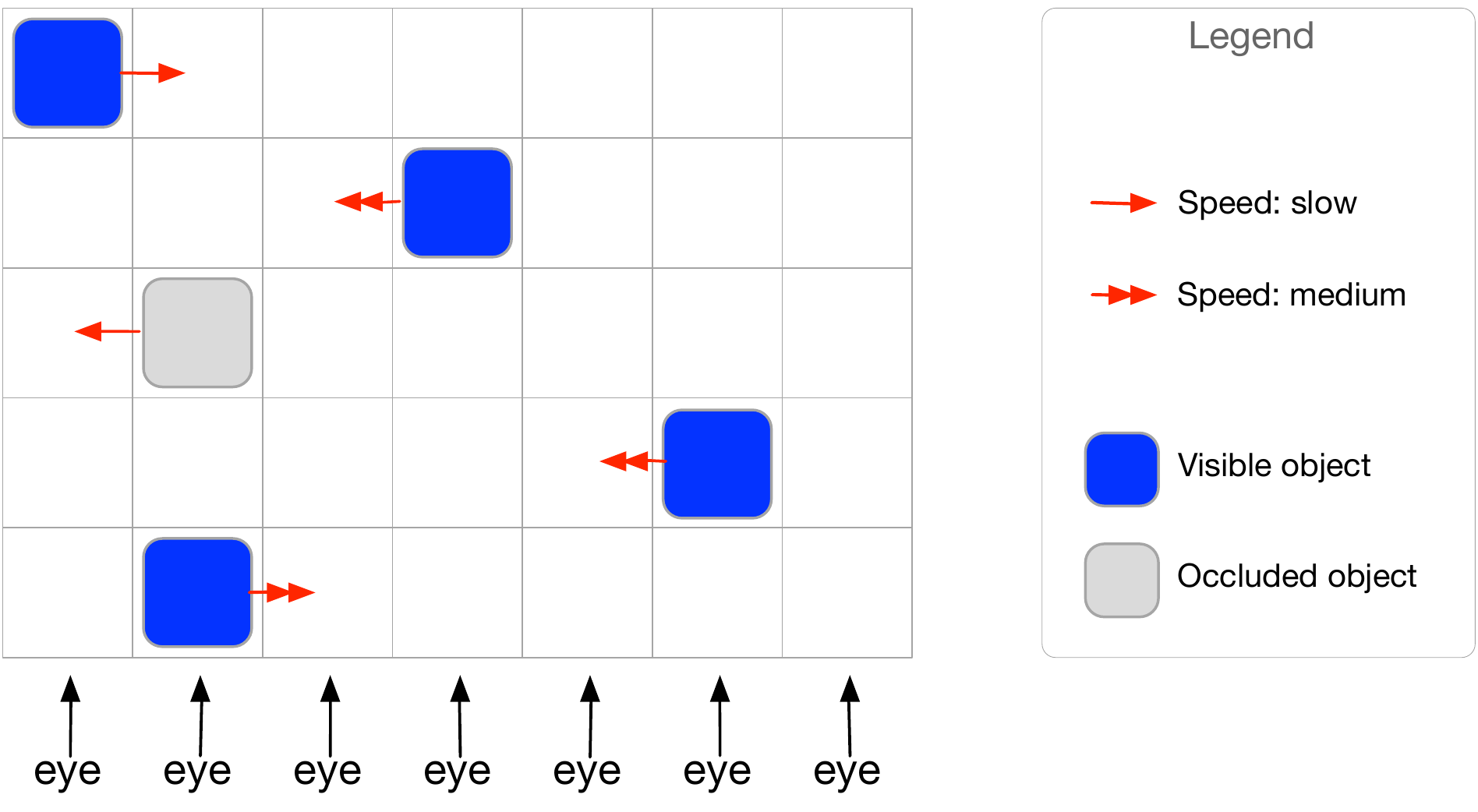}
\caption{An occlusion task}
\label{fig:occlusion}
\end{figure}

The system receives a sensory sequence consisting of the positions of the moving objects whenever they are visible.
The positions of the objects when they are occluded is used as held-out test data to verify the predictions of the model.
This is an imputation task.

\paragraph{Results} 
We generated 20 occlusion tasks by varying the size of the grid, the number of moving objects, their direction and speed.
The \sys{} achieved 90\% accuracy.

\subsection{Empirical evaluations}
\label{sec:results}

In this section, we test our system to evaluate the truth of the following hypotheses:
\begin{enumerate}
\item
The five domains of section \ref{sec:domains} are challenging tasks for existing baselines.
\item
Our system handles retrodiction and imputation just as easily as prediction.
\item
The features of our system (unity conditions, cost minimisation) are essential to its performance.
\item
Our system outperforms state-of-the-art inductive logic programming approaches in the five domains.
\end{enumerate}
We consider each in turn.

\subsubsection{The five domains of section \ref{sec:domains} are challenging tasks for existing baselines}
\label{sec:baselines}
To evaluate whether our domains are indeed sufficiently challenging, we compared our system against four baselines.
The first \define{constant} baseline always predicts the same constant value for every sensor for each time-step.
The second \define{inertia} baseline always predicts that the final hidden time-step equals the penultimate time-step. 
The third \define{MLP} baseline is a fully-connected multilayer perceptron (MLP) \cite{murphy2012machine} that looks at a window of earlier time-steps to predict the next time-step. 
The fourth \define{LSTM} baseline is a recurrent neural net based on the long short-term memory (LSTM) architecture \cite{hochreiter1997long}.

The neural baselines are designed to exploit potential statistical patterns that are indicative of hidden sensor states. In the MLP baseline, we formulate the problem as a multi-class classification problem, where the input consists in a feature representation ${\bf x}$ of relevant context sensors, and a feed-forward network is trained to predict the correct state ${\bf y}$ of a given sensor in question. In the prediction task, the feature representation comprises one-hot\footnote{A one-hot representation of feature $i$ of $n$ possible features is a vector of length $n$ in which all the elements are 0 except the $i$'th element.}  representations for the state of every sensor in the previous two time steps before the hidden sensor. The training data consists of the collection of all observed states in an episode (as potential hidden sensors), together with the respective history before. Samples with incomplete history window (at the beginning of the episode) are discarded.

The MLP classifier is a 2-layer feed-forward neural network, which is trained on all training examples derived from the current episode (thus no cross-episode transfer is possible). We restrict the number of hidden neurons to (20, 20) for the two layers, respectively, in order to prevent overfitting given the limited number of training points within an episode. We use a learning rate of $10^{-3}$ and train the model using the \emph{Adam} optimiser \cite{kingma2014adam} for up to $200$ epochs, holding aside 10\% of data for early stopping. 

Given that the input is a temporal sequence, a recurrent neural network (that was designed to model temporal dynamics) is a natural choice of baseline.
But we found that the LSTM performs only slightly better than the MLP on Seek Whence tasks, and worse on the other tasks. 
The reason for this is that the paucity of data (a single temporal sequence consisting of a small number of time-steps) does not provide enough information for the high capacity LSTM to learn desirable gating behaviour. The simpler and more constrained MLP with fewer weights is able to do slightly better on some of the tasks, yet both neural baselines achieve low accuracy in absolute terms.

Figure \ref{fig:baselines-chart} shows the results. Clearly, the tasks are very challenging for all four baseline systems.

\begin{figure}
\centering
\begin{tikzpicture}
\begin{axis}[
    ybar,
    enlargelimits=0.25,
    legend style={at={(0.5,-0.15)},
      anchor=north,legend columns=-1},
    ylabel={predictive accuracy},
    symbolic x coords={eca,music,Seek-Whence},
    xtick=data,
    ]
\addplot[ybar, pattern=dots] coordinates {(eca,97) (music,73) (Seek-Whence,76)};
\addplot[ybar, pattern=north west lines] coordinates {(eca,8) (music,2) (Seek-Whence,26)};
\addplot[ybar, pattern=north east lines] coordinates {(eca,29) (music,0) (Seek-Whence,33)};
\addplot[ybar, pattern=grid] coordinates {(eca,14) (music,7) (Seek-Whence,17)};
\addplot[ybar, pattern=horizontal lines] coordinates {(eca,3) (music,0) (Seek-Whence,18)};
\legend{our system (AE),constant baseline, inertia baseline, MLP baseline, LSTM baseline}
\end{axis}
\end{tikzpicture}
\caption[Comparison with baselines]{Comparison with baselines. We display predictive accuracy on the held-out final time-step.}
\label{fig:baselines-chart}
\end{figure}

\subsubsection{Our system handles retrodiction and imputation just as easily as prediction}
\label{sec:retrodiction-imputation}

To evaluate whether our system is just as capable of retrodicting earlier values and imputing missing intermediate values as it is at predicting future values, we ran tests where the unseen hidden sensor values were at the first time step (in the case of retrodiction) or randomly scattered through the time-series (in the case of imputation).\footnote{A deterministic transition dynamic is a function from a set of ground atoms to a set of ground atoms. If this function is not injective, then information is lost as we go through time: there will be a unique next step from the current time-step, but there can be multiple previous steps that transition into the current time-step. Our search procedure uses maximum a posteriori (MAP) estimation: we find a single model with the highest posterior (based on the likelihood -- how well it explains the sequence, and on the prior -- the program length), and use that model to predict, retrodict, and impute. A more ambitious Bayesian approach would construct a probability distribution over rival theories, and use a mixture model for retrodiction. But, given the computation complexity of finding a single solution (see Section \ref{sec:complexity}),  this ambitious approach is -- in the short term at least -- prohibitively expensive.}
We made sure that the number of hidden sensor values was the same for prediction, retrodiction, and imputation.

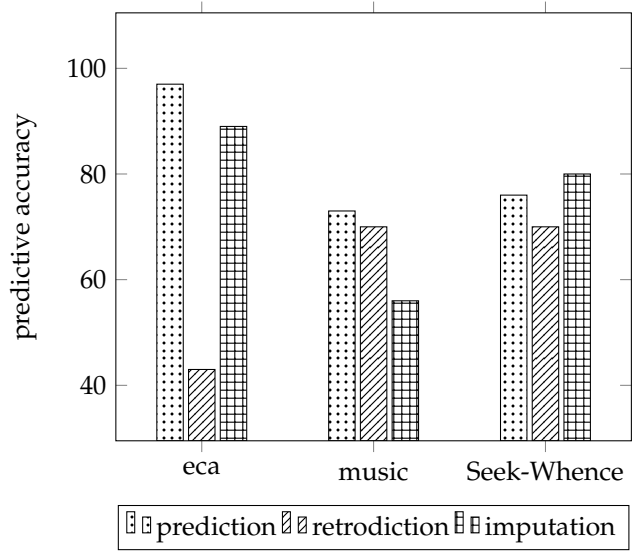
\begin{figure}
\centering
\begin{tikzpicture}
\begin{axis}[
    ybar,
    enlargelimits=0.25,
    legend style={at={(0.5,-0.15)},
      anchor=north,legend columns=-1},
    ylabel={predictive accuracy},
    symbolic x coords={eca,music,Seek-Whence},
    xtick=data,
    ]
\addplot[ybar, pattern=dots] coordinates {(eca,97) (music,73) (Seek-Whence,76)};
\addplot[ybar, pattern=north east lines] coordinates {(eca,43) (music,70) (Seek-Whence,70)};
\addplot[ybar, pattern=grid] coordinates {(eca,89) (music,56) (Seek-Whence,80)};

\legend{prediction,retrodiction,imputation}
\end{axis}
\end{tikzpicture}
\caption[Comparing prediction with retrodiction and imputation]{Comparing prediction with retrodiction and imputation. In retrodiction, we display accuracy on the held-out initial time-step. In imputation, a random subset of atoms are held-out; the held-out atoms are scattered throughout the time-series. In other words, there may be different held-out atoms at different times. The number of held-out atoms in imputation matches the number of held-out atoms in prediction and retrodiction.}
\label{fig:retrodiction-and-imputation-chart}
\end{figure}

Figure \ref{fig:retrodiction-and-imputation-chart} shows the results.
The results are significantly lower for retrodiction in the ECA tasks, but otherwise comparable.
The reason for retrodiction's lower performance on ECA is that for a particular initial configuration there are a significant number (more than 50\%) of the ECA rules that wipe out all the information in the current state after the first state transition, and all subsequent states then remain the same. 
The results for imputation are comparable with the results for prediction. Although the results for rhythm and music are lower, the results on Seek Whence are slightly higher.

\subsubsection{The features of our system are essential to its performance.}
\label{sec:ablation}

To verify that the unity conditions are doing useful work, we performed a number of experiments in which the various conditions were removed, and compared the results.
We ran four ablation experiments.
In the first, we removed the check that the theory's trace covers the input sequence: $S \sqsubseteq \tau(\theta)$
(see Definition \ref{def:makes-sense}).
In the second, we removed the check on conceptual unity.
Removing this condition means that the unary predicates are no longer connected together via exclusion relations $\oplus$, and the binary predicates are no longer constrained by $\exists !$ conditions.
(See Definition \ref{def:conceptual-unity}).
In the third ablation test, we removed the check on spatial unity.
Removing this condition means allowing objects which are not connected via binary relations.
In the fourth ablation test, we removed the cost minimization part of the system.
Removing this minimization means that the system will return the first interpretation it finds, irrespective of size.

The results of the ablation experiments are displayed in Table \ref{table:ablation}.
The first ablation test, where we remove the check that the generated sequence of sets of ground atoms respects the original sensory sequence ($S \sqsubseteq \tau(\theta)$), performs very poorly. 
Of course, if the generated sequence does not cover the given part of the sensory sequence, it is highly unlikely to accurately predict the held-out part of the sensory sequence. 
This test is just a sanity check that our evaluation scripts are working as intended.

The second ablation test, where we remove the check on conceptual unity, also performs poorly.
The reason is that without constraints, there are no incompossible atoms.
Recall from Definition \ref{def:trace} that two atoms are incompossible if there is some $\oplus$ constraint or some $\exists !$ constraint that means the two atoms cannot be simultaneously true.
But in Definition \ref{def:trace}, the frame axiom forces an atom that was true at the previous time-step to also be true at the next time-step unless the old atom is incompossible with some new atom:
we add $\alpha$ to $H_t$ if $\alpha$ is in $H_{t-1}$ and there is no atom in $H_t$ that is incompossible with $\alpha$.
But if there are no incompossible atoms, then all previous atoms are always added.
Therefore, if there are no $\oplus$ and $\exists !$ constraints, then the set of true atoms monotonically increases over time.
This in turn means that state information becomes meaningless, as once something becomes true, it remains always true, and cannot be used to convey information. 

\begin{table}
\centering
\begin{tabular}{|l|r|r|r|}
\hline
& {\bf ECA} & {\bf Rhythm \& Music}  & {\bf Seek Whence} \\
\hline
Full system (AE) & 97.3\% & 73.3\% & 76.7\% \\
\hline
No check that $S \sqsubseteq \tau(\theta)$ & 5.1\% & 3.0\% & 4.6\% \\
\hline
No conceptual unity & 5.3\% & 0.0\% & 6.7\% \\
\hline
No spatial unity & 95.7\% & 73.3\% & 73.3\% \\
\hline
No cost minimization & 96.7\% & 56.6\% & 73.3\% \\
\hline
\end{tabular}
\caption[Ablation experiments]{Ablation experiments. We display predictive accuracy on the final held-out time-step.}
\label{table:ablation}
\end{table}

When we remove the spatial unity constraint, the results for the rhythm tasks are identical, but the results for the ECA and Seek Whence tasks are lower. 
The reason why the results are identical for the rhythm tasks is because the background knowledge provided (the $r$ relation on notes, see Section \ref{sec:rhythms-and-tunes}) means that the spatial unity constraint is guaranteed to be satisfied. 
The reason why the results are lower for ECA tasks is because interpretations that fail to satisfy spatial unity contain disconnected clusters of cells (e.g.~cells $\{c_1, ..., c_5\}$ are connected by $r$ in one cluster, while cells $\{c_6, ..., c_{11}\}$ are connected in another cluster, but  $\{c_1, ..., c_5\}$ and $\{c_6, ..., c_{11}\}$ are disconnected). Interpretations with disconnected clusters tend to generalize poorly and hence predict with less accuracy. 
The reason why the results are only slightly lower for the Seek Whence tasks is because the lowest cost unified interpretation for most of these tasks also happens to satisfy spatial unity. 

The results for the fourth ablation test, where we remove the cost minimization, are broadly comparable with the full system in ECA and Seek Whence, but are markedly worse in the rhythm / music tasks. 
But even if the results were comparable in all tasks, there are independent reasons to want to minimize the size of the interpretation. Shorter interpretations are more human-readable, and transfer better to new situations (since they tend to be more general, as they have fewer atoms in the bodies of the rules). 

\subsubsection{Our system outperforms state-of-the-art inductive logic programming approaches in the five domains}
\label{sec:ilasp}

In order to assess the efficiency of our system, we compared it to ILASP\footnote{We compared against ILASP rather than Metagol (another state-of-the-art inductive logic programming system \cite{mugg:metagold,cropper}) because (i) ILASP is comparable in performance (it achieved slightly better results than Metagol in the Inductive General Game Playing task suite \cite{cropper2019inductive}, getting 40\% correct as opposed to Metagol's 36\%), and (ii) since ILASP also uses ASP we can compare the grounding size of our program with ILASP and get a fair apples-for-apples comparison. We used ILASP rather than HEXMIL (the ASP implementation of Metagol  \cite{kaminski2018exploiting}) because of scaling problems with HEXMIL \cite{cropper2019learning}. Our decision to use ILASP rather than Metagol for these tests was based on a number of discussions with Andrew Cropper (the developer of Metagol) and Mark Law (the developer of ILASP). We are very grateful to both for their advice on this.} \cite{law2014inductive,law2015learning,law2016iterative,law2018complexity}, a state-of-the-art Inductive Logic Programming algorithm\footnote{Strictly speaking, ILASP is a family of 
algorithms, rather than a single algorithm. We used ILASP2 \cite{law2015learning} in this evaluation.}.
Unlike traditional ILP systems that learn definite logic programs, ILASP learns \emph{answer set programs}\footnote{Answer set programming under the stable model semantics is distinguished from traditional logic programming in that it is purely declarative and each program has multiple solutions (known as answer sets). Because of its non-monotonicity, ASP is well suited for knowledge representation and common-sense reasoning \cite{mueller2014commonsense,gelfond2014knowledge}.}.
ILASP is a powerful and general framework for learning answer set programs; it is able to learn choice rules, constraints, and even preferences over answer sets \cite{law2015learning}.

ILASP is able to solve some simple apperception tasks. For example, ILASP is able to solve the task in Example \ref{example:eca1}. 
But for the ECA tasks, the music and rhythm tasks, and the Seek Whence tasks, the ASP programs generated by ILASP were not solvable because they required too much memory.

In order to understand the memory requirements of ILASP on these tasks, and to compare our system with ILASP in a fair like-for-like manner, we looked at the size of the grounded ASP programs.
Recall that both our system and ILASP generate ASP programs that are then grounded into propositional clauses that are then passed to a SAT solver.
The grounding size determines the memory usage and is strongly correlated with solution time.

We took a sample ECA, Rule 245, and looked at the grounding size as the number of cells increased from 2 to 11.
The results are in Figure \ref{fig:grounding-comparison}.

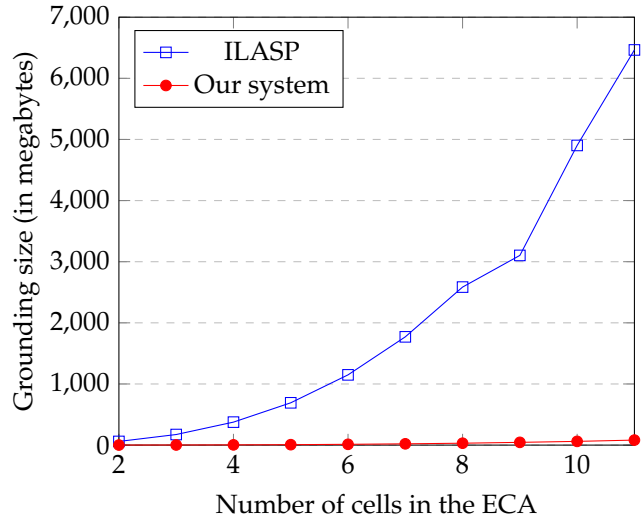
\begin{figure}
\begin{center}
\begin{tikzpicture}
\begin{axis}[
    title={},
    xlabel={Number of cells in the ECA},
    ylabel={Grounding size (in megabytes)},
    xmin=2, xmax=11,
    ymin=0, ymax=7000,
    xtick={2,4,6,8,10},
    ytick={0,1000,2000,3000,4000,5000,6000,7000},
    legend pos=north west,
    ymajorgrids=true,
    grid style=dashed,
]
 
\addplot[
    color=blue,
    mark=square,
    ]
    coordinates {
    (2,60)(3,173)(4,376)(5,692)(6,1149)(7,1771)(8,2585)(9,3103)(10,4902)(11,6464)
    };
    \addlegendentry{ILASP}

\addplot[
    color=red,
    mark=*,
    ]
    coordinates {
    (2,1)(3,2)(4,4)(5,7)(6,13)(7,21)(8,31)(9,45)(10,61)(11,82)
    };
    \addlegendentry{Our system}
 
\end{axis}
\end{tikzpicture}
\end{center}
\caption{Comparing our system and ILASP w.r.t. grounding size}
\label{fig:grounding-comparison}
\end{figure}

As we increase the number of cells, the grounding size of the ILASP program grows much faster than the corresponding  \sys{} program. 
The reason for this marked difference is the different ways the two approaches represent rules.
In our system, rules are interpreted by an interpreter that operates on reified representations of rules.
In ILASP, by contrast, rules are \emph{compiled} into ASP rules.
This means, if there are $|U_\phi|$ unground atoms and there are at most $N_B$ atoms in the body of a rule, then ILASP will generate $|U_\phi|^{N_B+1}$ different clauses.
When it comes to grounding, if there are $|\Sigma_\phi|$ substitutions and $t$ time-steps, then ILASP will generate at most
$|U_\phi|^{N_B+1} \cdot |\Sigma_\phi| \cdot t$ ground instances of the generated clauses.
Each ground instance will contain $N_B + 1$ atoms, so there are $(N_B + 1) \cdot |U_\phi|^{N_B+1} \cdot |\Sigma_\phi| \cdot t$ ground atoms in total.

Compare this with our system. 
Here, we do not represent every possible rule explicitly as a separate clause.
Rather, we represent the possible atoms in the body of a rule by an ASP choice rule:
\begin{verbatim}
0 { rule_body(R, VA) : is_unground_atom(VA) } k_max_body :- is_rule(R).
\end{verbatim}
If there are $N_\rightarrow$ static rules and $N_{\fork}$ causal rules, then this choice rule only generates $N_\rightarrow + N_{\fork}$ ground clauses, each containing $|U_\phi|$ atoms. 

The most expensive clauses in our encoding are analysed in Table \ref{table:num-ground-clauses}.
Recall from Section \ref{sec:complexity} that the total number of atoms in the ground clauses is approximately $5 \cdot |\Sigma_\phi| \cdot (N_{\rightarrow} + N_{\fork}) \cdot |U_\phi| \cdot t$. 

To compare this with ILASP, let us set $N_B = 4$ (which is representative). 
Then ILASP generates ground clauses with $5 \cdot |U_\phi|^5 \cdot |\Sigma_\phi| \cdot t$ ground atoms while our system generates clauses with  $5 \cdot |\Sigma_\phi| \cdot (N_{\rightarrow} + N_{\fork}) \cdot |U_\phi| \cdot t$ ground atoms.
The reason, then, why our system has such lower grounding sizes than ILASP is because $(N_{\rightarrow} + N_{\fork})  << |U_\phi|^4$.
Intuitively, the key difference is that ILASP considers \emph{every possible} subset of the hypothesis space, while our system (by restricting to at most $N_{\rightarrow} + N_{\fork}$ rules) only considers \emph{subsets of length at most $N_{\rightarrow} + N_{\fork}$}.

\section{Noisy apperception}
\label{sec:noise}

So far, we have assumed that our sensor readings are entirely noise-free:
some of the sensory readings may be missing, but none of the readings are inaccurate.

If we give the \sys{} a sensory sequence with mislabeled data, it will struggle to provide a theoretical explanation of the mislabeled input.
Consider, for example, $S_{1:20}$:
\begin{eqnarray*}
\begin{tabular}{lllll}
$S_1 = \set{p(a)}$  & $S_2 = \set{p(a)}$ & $S_3 = \set{ \mathbf{q(a)}}$ & $S_4 = \set{p(a)}$  & $S_5 = \set{p(a)}$ \\
 $S_6 = \set{ p(a)}$ & $S_7 = \set{p(a)}$  & $S_8 = \set{p(a)}$ & $S_9 = \set{ p(a)}$ & $S_{10} = \set{p(a)}$  \\
$S_{11} = \set{p(a)}$ & $S_{12}= \set{ p(a)}$ & $S_{13} = \set{p(a)}$  & $S_{14} = \set{p(a)}$ & $S_{15} = \set{ p(a)}$ \\
$S_{16} = \set{p(a)}$  & $S_{17} = \set{p(a)}$ & $S_{18} = \set{ p(a)}$ &$S_{19} = \set{p(a)}$ & $S_{20} = \set{p(a)}$
\end{tabular}
\end{eqnarray*}
Here, $S_3 = \set{ \mathbf{q(a)}}$ is an outlier in the otherwise tediously predictable sequence.

If we give sequences such as this to the \sys{}, it attempts to make sense of \emph{all} the input, including the anomalies. 
In this case, it finds the following baroque explanation:
\begin{eqnarray*}
\begin{tabular}{lll}
$I = \bigsetbegin{}
p(a) \\
c_1(a)
\bigsetend{}$ &
$R = \bigsetbegin{}
q(X) \rightarrow c_3(X) \\
c_3(X) \fork p(X) \\
c_1(X) \fork c_2(X) \\
c_2(X) \fork q(X)
\bigsetend{}$ &
$C' = \bigsetbegin{}
\forall X {:}s, \; p(X) \oplus q(X) \\
\forall X {:}s, \; c_1(X) \oplus c_2(X) \oplus c_3(X)
\bigsetend{}$
\end{tabular}
\end{eqnarray*}
Here, the \sys{} has introduced three new invented predicates $c_1, c_2, c_3$ in order to count how many $p$'s it has seen, so that it knows when to switch to $q$.
If we move the anomalous entry $q(a)$ later in the sequence, or add further anomalies, the engine is forced to construct increasingly complex theories.
This is clearly unsatisfactory.

In order to handle noisy mislabeled data, we shall relax our insistence that the sequence $S_{1:T}$ is \emph{entirely} covered by the trace of the theory $\theta$.
Instead of insisting that $S \sqsubseteq \tau(\theta)$, we shall minimise the number of discrepancies between each $S_i$ and $\tau(\theta)_i$, for $i = 1 .. T$.

We want to find the most probable theory $\theta$ given our noisy input sequence $S_{1:T}$:
\begin{eqnarray}
\label{eqn:noisy-1}
\argmax_{\theta} \; p\left(\theta \mid S_{1:T}\right)
\end{eqnarray}
By Bayes' rule, this is equivalent to
\begin{eqnarray}
\label{eqn:noisy-2}
\argmax_{\theta} \; \frac{p(\theta) \cdot p\left(S_{1:T} \mid \theta \right)}{p(S_{1:T})}
\end{eqnarray}
Since the denominator does not depend on $\theta$, this is equivalent to:
\begin{eqnarray}
\label{eqn:noisy-3}
\argmax_{\theta} \; p(\theta) \cdot p\left(S_{1:T} \mid \theta \right)
\end{eqnarray}
Since the probability of the state $S_i$ is conditionally independent of the previous state $S_{i-1}$ given $\theta$ (this is the assumption behind the Hidden Markov Model), the above is equivalent to:
\begin{eqnarray}
\label{eqn:noisy-4}
\argmax_{\theta} \; p(\theta) \cdot \prod_{i=1}^T p\left(S_i \mid \theta \right)
\end{eqnarray}
Now each $S_i$ depends only on $\tau(\theta)_i$, the trace of $\theta$ at time step $i$.
Thus we can rewrite to:
\begin{eqnarray}
\label{eqn:noisy-5}
\argmax_{\theta} \; p(\theta) \cdot \prod_{i=1}^T p\left(S_i \mid \tau(\theta)_i \right)
\end{eqnarray}
Since we assume a universal distribution on theories, the probability of $\theta$ is $2^{-\mathit{len}(\theta)}$.
Let the probability of  $S_i$ given $\tau(\theta)_i$ be $p\left(S_i \mid \tau(\theta)_i \right) = 2^{-|S_i - \tau(\theta)_i|}$.
Then we can rewrite to:
\begin{eqnarray}
\label{eqn:noisy-6}
\argmax_{\theta} \; 2^{-\mathit{len}(\theta)} \cdot \prod_{i=1}^T 2^{-|S_i - \tau(\theta)_i|}
\end{eqnarray}
Thus, we define the \define{$\mathit{cost}_\mathit{noise}$} of the theory to be:
\begin{eqnarray}
\label{eqn:noisy-8}
\mathit{cost}_\mathit{noise} = \mathit{len}(\theta) + \beta \sum_{i=1}^T |S_i - \tau(\theta)_i|
\end{eqnarray}
and search for the theory with lowest cost.

\begin{example}
Consider, for example, the following sequence $S_{1:10}$:
\begin{eqnarray*}
\begin{tabular}{lllll}
$S_1 = \set{}$  & $S_2 = \set{ \mathit{off}(a), \mathit{on}(b)}$ & $S_3 = \set{ \mathit{on}(a),  \mathit{off}(b)}$ & $S_4 = \set{ \mathit{on}(a), \mathit{on}(b)}$ & $S_5 = \set{ \mathit{on}(b)}$ \\
$S_6 = \set{ \mathit{on}(a),  \mathit{off}(b)}$ & $S_7 = \set{ \mathit{on}(a), \mathit{on}(b)}$ & $S_8 = \set{ \mathit{off}(a), \mathit{on}(b)}$ & $S_9 = \set{ \mathit{on}(a)}$ & $S_{10} = \set{ }$
\end{tabular}
\end{eqnarray*}
Because the sequence is so short, the lowest $\mathit{cost}_\mathit{noise}$ theory is:
\begin{eqnarray*}
\begin{tabular}{lll}
$I = \bigsetbegin{}
\bigsetend{}$ &
$R = \bigsetbegin{}
\bigsetend{}$ &
$C' = \bigsetbegin{}
\forall X {:}s, \; \mathit{on}(X) \oplus \mathit{off}(X)
\bigsetend{}$
\end{tabular}
\end{eqnarray*}
This degenerate empty theory has cost 14 (the number of atoms in $S$) which is shorter than any ``proper'' explanation that captures the regularities.
But as the sequence gets longer, the advantage of a ``proper'' explanation over a degenerate solution becomes more and more apparent.
Consider, for example, the following extension $S'_{1:30}$:
\begin{eqnarray*}
\begin{tabular}{lllll}
$S'_1 = \set{}$  & $S'_2 = \set{ \mathit{off}(a), \mathit{on}(b)}$ & $S'_3 = \set{ \mathit{on}(a),  \mathit{off}(b)}$ &
$S'_4 = \set{ \mathit{on}(a), \mathit{on}(b)}$ & $S'_5 = \set{ \mathit{on}(b)}$ \\
$S'_6 = \set{ \mathit{on}(a),  \mathit{off}(b)}$ & $S'_7 = \set{ \mathit{on}(a), \mathit{on}(b)}$ & $S'_8 = \set{ \mathit{off}(a), \mathit{on}(b)}$ & $S'_9 = \set{ \mathit{on}(a)}$ & $S'_{10} = \set{ }$ \\ 
$S'_{11}= \set{ \mathit{off}(a), \mathit{on}(b)}$ &  $S'_{12} = \set{ \mathit{on}(a),  \mathit{off}(b)}$ & $S'_{13} = \set{ \mathit{on}(a), \mathit{on}(b)}$ & $S'_{14}= \set{ \mathit{off}(a), \mathit{on}(b)}$ &  $S'_{15} = \set{ \mathit{on}(a),  \mathit{off}(b)}$ \\
$S'_{16} = \set{ \mathit{on}(a), \mathit{on}(b)}$ & $S'_{17}= \set{ \mathit{off}(a), \mathit{on}(b)}$ &  $S'_{18} = \set{ \mathit{on}(a),  \mathit{off}(b)}$ & $S'_{19} = \set{ \mathit{on}(a), \mathit{on}(b)}$ & $S'_{20}= \set{ \mathit{off}(a), \mathit{on}(b)}$ \\  
$S'_{21} = \set{ \mathit{on}(a),  \mathit{off}(b)}$ & $S'_{22} = \set{ \mathit{on}(a), \mathit{on}(b)}$ & $S'_{23}= \set{ \mathit{off}(a), \mathit{on}(b)}$ &  $S'_{24} = \set{ \mathit{on}(a),  \mathit{off}(b)}$ & $S'_{25} = \set{ \mathit{on}(a), \mathit{on}(b)}$ \\
$S'_{26}= \set{ \mathit{off}(a), \mathit{on}(b)}$ &  $S'_{27} = \set{ \mathit{on}(a),  \mathit{off}(b)}$ & $S'_{28} = \set{ \mathit{on}(a), \mathit{on}(b)}$ & $S'_{29}= \set{ \mathit{off}(a), \mathit{on}(b)}$ &  $S'_{30} = \set{ }$
\end{tabular}
\end{eqnarray*}
Now the lowest $\mathit{cost}_\mathit{noise}$ theory is one that finds the underlying regularity:
\begin{eqnarray*}
\begin{tabular}{lll}
$I = \bigsetbegin{}
\mathit{on}(a) \\
p_1(a) \\
p_2(b)
\bigsetend{}$ &
$R = \bigsetbegin{}
\mathit{off}(X) \rightarrow p_3(X) \\
p_2(X) \rightarrow \mathit{on}(X) \\
p_1(X) \fork \mathit{off}(X) \\
p_3(X) \fork p_2(X) \\
p_2(X) \fork p_1(X)
\bigsetend{}$ &
$C' = \bigsetbegin{}
\forall X {:}s, \; \mathit{on}(X) \oplus \mathit{off}(X) \\
\forall X {:}s, \; p_1(X) \oplus p_2(X) \oplus p_3(X)
\bigsetend{}$
\end{tabular}
\end{eqnarray*}
We can see, then, that the noise-robust version of the \sys{} is somewhat less data-efficient than the noise-intolerant version described earlier. 
\label{ex:noise-1}
\end{example}

\subsection{Experiments}

We used the following sequences to compare the noise-intolerant \sys{} with the noise-robust version:
\begin{center}
\begin{tabular}{ll}
\sequence{a,b,a,b,a,b,a,b,a,b,a,b} & \sequence{a,a,b,a,a,b,a,a,b,a,a,b} \\
\sequence{a,a,b,b,a,a,b,b,a,a,b,b} & \sequence{a,a,a,b,a,a,a,b,a,a,a,b} \\
\sequence{a,b,b,a,a,b,b,a,a,b,b,a} & \sequence{a,b,c,a,b,c,a,b,c,a,b,c} \\
\sequence{a,b,c,b,a,a,b,c,b,a,a,b,c,b,a} & \sequence{a,b,a,c,a,b,a,c,a,b,a,c} \\
\sequence{a,b,c,c,a,b,c,c,a,b,c,c} & \sequence{a,a,b,b,c,c,a,a,b,b,c,c}
\end{tabular}
\end{center}
We chose these particular sequences because they are simple, noise-free, and the \sys{} is able to solve them in a reasonably short time.

We performed two groups of experiments.
In the first, we evaluated how much longer the sequence needs to be for the noise-robust version to capture the underlying regularity, in comparison with the noise-intolerant version which is more data-efficient.
Figure \ref{fig:noise-1} shows the results. 
We plot mean percentage accuracy (over the ten sequences) against the length of the sequence that is provided to the \sys{}.
Note that the noise-intolerant version only needs sequences of length 10 to achieve 100\% accuracy, while the noise-tolerant version needs sequences of length 30.

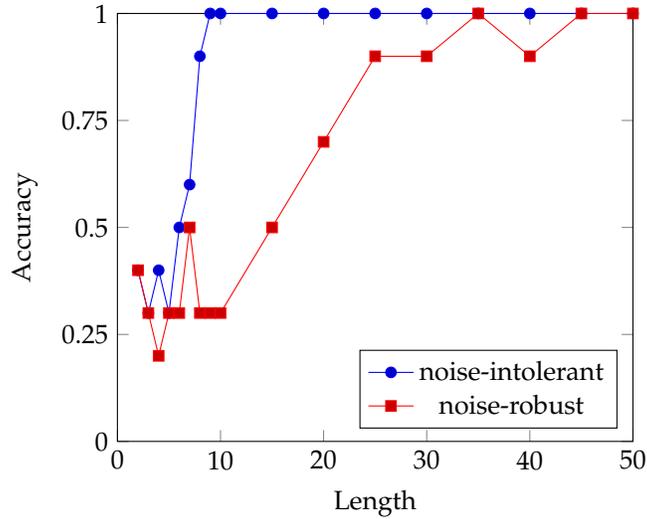
\begin{figure}
  \centering
  \begin{tikzpicture}
    \begin{axis}[
        xmin=0,
        xmax=50,
        xtick={0,10,20,30,40, 50},
        ytick={0,0.25,0.5,0.75,1.0},
        ymin=0,
        ymax=1.0,
        xlabel={Length},
        ylabel={Accuracy},
        legend pos= south east
              ]
      \addplot coordinates{(2,0.4) (3, 0.3) (4, 0.4) (5, 0.3) (6, 0.5) (7, 0.6) (8, 0.9) (9, 1.0) (10, 1.0) (15, 1.0) (20, 1.0) (25, 1.0) (30, 1.0) (35, 1.0) (40, 1.0) (45, 1.0) (50, 1.0) } ;
      \addlegendentry{noise-intolerant}
      \addplot coordinates{(2, 0.4) (3, 0.3) (4, 0.2) (5, 0.3) (6, 0.3) (7, 0.5) (8, 0.3) (9, 0.3) (10, 0.3) (15, 0.5) (20, 0.7) (25, 0.9) (30, 0.9) (35, 1.0) (40, 0.9) (45, 1.0) (50, 1.0) };
      \addlegendentry{noise-robust}
    \end{axis}
  \end{tikzpicture}
\caption[Comparing the noise-robust and noise-intolerant versions for data-efficiency]{Comparing the data-efficiency of the noise-robust version of the \sys{} with the noise-intolerant version. We plot mean percentage accuracy against length of the sequence. The noise-intolerant version achieves 100\% accuracy when the sequence is length 10 or over, while the noise-robust version only achieves this level of accuracy when the length is over 30.}
\label{fig:noise-1}
\end{figure}

In the second experiment, we evaluate how much better the noise-robust version of the \sys{} is at handling mislabeled data.
We take the same ten sequences above, extended to length 100, and consider various perturbations of the sequence where we randomly mislabel a certain number of entries.
Figure \ref{fig:noise-2} shows the results.
We plot mean percentage accuracy (over the ten sequences) against the percentage of mislabellings. 
Note that the noise-intolerant version deteriorates to random as soon as any noise is introduced, while the noise-robust version is able to maintain reasonable accuracy with up to 30\% of the sequence mislabeled.

\begin{figure}
  \centering
  \begin{tikzpicture}
    \begin{axis}[
        xmin=0,
        xmax=50,
        xtick={0,10,20,30,40, 50},
        ytick={0,0.25,0.5,0.75,1.0},
        ymin=0,
        ymax=1.0,
        xlabel={Percentage of mislabelled data},
        ylabel={Accuracy},
        legend pos= south east
              ]
      \addplot coordinates{(0,1.0) (5, 0.4) (10, 0.4) (15, 0.4) (20, 0.4) (25, 0.4) (30, 0.4) (35, 0.4) (40, 0.4) (45, 0.4) (50, 0.4)  } ;
      \addlegendentry{noise-intolerant}
      \addplot coordinates{(0,1.0) (5, 0.96) (10, 0.92) (15, 0.96) (20, 0.97) (25, 0.96) (30, 0.88) (35, 0.76) (40, 0.56) (45, 0.52) (50, 0.51) };
      \addlegendentry{noise-robust}
    \end{axis}
  \end{tikzpicture}
\caption[Comparing the noise-robust and noise-intolerant versions for accuracy]{Comparing the accuracy of the noise-robust version of the \sys{} with the noise-intolerant version. We plot mean percentage accuracy against the number of mislabellings. The noise-intolerant version deteriorates to random as soon as any noise is introduced, while the noise-robust version is able to maintain reasonable (88\%) accuracy with up to 30\% of the sequence mislabelled.}
\label{fig:noise-2}
\end{figure}
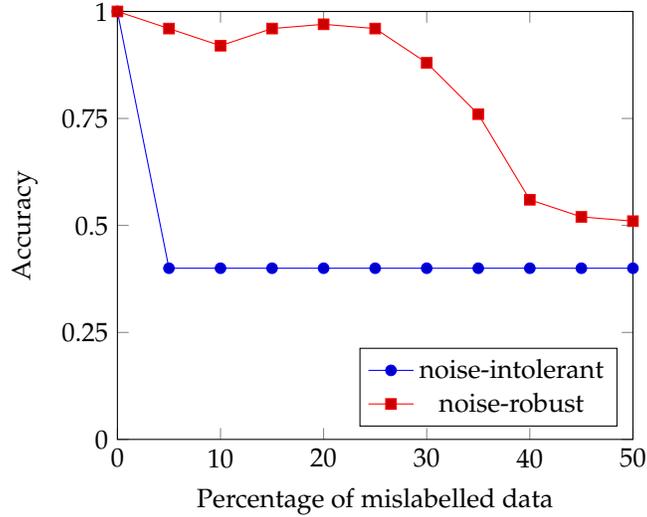

\section{Related work}
\label{sec:discrete-related}

In this section, we describe particular systems that are related to our approach. For a general overview of the space of different approaches, see Section \ref{sec:intro-related}.

\subsection{``Theory learning as stochastic search in a language of thought''}
\label{sec:related-ullman}

Ullman et al \cite{goodman2011learning,ullman2012theory} describe a system for learning first-order rules from symbolic data.
Recasting their approach into our notation, their system is given as input a set $S$ of ground atoms\footnote{Compare with our system, which is given a sequence $(S_1, ..., S_T)$ of sets of ground atoms.}, and it searches for a set of static rules $R$ and a set $I$ of atoms such that $R, I \models S$.

Of course, the task as just formulated admits of entirely trivial solutions: for example, let $I = S$ and $R = \{\}$.
Ullman et al rule out such trivial solutions by adding two restrictions.
First, they distinguish between two disjoint sets of predicates:
the \emph{surface} predicates are the predicates that appear in the input $S$, while the \emph{core} predicates are the latent predicates. 
Only core predicates are allowed to appear in the initial conditions $I$.
This distinction rules out the trivial solution above, but there are other degenerate solutions:
for each surface predicate $p$, add a new core predicate $p_c$.
If $p(k_1, ..., k_n)$ is in $S$, add $p_c(k_1, ..., k_n)$ to $I$.
Also, add the rule $p(X_1, ..., X_n) \leftarrow p_c(X_1, .., X_n)$ to $R$.
Clearly, $R, I \models S$ but this solution is unilluminating, to say the least.
To prevent such degenerate solutions, the second restriction that Ullman et al add is to prefer \emph{shorter} rule-sets  $R$ and smaller sets $I$ of initial atoms.
The idea is that if $S$ contains structural regularities, their system will find an  $R$ and $I$ that are much simpler than the degenerate solution above.

Consider, for example, the various surface relations in a family tree: John is the father of William; William is the husband of Anne; Anne is the mother of Judith; John is the grandfather of Judith.
All the various surface relations (father, mother, husband, grandfather...) can be explained by a small number of core relations: $\mathit{parent}(X, Y)$, $\mathit{spouse}(X, Y)$, $\mathit{male}(X)$, and $\mathit{female}(X)$.
Now the surface facts $S = \{\mathit{father}(\mathit{john}, \mathit{william}), ...\}$ can be explained by a small number of facts involving core predicates $I = \{\mathit{parent}(\mathit{john}, \mathit{william}), \mathit{male}(\mathit{john}), ...\}$ together with rules such as:
\begin{eqnarray*}
\mathit{father}(X, Y) \leftarrow \mathit{parent}(X, Y), \mathit{male}(X)
\end{eqnarray*}

At the computational level, then, the task that Ullman et al set out to solve is: given a set $S$ of ground atoms featuring surface predicates, find the smallest set $I$ of ground atoms featuring only core predicates, and the smallest set $R$ of static rules, such that $R, I \models S$.
Recasting this task in the language of probability, they wish to find:
\begin{eqnarray*}
\operatorname*{arg\,max}_{R, I} p(R, I \mid S)
\end{eqnarray*}
Using Bayes' rule this can be recast as:
\begin{eqnarray*}
\operatorname*{arg\,max}_{R, I} p(R, I \mid S) & = & \operatorname*{arg\,max}_{R, I} \frac{p(S \mid R, I) p(R, I)}{p(S)} \\
& = & \operatorname*{arg\,max}_{R, I} p(S \mid R, I) p(R, I) \\
& = & \operatorname*{arg\,max}_{R, I} p(S \mid R, I) p(R) p(I \mid R)
\end{eqnarray*}
Here, the likelihood $p(S \mid R, I)$ is the proportion of $S$ that is entailed by $R$ and $I$, the prior $p(R)$ is the size of the rules, and $p(I \mid R)$ is the size of $I$.

At the algorithmic level, Ullman et al apply Markov Chain Monte Carlo (MCMC).
MCMC is a stochastic search procedure.
When it is currently considering search element $x$, it generates a candidate next element $x'$ by randomly perturbing $x$. 
Then it compares the scores of $x$ and $x'$.
If $x'$ is better, it switches attention to focus on $x'$. 
Otherwise, if $x'$ is worse than $x$, there is still a non-zero probability of switching (to avoid local minima), but the probability is lower when $x$' is significantly worse than the current search element $x$.

In their algorithm, MCMC is applied at two levels.
At the first level, a set $R$ of rules is perturbed into $R'$ by adding or removing atoms from clauses, or by switching one predicate for another predicate with the same arity. At the second level, $I$ is perturbed into $I'$ by changing the extension of the core predicates.

Given that the search space of sets of rules is so enormous, and that MCMC is a stochastic search procedure that only operates locally, the algorithm needs additional guidance to find solutions.
In their case, they provide a \emph{template}, a set of meta-rules that constrain the types of rules that are generated in the outermost MCMC loop. 
A meta-rule is a higher-order clause in which the predicates are themselves variables. 
For example, in the following meta-rule for transitivity, $P$ is a variable ranging over two-place predicates:
\begin{eqnarray*}
P(X, Y) \leftarrow P(X, Z), P(Z, Y)
\end{eqnarray*}
Meta-rules are a key component in many logic program synthesis systems \cite{mugg:metagold,crop:metafunc,crop:thesis,law2014inductive,law2018complexity}.

Ullman et al tested their system in a number of domains including taxonomy hierarchies, simplified magnetic theories, kinship relations, and psychological explanations of action. 
In each domain, their system is able to learn human-interpretable theories from small amounts of data.

At a high level, Ullman et al's system has much in common with the \sys{}.
They are both systems for generating interpretable explanations from small quantities of symbolic data.
While the \sys{} generates a $(\phi, I, R, C)$ tuple from a sequence $(S_1, ..., S_T)$, their system generates an $(I, R)$ pair from a single set $S$ of atoms.
But there are a number of significant differences.
First, our system takes as input a \emph{sequence} $(S_1, ..., S_T)$ while their system considers only a single state $S$.
Because they do not model facts changing over time, their system only needs to represent static rules and does not need to also represent causal rules.
Second, a unified interpretation $\theta = (\phi, I, R, C)$ in our system includes a set $C$ of \emph{constraints}.
These constraints play a critical role in our system: they are both regulative (ruling out certain incompossible combinations of atoms) and constitutive (the constraints determine the incompossible relation that in turn grounds the frame axiom). 
There is no equivalent of our constraints $C$ in their system.
A third key difference is that our system has to produce a theory that, as well as explaining the sensory sequence, also has to satisfy the \emph{unity conditions}: spatial unity, conceptual unity, static unity, and temporal unity.
There is no analog of our unity conditions in Ullman et al's system.
Fourth, their system requires \emph{hand-engineered templates} in order to find a theory that explains the input.
This reliance on hand-engineered templates restricts the domain of application of their technique:
in a domain in which they do not know, in advance, the structure of the rules they want to learn, their system will not be applicable.
Fifth, our system posits \emph{latent objects} as well as latent predicates, while their system only posits latent predicates.
The ability to imagine unobserved objects, with unobserved attributes that explain the observed attributes of observed objects, is a key feature of the \sys{}.

At the algorithmic level, the systems are very different.
While we use a form of meta-interpretive learning (see Section \ref{sec:searching}), they use MCMC.
Our system compiles an apperception problem into the task of finding an answer set to an ASP program that minimises the program cost. 
The ASP problem is given to an ASP solver, that is guaranteed to find the \emph{global minimum}.
MCMC, by contrast, is a stochastic procedure that operates \emph{locally} (moving from one single point in program space to another), and is not guaranteed to (in fact, in practice, it rarely does) find a global minimum.

\subsection{``Learning from interpretation transition''}

Inoue, Ribeiro, and Sakama \cite{inoue2014learning} describe a system (\define{LFIT}) for learning logic programs from sequences of sets of ground atoms.
Since their task definition is broadly similar to ours, we focus on specific differences. 
In our formulation of the apperception task, we must construct a $(\phi, I, R, C)$ tuple from a sequence $(S_1, ..., S_T)$ of sets of ground atoms.
In their task formulation, they learn a set of causal rules from a set $\{(A_i, B_i)\}_{i=1}^N$ of pairs of sets of ground atoms.

In some respects, their task formulation is more general than ours.
First, their input  $\{(A_i, B_i)\}_{i=1}^N$ can represent transitions from \emph{multiple} trajectories, rather than just a single trajectory, and corresponds to a generalized apperception task (see Definition \ref{def:generalized-apperception-task}). 
Second, they learn \emph{normal} logic programs, allowing negation as failure in the body of a rule, while our system only learns definite clauses. 

But there are a number of other ways in which our task formulation is significantly more general than LFIT. 
First, our system posits \emph{latent information} to explain the observed sequence, while LFIT does not construct any latent information. 
Their system searches for a program $P$ that generates \emph{exactly} the output state. In our approach, by contrast, we  search for a program whose trace \emph{covers} the output sequence, but does not need to be identical to it. The trace of a unified interpretation typically contains much extra information that is not part of the original input sequence, but that is used to explain the input information. 

Second, our system abduces a set of \emph{initial conditions} as well as a set of rules, while LFIT does not construct initial conditions. Because of this, our system is able to predict the future, retrodict the past, and impute missing intermediate values. 
LFIT, by contrast, can only be used to predict future values. 

Third, our system generates a set of \emph{constraints} as well as rules. The constraints perform double duty: on the one hand, they restrict the sets of compossible atoms that can appear in traces; on the other hand, they generate the incompossibility relation that grounds the frame axiom. Note that there is no frame axiom in LFIT.

In \cite{inoue2014learning}, Inoue et al use a bottom-up synthesis method to learn rules.
Given a state transition $(A, B)$ in $E$, they construct a normal ground rule for each $\beta \in B$:
\begin{eqnarray*}
\bigwedge_{\alpha \in A} \alpha \wedge \bigwedge_{\alpha \in \mathcal{G} - A} \mathit{not} \; \alpha \fork \beta
\end{eqnarray*}
Then, they use resolution to generalize the individual ground rules. 
It is important to note that this strategy is quite conservative in the generalizations it performs, 
since it only produces a more general rule if it turns out to be a resolvent of a pair of previous rules. 
While the \sys{} searches for the shortest (and hence most general) rules, LFIT searches for the most specific generalization.
In more recent work \cite{ribeiro2015learning}, LFIT has been changed to perform top-down specialization, rather than bottom-up generalization. With this change, LFIT is guaranteed to find the shortest set of rules that explain the transitions.  

LFIT was tested on Boolean networks and on Elementary Cellular Automata. 
It is instructive to compare our system with LFIT on the ECA tasks.
When LFIT is applied to the ECA task, it is provided with the one-dimensional spatial relation between the cells as background knowledge. 
In our approach, by contrast, we do \emph{not} hand-code the spatial relation, but rather let the \sys{} generate the spatial relation itself as part of the initial conditions. (See Section \ref{sec:eca}). It is precisely because our system is able to posit latent information to explain the surface features that it is able to generate the spatial relation itself, rather than having to be given it.

In some situations, positing latent information allow us to constructer a simpler theory.
In other situations, however, positing latent information is absolutely essential to making sense of the sequence. There are many apperception tasks for which every interpretation that makes sense of the sequence \emph{must} include latent information:
consider, for example, learning the dynamics of the ECA without spatial information (Section \ref{sec:eca}), the \emph{Seek Whence} sequences (Section \ref{sec:seek-whence}), or the binding tasks (Section \ref{sec:probe-task-binding}).
Even the following simple example shows the unavoidable need to posit latent information:
\begin{example}
Consider the task $(S, \phi, C)$ where $S = (S_1, S_2, S_3, ..., S_n)$, $S_1 = S_2= \{p(a)\}$, and $S_3 = S_4 = ... = S_n = \{q(a)\}$. Here, $\phi$ contains one type $t$, one object $a$ of type $t$, two unary predicates $p$ and $q$, and one constraint $\forall X: t, p(X) \oplus q(X)$. Since $p$ and $q$ are incompatible, the transition from $p$ to $q$ from state $S_2$ to $S_3$ must be explained by a causal rule $\phi \fork q(X)$, where $\phi$ is a set of atoms. Now $\phi$ cannot be empty, or the rule would be unsafe, $\phi$ cannot be $\{p(X)\}$, or else $q(a)$ would be derivable at the second time-step, contradicting $S_2 = \{p(a)\}$. Similarly, $\phi$ cannot be $\{q(X)\}$, or else $q(a)$ would not be derivable at time-step 3. Hence $\phi$ must contain an atom featuring a predicate distinct from $p$ or $q$. Hence, every interpretation that makes sense of $S$ must invoke latent information. 

One of the theories found by the \sys{} for this is $\theta = (\phi, I, R, C)$, where:
\begin{eqnarray*}
\begin{tabular}{lll}
$I = \bigsetbegin{}
p(a) \\
r(a)
\bigsetend{}$ &
$R = \bigsetbegin{}
r(X) \fork s(X) \\
p(X) \wedge s(X) \fork q(X)
\bigsetend{}$ &
$C = \bigsetbegin{}
\forall X {:}t, \; p(X) \oplus q(X) \\
\forall X {:}t, \; r(X) \oplus s(X)
\bigsetend{}$
\end{tabular}
\end{eqnarray*}
Here, there are two latent predicates, $r$ and $s$, that are used as counters, so the system can distinguish between the two occurrences of $p(a)$ in $S_1$ and $S_2$.
Thus for some sequences, the positing of latent predicates, and the abduction of initial conditions for the latent atoms, is \emph{unavoidable}.
\end{example}
LFIT has been extended in a number of ways, to increase the range of real-world problems that it can tackle.
In \cite{martinez2017relational}, LFIT was extended to learn probabilistic models.
In \cite{ribeiro2015learningb}, the system was extended from the Markov($1$) assumption (where the new state depends only on the current state) to the more general Markov($k$) setting (where the new state depends on the last $k$ states). 
In \cite{ribeiro2018learning}, LFIT was generalised so that as well as working with deterministic models (where all state transitions happen simultaneously), it also can work with other semantics (where only a subset of the transitions may happen at each time-step).
In \cite{ribeiro2017inductive}, LFIT was extended to work directly with continuous sensor data, rather than assuming the continuous sensor data has first been discretised by some other process. 
In \cite{tourret2017learning,phua2017learning}, LFIT was reimplemented in a feed-forward neural network, so as to robustly handle noisy and continuous data.

\subsection{``Unsupervised learning by program synthesis''}
\label{sec:related-ellis}

Ellis et al \cite{ellis2015unsupervised} use program synthesis to solve an unsupervised learning problem.
Given an unlabeled dataset $\{x_i\}_{i=1}^N$, they find a program $f$ and a set of inputs $\set{I_i}_{i=1}^N$ such that $f(I_i)$ is close to $x_i$ for each $i = 1 .. N$.
More precisely, they use Bayesian inference to find the $f$ and $\set{I_i}_{i=1}^N$ that minimizes the combined log lengths of the program, the initial conditions, and the data-reconstruction error:
\begin{eqnarray*}
-log P_f(f) + \sum_{i=1}^N \left( - log P_{x|z} (x_i \mid f(I_i)) - log P_I(I_i) \right)
\end{eqnarray*}
where $P_f(f)$ is a description length prior over programs, $P_I(I_i)$ is a description length prior over initial conditions, and $P_{x|z}(\cdot \mid z_i)$ is a noise model. 
This system was designed from the outset to be robust to noise, using Bayesian inference to calculate the desired tradeoff between the program length, the initial conditions length, and the data-reconstruction error cost.
They tested this system in two domains: reproducing two dimensional pictures, and learning morphological rules for English verbs.

This system is similar to ours in that it produces interpretable programs from a small number of data samples.
Like ours, their program length prior acts as an inductive bias that prefers general solutions over special-case memorized solutions.
Like ours, as well as constructing a program, they also learn initial conditions that combine with the program to produce the desired results\footnote{In fact, they learn a different set of initial conditions $I_i$ for each data point $x_i$. This corresponds to the generalized apperception task of Definition \ref{def:generalized-apperception-task}.}.
At a high level, their algorithm is also similar: they generate a Sketch program \cite{solar2006combinatorial} from the dataset $\{x_i\}_{i=1}^N$ of examples, and use a SMT solver to fill in the holes.
They then extract a readable program from the SMT solution, which they then apply to new instances, exhibiting strong generalization.

As well as the high level architectural similarities, there are a number of important differences.
First, their goal was to generate an object $f(I_i)$ that matches as closely as possible to the input object $x_i$. 
Our goal is more general: we seek to generate a sequence $\tau(\theta)$ that \emph{covers} the input sequence.
The covering relation is much more general, as $S_i$ only has to be a \emph{subset} of $(\tau(\theta))_i$, not identical to it.
This allows the addition of latent information to the trace of the theory.
A second key difference is that we focus on generating sequences, not individual objects. 
Our system is designed for making sense (unsupervisedly) of time series, sequences of states, not of reconstructing individual objects.
A third key difference is that we use a single domain-independent language, \logic{}, for all domains, while Ellis et al use a different domain-specific imperative language for each domain they consider.
A fourth key difference is that we use a declarative language, rather than an imperative language. 
An individual rule or constraint has a truth-conditional interpretation, and can be interpreted as a \emph{belief} of the synthesising agent. An individual line of an imperative procedure, by contrast, cannot be interpreted as a belief.
A fifth major difference is that we synthesise \emph{constraints} as well as rules.
Constraints are the ``special sauce'' of our system: exclusive disjunctions combine predicates into groups, enforce that each state is fully determinate, and ground the incompossibility relation that underlies the frame axiom.

\subsection{``Learning symbolic models of stochastic domains''}
\label{sec:related-learning-symbolic-models}

Pasula et al \cite{pasula2007learning} describe a system for learning a state transition model from data.
The model learns a probability distribution $p(s' \mid s, a)$ where $s$ is the previous state, $a$ is the action that was performed and $s'$ is the next state. 

Each state is represented as a set of ground atoms, just like in our system.
They assume \emph{complete observability}: they assume they are given the value of every sensor and the task is just to predict the next values of the sensors.

They represent a state transition model by a set of ``dynamic rules'': these are first-order clauses determining the future state given a current state and an action. These dynamic rules are very close to the causal rules in \logic{}. Unlike in our system, their rules have a probability outcome for each possible head. 
Note their system does not include static rules or constraints.

In their semantics, they assume that \emph{exactly one dynamic rule fires every time-step}. This is a very strong assumption. But it makes it easier to learn rules with probabilistic outcomes. 

They learn state transitions for the noisy gripper domain (where a robot hand is stacking bricks, and sometimes fails to pick up what it attempts to pick up) and a logistics problem (involving trucks transporting objects from one location to another). Impressively, they are able to learn probabilistic rules in noisy settings.
They also verify the usefulness of their learned models by passing them to a planner (a sparse sampling MDP planner), and show, reassuringly, that the agent achieves more reward with a more accurate model.

At a strategic level, their system is similar in approach to ours.
First, they learn first-order rules, not merely propositional ones. In fact, they show in ablation studies that learning propositional rules generalises significantly less well, as you would expect.
Second, they use an inductive bias against constants (p.14), just as we do: ``learning action models which are restricted to be free of constants provides a useful bias that can improve generalisation when training with small data sets''. 
Third, their system is able to construct new invented predicates. 

But there are also a number of differences. 
In our system, many rules can fire simultaneously. 
But in theirs, only one rule can fire in any state. 
Because of this assumption, they cannot model e.g.~a cellular automaton, where each cell has its own individual update rule firing simultaneously.
Another limiting assumption is that they assume they have complete observability of all sensory predicates. 
This means they would not be able to solve e.g.~occlusion tasks.  

\subsection{``Nonmonotonic abductive inductive learning''}
\label{sec:related-work-xhail}

Ray \cite{ray2009nonmonotonic} described a system, \define{XHAIL}, for jointly learning to abduce ground atoms and induce first-order rules. 
XHAIL learns normal logic programs that can include negation as failure in the body of a rule. 

XHAIL is similar to the \sys{} in that as well as inducing general first-order rules, it also constructs a set of initial ground atoms. 
This enables it to model latent (unobserved) information, which is a very powerful and useful feature.
At the implementation level, it uses a similar strategy in that solutions are found by iterative deepening over a series of increasingly complex ASP programs.
The simplified event calculus \cite{kowalski1989logic} is represented explicitly as background knowledge.

But there are also a number of key differences.
First, it does not model constraints. This means it is not able to represent the incompossibility relation between ground atoms. 
Also, XHAIL does not try to satisfy our other unity conditions, such as spatial and conceptual unity.
Second, the induced rules are \emph{compiled} in XHAIL, rather than being interpreted (as in our system). 
Representing each candidate induced rule explicitly as a separate ASP rule means that the number of ASP rules considered grows exponentially with the size of the rule body\footnote{It shares the same implementation strategy as ASPAL \cite{corapi2012inductive} and ILASP \cite{law2014inductive}. See Section \ref{sec:ilasp} for discussion of the grounding problem associated with this family of approaches. The discussion is specifically focused on ILASP, but we believe the same issue affects ASPAL and XHAIL \emph{mutatis mutandem}.}.
Third, XHAIL needs to be provided with a set of mode declarations to limit the search space of possible induced rules.
These mode declarations constitute a significant piece of background knowledge.
Now of course there is nothing wrong with allowing an ILP system to take advantage of background knowledge to aid the search. But when an ILP system \emph{relies} on this hand-engineered knowledge, then it restricts the range of applicability to domains in which human engineers can anticipate in advance the form of the rules they want the system to learn\footnote{See Appendix C of \cite{evans2018learning} for a discussion of the use of mode declarations as a language bias in ILP systems.}.

\subsection{The Game Description Language and \logic{}}
\label{sec:distinctive-aspects}

Our language \logic{} is an extension of Datalog that incorporates, as well as the standard static rules of Datalog, both causal rules (Definition \ref{def:rules}) and constraints (Definition \ref{def:constraints}). 
The semantics of \logic{} are defined according to Definition \ref{def:trace}.
Unlike standard Datalog, the atoms and rules of \logic{} are strongly typed (see Definitions \ref{def:suitable-type-signature}, \ref{def:unground}, and \ref{def:rules}).

At a high level, \logic{} is related to the Game Description Language (GDL) \cite{love2008general}.
The GDL is an extension of Datalog that was designed to express deterministic multi-agent discrete Markov decision processes. 
The GDL includes (stratified) negation by failure, as well as some (restricted) use of function symbols, but these extensions were carefully designed to preserve the key Datalog property that a program has a unique subset-minimal Herbrand model.
The GDL includes special keywords, including \verb|init| for specifying initial conditions (equivalent to the initial conditions $I$ in a $(\phi, I, R, C)$ theory), and \verb|next| for specifying state transitions (equivalent to our causal rules). 
The \emph{inductive general game playing} (IGGP) task \cite{genesereth2013international,cropper2019inductive} involves learning the rules of a game from observing traces of play. 

An IGGP task is broadly similar to an apperception task in that both involve inducing initial conditions and rules from traces. 
But there are many key differences. 
One major feature of \logic{} is the use of \emph{constraints} to generate incompossible sets of ground atoms.
These exclusion constraints are needed to generate the incompossibility relation which in turn is needed to restrict the scope of the frame axiom (see Definition \ref{def:trace}). 

The main difference between \logic{} and the GDL is that the former includes exclusion constraints.
The exclusion constraints play two essential roles.
First, they enable the theory as a whole to satisfy the condition of conceptual unity.
Second, they provide constraints, via the condition of static unity, on the generated trace: since the constraints must always be satisfied, this restricts the rules that can be constructed.
Satisfying these constraints means \emph{filling in missing information}. 
This is why a unified interpretation is able to make sense of incomplete traces where some of the sensory data is missing.

\subsection{Related application areas}
\label{sec:other-related-work}

We briefly outline three related research areas where the \sys{} can be applied.
One application area is relational reinforcement learning \cite{dvzeroski2001relational,de2008logical,bu2008comprehensive}.
Here, the agent works out how to optimize its reward in an environment by constructing (using ILP) a first-order model of the dynamics of that environment, which it then uses to plan. 
Here, the \sys{} can be used to construct the dynamics model.

A second application area is learning game rules from player traces \cite{goodacre1996inductive,law2014inductive}.
Here, the learning system is presented with traces (typically, sequences of sets of ground atoms), representing the state of the game at various points in time, and has to learn the transition dynamics (or reward function, or action legality function) of the underlying system. 

A third related area is the predictive processing (PP) paradigm \cite{friston2005theory,friston2012history,clark2013whatever,swanson2016predictive}, an increasingly popular model in computational and cognitive neuroscience.
Inspired by Helmholtz, the model learns to make sense of its sensory stream by attempting to predict future percepts. When the predicted percepts diverge from the actual percepts, the model updates its parameters to minimize prediction error. The PP model is probabilistic, Bayesian, and hierarchical: probabilistic in that the predicted sensory readings are represented as probability density functions, Bayesian in that the likelihood (represented by the information size of the misclassified predictions) is combined with prior expectations \cite{friston2012history}, and hierarchical in that each layer provides predictions of sensory input for the layer below; there are typically many layers.
While PP focuses on prediction, the \sys{} generates an interpretation that is equally adept at predicting future signals, retrodicting past signals, and imputing missing intermediary signals. 
In our approach, the ability to predict future signals is a derived capacity, a capacity that emerges from the more general capacity to construct a unified interpretation -- but prediction is not singled out in particular. The \sys{} is able to predict, retrodict, and impute -- in fact, it is able to do all three \emph{simultaneously} using a single incomplete sensory sequence with elements missing at the beginning, at the end, and in the middle.

%
%

\subsection{Summary}
\label{sec:summary-related-work}

In summary, there are various other systems that construct dynamic rules for explaining sequences. 
But these systems are unable to posit latent hidden information to make sense of the sequence.
These systems are able to predict future elements of the sequence, but are not able to retrodict earlier elements, or impute missing intermediary elements.
For problems that require positing latent hidden information\footnote{For example, making sense of the ECA sequences \emph{without} spatial information (Section \ref{sec:eca}), or interpreting the \emph{Seek Whence} sequences (Section \ref{sec:seek-whence}), or the binding tasks (Section \ref{sec:probe-task-binding}).}, or problems that require retrodiction and imputation as well as prediction, the \sys{} is particularly well suited.

\section{Conclusion}
\label{sec:conclusion}

This paper is an attempt to answer a key question of unsupervised learning: what does it mean to ``make sense'' of a sensory sequence? 
Our answer is that making sense means constructing a symbolic theory containing a set of objects that persist over time, with attributes that change over time, according to general laws.
This theory must both explain the sensory input, and satisfy the unity conditions of Section \ref{sec:unity-conditions}.
As well as providing a precise formalization of this task, we also provide a concrete implementation of a system that is able to make sense of the sensory stream.
We have tested the \sys{} in a variety of domains;
in each domain, we tested its ability to predict future values, retrodict previous values, and impute missing intermediate values. 
Our system achieves good results across the board, outperforming neural network baselines and also state-of-the-art ILP systems.

Of particular note is that the \sys{} is able to achieve human performance on challenging sequence induction intelligence tests.
We stress, once more, that the system was not hard-coded to solve these tasks.
Rather, it is a general \emph{domain-independent} sense-making system that is able to apply its general architecture to the particular problem of Seek Whence induction tasks, and is able to solve these problems ``out of the box'' without human hand-engineered help.
We also stress, again, that the system did not learn to solve these sequence induction tasks by being presented with hundreds of training examples\footnote{Barrett et al \cite{barrett2018measuring} train a neural network to learn to solve Raven's progressive matrices from thousands of training examples.}. 
Indeed, the system had never seen a \emph{single} such task before.
Instead, it applied its general sense-making urge to each individual task, \emph{de novo}.

Our architecture, an unsupervised program synthesis system, is a purely symbolic system, and 
as such, it inherits two key advantages of ILP systems \cite{evans2018learning}.
First, the interpretations produced are \emph{interpretable}. 
Because the output is symbolic, it can be read and verified by a human\footnote{Large machine-generated programs are not always easy to understand. But machine-generated symbolic programs are certainly easier to understand than the weights of a neural network. See Muggleton et al \cite{muggleton2018ultra} for an extensive discussion. }.
Second, it is very \emph{data-efficient}.
Because of the language bias of the \logic{} language, and the strong inductive bias provided by the unity conditions, the system is able to make sense of extremely short sequences of sensory data, without having seen any others.

However, the system in its current form has some limitations that we wish to make explicit.
First, the sensory input must be discretized before it can be passed to the system.
We assume some prior system has already discretized the continuous sensory values by grouping them into buckets.
One possible approach to deal with continuous sensory values is to combine the \sys{} with a neural network that maps the raw continuous inputs into categories.
We have recently developed such an extension, in which we discretize the input by simulating a binary neural network. The binary neural network is implemented in ASP, so the weights of the network and the rules of the theory can be found \emph{simultaneously} by solving one large SAT problem. 

Second, our implementation as described above assumes all causal rules are fully deterministic.
It is quite straightforward to add non-determinism to the \logic{} framework:
we can define an \define{extended theory} as a theory with initial conditions for \emph{each time-step} (rather than only allowing initial conditions for the first time-step, as in Definition \ref{def:theory}).
An extended theory $\theta = (\phi, \{I_1, ..., I_{T}\}, R, C)$ generates a trace $\tau(\theta)  = (A_1, A_2, ...)$ in exactly the same way as in Definition \ref{def:trace}, with one small exception:
$I_t \subseteq A_t$ replaces $I \subseteq A_1$. In other words, new atoms can be abduced at each time-step.
This would allow us to handle non-determinism by abducing atoms that change their truth-value according to $\{I_1, ..., I_{T}\}$ instead of according to the rules in $R$.

Third, the size of the search space means that our system is currently restricted to small-to-medium-size problems.\footnote{This is not because our system is carelessly implemented: the \sys{} is able to synthesize significantly larger programs than state-of-the-art ILP systems (see the comparison with ILASP in Section \ref{sec:ilasp}).}
Going forward, we believe that the right way to build complex theories is incrementally, using curriculum learning: the system should consolidate what it learns in one episode, storing it as background knowledge, and reusing it in subsequent episodes.

We hope in future work to address these limitations. 
But we believe that, even in its current form, the \sys{} shows considerable promise as a prototype of what a general-purpose domain-independent sense-making machine must look like.

\section*{Acknowledgements}
We are very grateful to Andrew Cropper, Jessica Hamrick, Mark Law, Matko Bo\u snjak, Murray Shanahan, Nando de Freitas, and the anonymous reviewers for extensive feedback.

\section*{References}

\bibliography{main}

\end{document}